\documentclass{article}
\usepackage{arxiv}

\usepackage{amsmath,amsfonts,bm}
\newcommand*{\diffdchar}{\mathrm{d}}    % or {ⅆ}, or {\mathrm{d}}

\DeclareRobustCommand\onedot{}
\def\ie{\emph{i.e}\onedot}
\def\wrt{\emph{w.r.t }\onedot}

\def\ow{\emph{o.w}\onedot}
\newcommand{\st}{\mathrm{s.t\onedot}}

\def\eqref#1{(\ref{#1})}
\def\1{\bm{1}}

\def\veps{{\varepsilon}}

\def\ra{{\textnormal{a}}}

\def\rva{{\mathrm{a}}}

\def\rvi{{\mathrm{i}}}
\def\rvj{{\mathrm{j}}}

\def\rvx{{\mathrm{x}}}
\def\rvy{{\mathrm{y}}}

\DeclareMathAlphabet{\mathsfit}{\encodingdefault}{\sfdefault}{m}{sl}
\SetMathAlphabet{\mathsfit}{bold}{\encodingdefault}{\sfdefault}{bx}{n}

\def\cA{{\mathcal{A}}}

\def\cF{{\mathcal{F}}}

\def\cN{{\mathcal{N}}}

\def\cP{{\mathcal{P}}}

\def\cX{{\mathcal{X}}}
\def\cY{{\mathcal{Y}}}

\def\bE{{\mathbb{E}}}

\def\bP{{\mathbb{P}}}

\def\bR{{\mathbb{R}}}

\newcommand{\tdf}{\tilde{f}}

\newcommand{\Ber}{\operatorname{Bern}}

\newcommand{\haty}{\hat{\rvy}}
\DeclareMathOperator{\sign}{sign}

\usepackage{url}

\usepackage{xcolor}
\usepackage[abs]{overpic}
\usepackage{pifont}
\definecolor{mycolor}{HTML}{226CE0}
\usepackage{microtype}
\usepackage{graphicx}
\usepackage{subfigure}
\usepackage{booktabs} % for professional tables
\usepackage{bm}
\usepackage{amsmath}
\usepackage{wrapfig}
\usepackage{amssymb} 
\usepackage{stmaryrd}
\usepackage{array}
\usepackage{amsthm}
\usepackage{natbib}
\usepackage{tabularx} % in the preamble
\usepackage{multirow}
\newtheorem{theorem}{Theorem}
\newtheorem{example}{Example}

\newtheorem{corollary}[theorem]{Corollary}

\newtheorem{lemma}[theorem]{Lemma}
\newtheorem{assumption}{Assumption}
\newtheorem{proposition}[theorem]{Proposition}
\newtheorem{remark}{Remark}
\newtheorem{case}{Case}
\counterwithin*{case}{theorem}
\usepackage[breaklinks=true,bookmarks=false]{hyperref}
\usepackage{cleveref}
\usepackage{autonum}
\usepackage{enumitem}
\usepackage[font=small,labelfont=bf]{caption}
\usepackage[ruled,vlined,algo2e]{algorithm2e}
\newcommand{\norm}[1]{\left\lVert#1\right\rVert}
\newcolumntype{C}[1]{>{\Centering\arraybackslash}p{#1\linewidth}}
\newcolumntype{L}[1]{>{\RaggedRight\arraybackslash}p{#1\linewidth}}

\DeclareRobustCommand\onedot{}
\def\ie{\emph{i.e.}\onedot}
\def\wrt{\emph{w.r.t }\onedot}

\newcommand{\cmark}{\ding{51}}%
\newcommand{\xmark}{\ding{55}}%
\newcommand{\fflfedavg}{LFT+\textsc{FedAvg}}
\newcommand{\ufl}{LFT+Ensemble}
\newcommand{\fedfb}{\textsc{FedFB}}
\newcommand{\fb}{\textsc{FB}}
\newcommand{\fairbatch}{\textsc{FairBatch}}
\newcommand{\ditto}{\textsc{Ditto}}
\newcommand{\qffl}{\textsc{q-FFL}}
\newcommand{\gifair}{\textsc{GIFAIR}}
\newcommand{\fedavg}{\textsc{FedAvg}}
\newcommand{\fairfed}{\textsc{FairFed}}
\newcommand{\agnosticfair}{\textsc{AgnosticFair}}

\newcommand{\fuflm}{f_{\boldsymbol{\veps}}^{\text{LFT+Ensemble}}}
\newcommand{\ffflm}{f_{\boldsymbol{\veps}}^{\text{LFT+FedAvg}}}
\newcommand{\fufl}{f_{\veps_0, \veps_1}^{\text{LFT+Ensemble}}}
\newcommand{\fffl}{f_{\veps_0, \veps_1}^{\text{LFT+FedAvg}}}
\newcommand{\fcfl}{f_{\veps}^{\text{CFL}}}
\newcommand{\lbdcfl}{\lambda_{\veps}^{\text{CFL}}}
\newcommand{\lbdufl}{\lambda_{\veps_i}^{\text{LFT+Ensemble}_i}}
\newcommand{\lbduflo}{\lambda_{\veps_0}^{\text{LFT+Ensemble}_0}}
\newcommand{\lbdufll}{\lambda_{\veps_1}^{\text{LFT+Ensemble}_1}}
\newcommand{\md}{\text{MD}}
\newcommand{\disp}{\text{DP Disp}}
\newcommand{\set}[1]{\left\{#1\right\}}
\newcommand{\indicator}[1]{\llbracket#1\rrbracket}

\newcommand{\fuflp}{\Tilde{f}_{\Tilde{\veps}_0, \Tilde{\veps}_1}^{\text{LFT+Ensemble}}}
\newcommand{\lbduflop}{\Tilde{\lambda}_{\Tilde{\veps}_0}^{\text{LFT+Ensemble}_0}}
\newcommand{\lbdufllp}{\Tilde{\lambda}_{\Tilde{\veps}_1}^{\text{LFT+Ensemble}_1}}
\newcommand{\tg}{\Tilde{g}}
\newcommand{\tpsi}{\Tilde{\psi}}
\newcommand{\tveps}{\Tilde{\veps}}
\newcommand{\tdelta}{\Tilde{\delta}}

\title{Improving Fairness via Federated Learning}

\date{} 					% Or removing it
\author{Yuchen~Zeng, ~~Hongxu~Chen,~~Kangwook~Lee \\
University of Wisconsin-Madison\\
}

\renewcommand{\headeright}{ }
\renewcommand{\undertitle}{ }
\renewcommand{\shorttitle}{Improving Fairness via Federated Learning}

\hypersetup{
pdftitle={Improving Fairness via Federated Learning},
pdfsubject={q-bio.NC, q-bio.QM},
pdfauthor={David S.~Hippocampus, Elias D.~Striatum},
pdfkeywords={First keyword, Second keyword, More},
}

\begin{document}
\maketitle

\begin{abstract}
Recently, lots of algorithms have been proposed for learning a fair classifier from decentralized data. However, many theoretical and algorithmic questions remain open.
First, is federated learning necessary, i.e., can we simply train locally fair classifiers and aggregate them? 
In this work, we first propose a new theoretical framework, with which we demonstrate that federated learning can strictly boost model fairness compared with such non-federated algorithms. 
We then theoretically and empirically show that the performance tradeoff of \fedavg{}-based fair learning algorithms is strictly worse than that of a fair classifier trained on centralized data. 
To bridge this gap, we propose \fedfb{}, a private fair learning algorithm on decentralized data. 
The key idea is to modify the \fedavg{} protocol so that it can effectively mimic the centralized fair learning.
Our experimental results show that \fedfb{} significantly outperforms existing approaches, sometimes matching the performance of the centrally trained model.
\end{abstract}

% keywords can be removed
% \keywords{Group fairness \and Federated Learning}

\section{Introduction}
Fair learning has recently received increasing attention. 
Various fairness notions have been introduced in the past few years~\citep{dwork2012individual,Hardt2016Equality,zafar2019preferece,Zafar2017Mistreatment,kearns2018audit,friedler2016impossibility}. 
Among various fairness notions, \textit{group fairness} \citep{Hardt2016Equality,Zafar2017Mistreatment} is the most studied one; it requires the classifier to treat different groups similarly, where groups are defined with respect to sensitive attributes such as gender and race.
One of the most commonly used group fairness notions is \textit{demographic parity}, which requires that different groups are equally likely to receive desirable outcomes.

There has been a large amount of work in training fair classifiers~\citep{Zafar2017FC,Hardt2016Equality,roh2021fairbatch}, and almost all of these studies assume that the learner has access to the entire training data.
Unfortunately, this is \emph{not} the case in many critical applications. 
For example, the government may want to train a single classifier that is fair and accurate using datasets owned by different stakeholders across the country, but the stakeholders may not be able to share their raw data with the government due to the privacy act.
% to mitigate bias from human stereotypes, multiple courts are willing to collaborate to learn a single model that is fair to different races.
% However, they cannot directly share their data with the others due to the privacy act.
This precisely sets the core question that we aim to answer in this paper -- \emph{how can we privately train a fair classifier on decentralized data?}
To answer this, consider the following approaches: Local Fair Training + Ensemble and Local Fair Training + \fedavg{}.

\begin{figure}[t]
    \centering
\begin{overpic}[width=\textwidth]{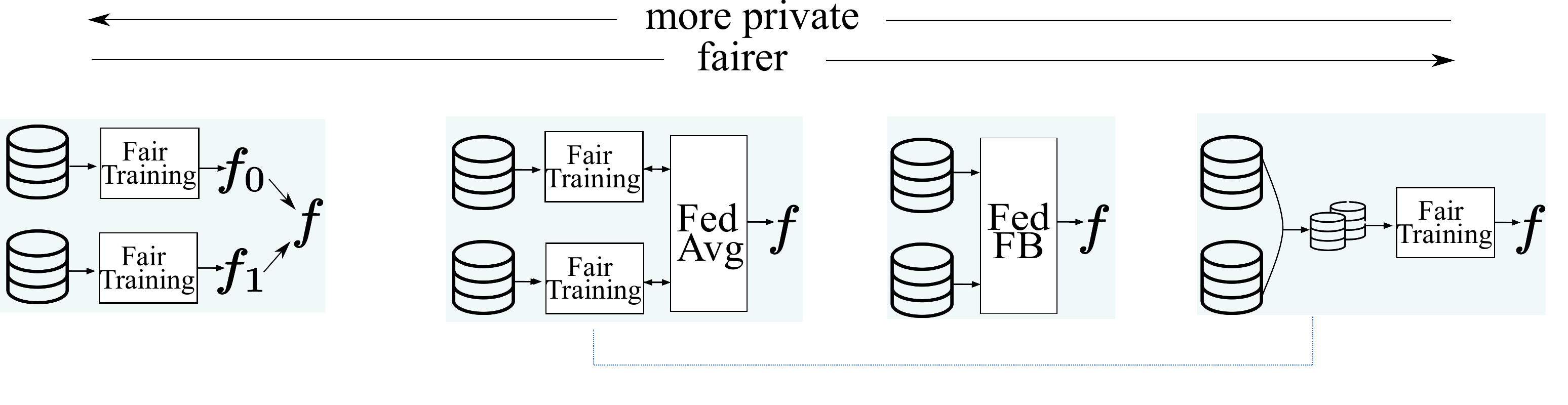}
 \put (112,60) {\large$\displaystyle \textcolor{mycolor}{<}$}
 \put (250,50) {\large$\displaystyle \textcolor{mycolor}{<}$}
 \put (280,15) {\large$\displaystyle \textcolor{mycolor}{<}$}
 
 \put (340,50) {\large$\displaystyle \textcolor{mycolor}{\approx}$}
 \put (20,90) {\normalsize \ufl{}}
 \put (160,90) {\normalsize{\fflfedavg{}}}
  \put (280,90) {\normalsize{FedFB (Ours)}}
 \put (375,90) {\normalsize{CFL (Upper Bound)}}
\put (100,50) {\normalsize\textcolor{mycolor}{Sec.~\ref{sec:necessity}}}
\put (98,40) {\normalsize\textcolor{mycolor}{(Thm.~\ref{thm:uflvsffl})}}

\put (240,40) {\normalsize\textcolor{mycolor}{Sec.~\ref{sec:experiments}}}
\put (335,40) {\normalsize\textcolor{mycolor}{Sec.~\ref{sec:experiments}}}

 \put (250,0) {\normalsize\textcolor{mycolor}{Sec.~\ref{sec:4.1} (Lemma~\ref{lemma:ffl_limit})}}
\end{overpic}
  \caption{\textbf{A high-level illustration of various approaches to fair learning on decentralized data and our contributions.} Assuming two data owners, from left to right, we show \ufl{} (Local Fair Training + Ensemble), \fflfedavg{} (Local Fair Training +  \fedavg{}), \fedfb{} (ours), and CFL (Centralized Fair Learning). In \ufl{}, each data owner trains a locally fair model on its own data and we assemble the local models to obtain a global model. 
  \fflfedavg{} applies \fedavg{} together with off-the-shelf fair training algorithms for local training. 
  Our proposed solution FedFB consists of a modified \fedavg{} protocol with a custom-designed fair learning algorithm. 
  CFL is the setting where a fair model is trained on the pooled data. 
  In this work, we theoretically characterize the strict ordering between the existing approaches and empirically demonstrate the superior performance of FedFB.}\label{fig:fldemo}
\end{figure}
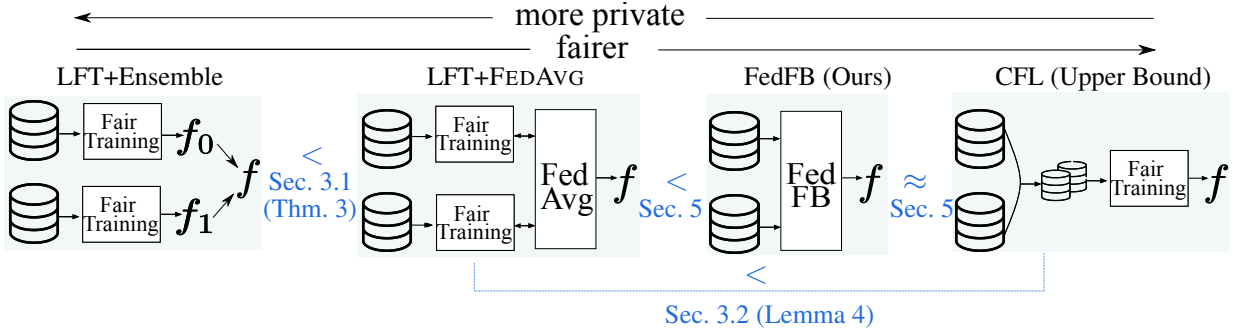

\paragraph{Local Fair Training + Ensemble (\ufl{})} %  and Centralized Fair Learning (CFL)
One na\"ive yet the simplest way to learn fair classifiers from decentralized data without violating privacy policy is training locally fair classifiers and assembling them. 
We call this strategy \ufl{}. 
This is more private than any data sharing schemes or federated learning approaches as only the fully trained models are shared just once.

% To evaluate the performance of this approach, we consider the randomized classifier that makes a prediction using a randomly chosen local classifier. 
% Note that this can be viewed as a random customer model, \emph{i.e.}, a user drawn from the overall data distribution picks and visits one of the institutions, uniformly at random.

\paragraph{Local Fair Training +  \fedavg{} (\fflfedavg{})}
\fflfedavg{} applies \emph{federated learning}~\citep{konevcny2016federated} together with existing fair learning algorithms. 
Federated learning is a distributed learning framework, using which many data owners can collaboratively train a model under the orchestration of a central server while keeping their data decentralized.
For instance, under \fedavg{}, the de facto aggregation protocol for federated learning, the central server periodically computes a weighted average of the locally trained model parameters.
If each data owner runs a fair learning algorithm on its own data and these locally trained models are aggregated via \fedavg{}, then we call this approach \emph{Local Fair Training +  \fedavg{}} (\fflfedavg{}). 

\paragraph{Goal and Main Contributions}
As \fflfedavg{} makes more than one round of communications between the clients and the centralized aggregator, it is reasonable to expect a better performance (fairness and accuracy) with \fflfedavg{} over \ufl{}.
However, rigorous performance analysis has not been made yet in the literature.
Inspired by this, we first study the following question:
\begin{center}
    \emph{Is federated learning needed for decentralized fair learning?}
\end{center}
 \begin{wrapfigure}{r}{0.35\textwidth}
  \begin{center}
    \includegraphics[width=0.35\textwidth]{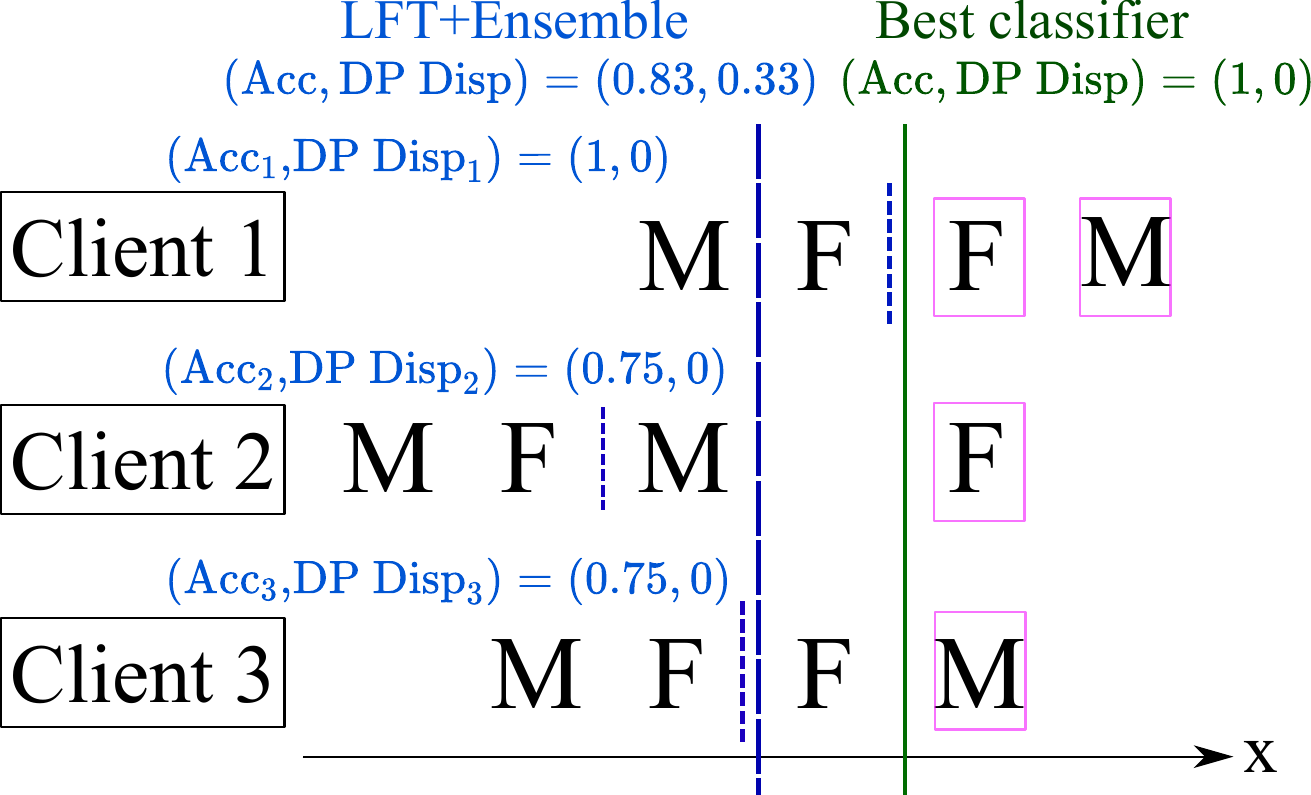}
  \end{center}
\caption{
An example illustrating the failure of \ufl{} (Local Fair Training + Ensemble).
Consider three clients with four samples of one-dimensional feature $x$. 
F and M denote female and male, respectively, and pink boxes denote positive labels.
Under \ufl{}, every client first finds a locally fair classifier on their data, which are depicted in short blue dotted lines. 
The voting ensemble model is visualized in a long blue line.
Note that this is strictly worse than the green line that is more fair and accurate. 
    }
        \label{fig:ufl}
        \vspace{-.2in}
\end{wrapfigure}
In other words, it is not clear whether there is any strict performance (fairness and accuracy) improvement by using \fflfedavg{} over \ufl{}.
If not, there is no reason to use \fflfedavg{} as \ufl{} is a more private scheme.
Our first contribution is to theoretically compare their performances and show that federated learning is indeed necessary. 
Intuitively, if the data is highly heterogeneous, an ensemble model of locally trained models will not be fair even if the consisting models are locally fair.
See Fig.~\ref{fig:ufl} for a toy example illustrating this phenomenon.
On the other hand, \fflfedavg{} might find a globally fairer model by making frequent communications, of course at the cost of privacy loss.
Another natural question follows:

    \emph{How good is \fflfedavg{} compared to the case where the whole data is centralized?}

That is, can \fflfedavg{} achieve similar performance with \emph{Centralized Fair Learning (CFL)}?
If not, can we develop a better federated learning approach?
Our second contribution is to identify the performance gap between \fflfedavg{} and CFL and to develop \fedfb{}, a new federated learning algorithm that bridges the gap.
Our major contributions are summarized as follows (see Fig.~\ref{fig:fldemo}):
\begin{itemize}[itemsep=1pt,topsep=0pt,parsep=0pt,partopsep=0pt,leftmargin=10pt]
    \item We develop a theoretical framework for analyzing approaches for decentralized fair learning.
Using this, we show the strict ordering between the existing approaches, \ie, under certain conditions, \ufl{} $<$ \fflfedavg{} $<$ CFL, w.r.t. their fairness-accuracy tradeoffs. 
To the best of our knowledge, it is the first rigorous comparison of these approaches. 
Our finding implies the necessity of federated learning, but at the same time, implies the limitation of the nai\"ve application of FedAvg together with local fair training.
    \item Leveraging the state-of-the-art algorithm for (centralized) fair learning~\citep{roh2021fairbatch}, we design \fedfb{}, a novel fairness-aware federated learning algorithm. 
    \item Via extensive experiments, we show that (1) our theoretical findings hold under more general settings, and (2) \fedfb{} significantly outperforms the existing approaches and achieves similar performance as CFL.
\end{itemize}

% Our proposed solution \fedfb{} achieves state-of-the-art performance on many datasets, sometimes achieving a similar tradeoff as the one trained on centralized data. 

\section{Related work}
\paragraph{Model Fairness} Several works have studied the fundamental tradeoff between fairness and accuracy ~\citep{menon2018cost,wick2019unlock,zhao2019representation}. 
Our analysis is highly inspired by the one by \citet{menon2018cost}, which characterized the tradeoff between accuracy and fairness under the centralized setting.
Among various fair training algorithms~\citep{zemel2013learning,jiang2020biascorrecting, Zafar2017FC,Zafar2017Mistreatment,Hardt2016Equality}, our work is based on the current state of the art — FairBatch~\citep{roh2021fairbatch}, which we will detail in Sec.~\ref{sec:fedfb}

\paragraph{Federated Learning} Unlike traditional, centralized machine learning approaches, federated learning keeps the data decentralized throughout training, reducing the privacy risks involved in traditional
approaches~\citep{konevcny2016federated,McMahan17FedAvg}. 
\fedavg{}~\citep{McMahan17FedAvg} is the first and most widely used federated learning algorithm. 
The idea is to iteratively compute a weighted average of the local model parameters, with the weights proportional to the local datasets' sizes. 
Prior work~\citep{li2020on} has shown that \fedavg{} provably converges under some mild conditions.
The design of our proposed algorithm \fedfb{} is also based on that of \fedavg{}.

\paragraph{Federated Fair Learning for Group Fairness} A few very recent studies~\citep{ezzeldin2021fairfed,rodriguez2021enforcing,chu2021fedfair,du2020fairfl,cui2021ffl}, conducted concurrently with our work, aim at achieving group fairness under federated learning.
In particular, \citet{du2020fairfl}, \citet{rodriguez2021enforcing} and \citet{chu2021fedfair} mimic the centralized fair learning setting by very frequently exchanging information for each local update, not for each round of local training. 
% While our \fedfb{} applies to various popular group fairness notions, \citet{rodriguez2021enforcing} is limited to false-negative parity and accuracy parity, and \cite{chu2021fedfair} only works for equal opportunity.
% Moreover, since \fedfb{} is based on \fedavg{}, \fedfb{} requires much less communication costs. 
In contrast, \fedfb{} requires much less frequent communications (one per round), resulting in higher privacy and lower communication costs.
Moreover, \citet{rodriguez2021enforcing} and \citet{chu2021fedfair} are designed specifically for certain group fairness notions (such as equal opportunity and accuracy parity) and cannot be applied for other definitions such as demographic parity.  
On the other hand, \fedfb{} is applicable for all popular group fairness notions including demographic parity, equalized odds, and equal opportunity.
Similar to \fedfb{}, \citet{ezzeldin2021fairfed} employs \fedavg{} and a reweighting mechanism, but it achieves a worse performance.

 \begin{wrapfigure}{r}{0.5\textwidth}
  \begin{center}
    \includegraphics[width=.5\textwidth]{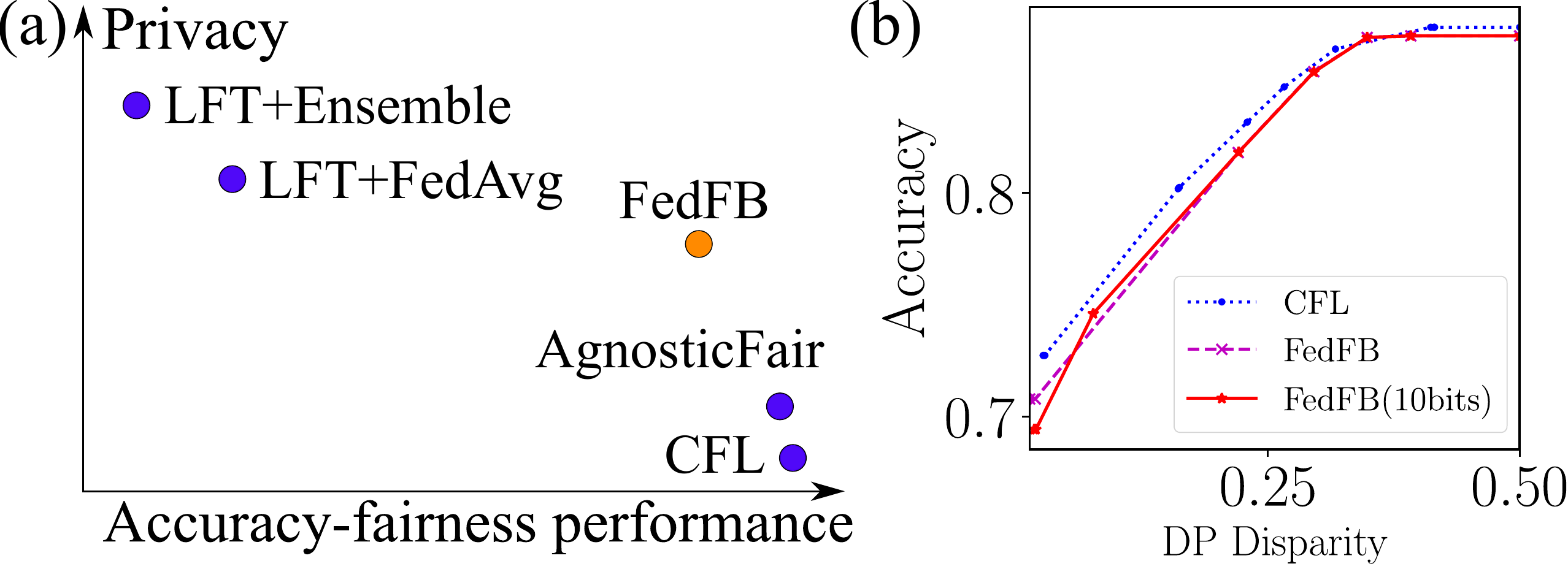}
  \end{center}
  \caption{(a) A high-level comparison of various fair learning methods. (b) Accuracy-fairness tradeoff curves on the synthetic dataset. \fedfb{} nearly matches the performance of CFL.
}\label{fig:comparison}
        \vspace{-.2in}
\end{wrapfigure}
We visualize a high-level comparison of \ufl{}, \fflfedavg{}, CFL, \fedfb{} and \agnosticfair{}~\citep{du2020fairfl} in terms of performance and privacy in Fig.~\ref{fig:comparison}(a).
% This is in contrast to our work, which instead focuses on achieving global fairness in the overall data distribution. 
% Our setting is more appropriate in domains such as criminal justice and social welfare. 
Most of the work in the literature and our work focus on satisfying a single fairness requirement on the ``global'' data distribution, but there is also a recent work that aims at finding a classifier that simultaneously satisfies multiple fairness requirements, each defined with respect to each client's ``local'' data~\citep{cui2021ffl}.

\paragraph{Federated Fair Learning for Client Fairness}
% There have been only a few attempts in achieving fairness under the federated setting.
% Moreover, 
The definition of fairness used in most of the existing federated learning work is slightly different from the standard notion of group fairness.
One popular definition of fairness in the federated setting is \emph{client parity (CP)}, which requires that all clients (\emph{i.e.} data owners) achieve similar accuracies (or loss values).
Several algorithms have been proposed to achieve this goal~\citep{li2021ditto,li2020fair,mohri2019agnostic,yue2021gifair,zhang2020fairfl}. 
Even though their goal is different from ours, we show that an adapted version of \fedfb{} can also achieve comparable client parity, compared with the existing algorithms proposed specifically for CP.
% In Sec.~\ref{sec:experiments}, we will show that \fedfb{} can achieve as good client parity as the existing algorithms, though \fedfb{} is not specifically designed for client parity. 

\section{Analysis of LFT{\small +}Ensemble \& LFT{\small +}FedAvg} \label{sec:thy}

In this section, we first show the necessity of federated learning by proving that \fflfedavg{} achieves strictly higher fairness than \ufl{}. 
We then prove the limitation of \fflfedavg{} by comparing its performance with an oracle bound of CFL (non-private). 
These two results together imply that federated learning is necessary, but there exists a limit on what can be achieved by \fedavg{}-based approaches. 
We will present informal theoretical statements, deferring the formal versions and proofs to Sec.~\ref{appendix:UFL_FFL_CFL}. 

\paragraph{Problem Setting} Denote $[N]:=\{0,1,\dots,N-1\}$ for any $N\in \mathbb{Z}^+$. We assume $I$ clients, which have the same amount of data.
We further assume a simple binary classification setting with a binary sensitive attribute, \emph{i.e.}, $\rvx\in\cX = \bR$ is the input feature, $\rvy \in \cY = [1]$ is the outcome and $\rva \in \cA = [A] = [1]$ is the binary sensitive attribute.
Assume that $\rvx$ is a continuous variable. 
% While we will develop our theory focusing on this simple setting, we believe that our analysis can be extended to more general settings, and we leave such an extension as future work.
The algorithm we will develop later is applicable to general settings.

We now introduce parameters for describing the data distribution.
Let $\rvy\mid \rvx \sim \Ber(\eta(\rvx))$ for all client $i$, where $\eta(\cdot): \cX \rightarrow [0,1]$ is a strictly monotone increasing function. 
Assume $\rvx \mid \rva = a , \rvi = i \sim \cP_a^{(i)}$, $\rva \mid \rvi = i \sim \Ber(q_i)$, where $\rvi$ is the index of the client, $\cP_a^{(i)}$ is a distribution, and $q_i \in [0,1]$ for $a=0,1,i\in [I]$. Let $\cF = \{f: \cX \times \cA \rightarrow [0,1]\}$. 
Given $f\in\cF$ and data sample $(\rvx, \rva)$, we consider the following randomized classifier: $\haty \mid \rvx, \rva \sim \Ber(f(\rvx, \rva))$. 
Using these definitions, we now define \emph{demographic parity}, specialized for a binary case as 
$$\textrm{(Demographic parity)}:~\bP(\haty = 1\mid \rva = 0) = \bP(\haty = 1\mid \rva = 1).$$
To measure how \emph{unfair} a classifier is with respect to demographic parity (DP), we define \emph{DP disparity} as:
$$\text{DP Disp}(f) = |\bP(\haty = 1\mid \rva = 0) - \bP(\haty = 1 \mid \rva = 1)|.$$

\begin{figure}
    \centering
    \includegraphics[width=\textwidth]{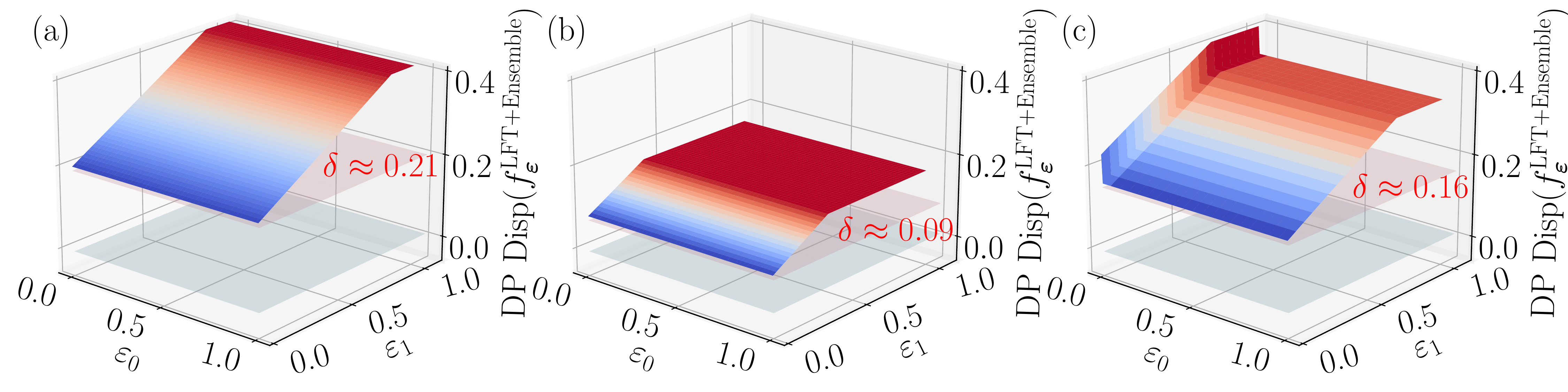}
    \caption{\textbf{Fundamental limitation of \ufl{} in terms of fairness range.}
    DP disparity of \ufl{} as a function of local unfairness budgets. 
    The blue plane is of perfect fairness, and the pink plane visualizes the value of $\delta$ (the lowest DP disparity that \ufl{} can achieve). 
    (a) (Example~\ref{example}) On a distribution that satisfies the conditions of Corollary~\ref{cor:informal_gaussian}.
    (b, c) On distributions illustrated in Sec.~\ref{exp:ufl_limit}, which do not satisfy the conditions of Corollary~\ref{cor:informal_gaussian}.
    In all three cases, the \ufl{} cannot achieve perfect fairness, \ie, $\delta > 0$. 
    }
    \label{fig:figure1}
\end{figure}

\subsection{\fflfedavg{} strictly outperforms \ufl{}}
\label{sec:necessity}
% We now study the necessity of federated learning for achieving high fairness by identifying a strict gap between the performance tradeoff achieved by \ufl{} and that achieved by \fflfedavg{}.

\paragraph{\ufl{} Optimization}
We now present the optimization problem solved in the \ufl{} scenario.
Here, for analytical tractability, we will assume the population limit, \emph{i.e.}, the true data distribution is used in optimization.
Recall that in this scenario, each client solves their own optimization problem to train a locally fair classifier. 
In particular, each client $i\in [I]$ first sets its own local fairness constraint $\veps_i \in [0,1]$, and solves the following problem:
\begin{align}\tag{\ufl{}($i,\veps_i$)}\label{prob:ufl}
         &\min_{f\in\cF} \bP(\haty \neq \rvy \mid \rvi = i), \\
         & \st.~|\bP(\haty = 1\mid \rva = 0, \rvi = i) - \bP(\haty = 1 \mid \rva = 1, \rvi = i)| \leq \veps_i.
\end{align}
We denote by $f_i^{\veps_i}$ the solution to~\ref{prob:ufl}. 
Let $\boldsymbol{\veps} = (\veps_0,\dots,\veps_{I-1})$.
We consider the following ensemble of $I$ locally trained classifiers:
\begin{align}
    \haty \mid \rvx, \rva \sim  \Ber\bigg(\sum_{i\in [I]}f_i^{\veps_i}(\rvx, \rva)/I \bigg) = \Ber(\fuflm),
\end{align}
where $\fuflm := \sum_{i\in [I]}f_i^{\veps_i}(\rvx, \rva)/I$. 

% Note that $\fufl$ is equivalent to directly averaging $f_0^{\veps_0}$ and $f_1^{\veps_1}$. 
Since there always exists a perfectly fair classifier~\citep{menon2018cost}, one question we can ask is if $\fuflm$ can achieve an arbitrary level of fairness. 
The following lemma asserts that it cannot.% \ufl{} cannot achieve a high enough fairness level beyond a certain threshold. 
\begin{lemma}[(informal) Achievable fairness range of \ufl{}]\label{lemma:informal_ufl_limit}
Let $q_i = q$ for all $i \in [I]$, where $q\in (0,1)$ is a constant. Under certain conditions, there exists $\delta$ ($\delta >0$) such that
$\min_{\boldsymbol{\veps} \in [0,1]^I} \disp(\fuflm) > \delta.$
\end{lemma}
Lemma~\ref{lemma:informal_ufl_limit} implies that \ufl{} fails to achieve perfect fairness even when the sensitive attribute ($\rva$) distribution is identical across the clients.
Instead, Lemma~\ref{lemma:informal_ufl_limit} only requires the insensitive attribute distribution ($\rvx$) to be highly heterogeneous across the clients.
The following corollary provides a concrete example satisfying such conditions.
\begin{corollary}[informal]\label{cor:informal_gaussian}
Let $I=2$ and $q_0 = q_1=0.5$, $\eta(x) = \frac{1}{1+\text{e}^{-x}}, \cP_a^{(i)} = \cN(\mu_a^{(i)}, \sigma^{(i)2})$, where $i = 0,1$. 
If one client has much larger variance than the other, under certain assumptions, there exists $\delta >0$ such that $\min_{\veps_0, \veps_1 \in [0,1]} \disp(\fuflm) > \delta.$
\end{corollary}
Note that the condition that one client has a much larger variance than the other contributes to the ``high data heterogeneity" requirement. 
We provide the explicit form of $\delta$ in Corollary~\ref{cor:gaussian} in Sec.~\ref{appendix:ufl_limit}. Corollary~\ref{cor:informal_gaussian} implies that under a limiting case of Gaussian distribution, \ufl{} cannot achieve high fairness requirements. 
In Sec.~\ref{sec:experiments}, we will numerically demonstrate the same claim holds for more general cases.
Next, we give an example that satisfies the conditions of Corollary~\ref{cor:informal_gaussian}, whose achievable DP disparity and lower bound are visualized in Fig.~\ref{fig:figure1}(a). 
\begin{example}\label{example}
Let $\mu_0^{(0)} = \mu_1^{(0)} = 0, \mu_0^{(1)} = 3, \mu_1^{(1)} = -1, \sigma^{(0)} = 70$ and $\sigma^{(1)} = 1$. Then, $\delta \approx 0.21$.
\end{example}
\begin{figure}[t]
    \centering
    \includegraphics[width=\textwidth]{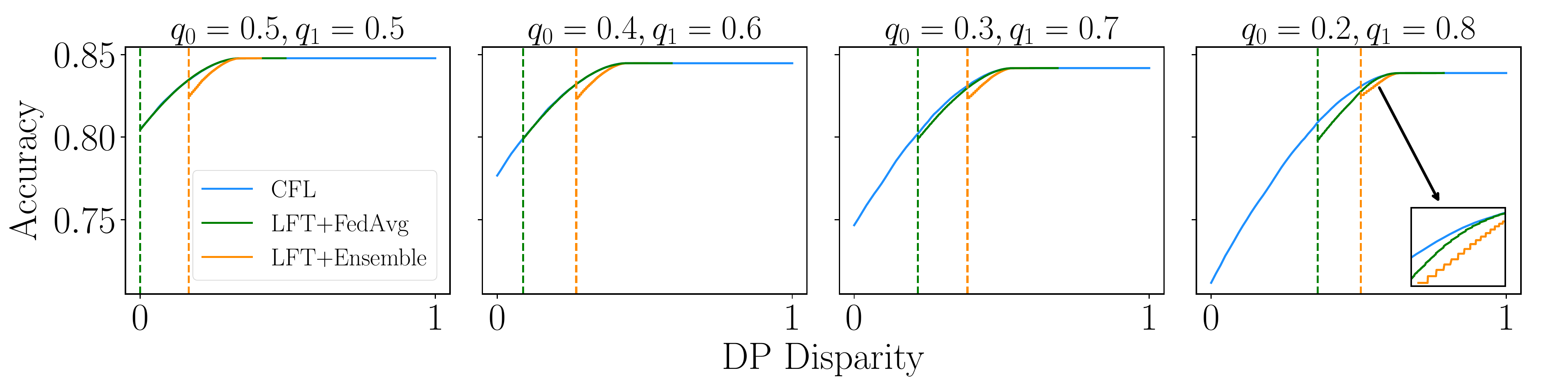}
    \caption{\textbf{Accuracy-fairness tradeoff curves of CFL, \fflfedavg{}, and \ufl{} for two clients cases.} 
    As $q_i$ denotes the sensitive attribute distribution for  client $i$, $|q_1 - q_0|$ captures the sensitive attribute distribution's heterogeneity. 
    % The data heterogeneity increases from left to right. 
    % We let $\rvx \mid \rva = 0, \rvi = 0 \sim \cN(3, 1), \rvx \mid \rva = 1, \rvi = 0 \sim \cN(5, 1), \rvx \mid \rva = 0, \rvi = 1 \sim \cN(1, 1), \rvx \mid \rva = 1, \rvi = 1 \sim \cN(-1, 1), \rva \mid \rvi = 0\sim \Ber(q_0), \rva \mid \rvi=1 \sim \Ber(q_1)$. 
    The dotted vertical lines describe the lowest DP disparities that can be achieved.
    \fflfedavg{}'s maximum fairness is strictly higher than that of \ufl{} but strictly lower than that of CFL. Moreover, the tradeoff curves are strictly ordered in the same order. 
    }

    \label{fig:figure2}
\end{figure}

\begin{remark}
In this work, we classify local training followed by an ensemble method as a non-federated learning algorithm. However, one may also view it as an extreme instance of federated learning, called one-shot (single round) federated learning~\citep{guha2019one}. 
Depending on how one would classify this approach, our finding can be viewed as either the necessity of federated learning or the necessity of multi-round federated learning.
\end{remark}

\paragraph{\fflfedavg{} Optimization}
% \fflfedavg{} enables federated training of a fair classifier on decentralized data. % In Sec.~\ref{sec:fflbestclassifier}, 
The classifier obtained by \fflfedavg{} can be shown to be equivalent to $\ffflm$, the solution to the following problem.
    \begin{align}\tag{\fflfedavg{}$(\boldsymbol{\veps})$}\label{prob:ffl}
         &~~~~~~~~~~~~~~~~~~~~~\min_{f\in\cF}  \bP(\haty \neq \rvy)  \\
         &\st.~|\bP(\haty = 1\mid \rva = 0, \rvi = i) - \bP(\haty = 1 \mid \rva = 1, \rvi = i)| \leq \veps_i .
    \end{align}
The proof of the equivalence is provided in Sec.~\ref{sec:fflbestclassifier}.
The following theorem asserts that \fflfedavg{} can achieve a \emph{strictly} higher fairness than \ufl{}.
% in terms of fairness. 
\begin{theorem}[(informal) Fundamental values of federated learning for fairness]\label{thm:uflvsffl}
Let $q_i=q $ for all $i\in [I]$, where $q\in (0,1)$ is a constant. Under certain conditions, $\min_{\boldsymbol{\veps}\in [0,1]^I} \disp (\fuflm) > \min_{\boldsymbol{\veps} \in [0,1]^I} \disp(\ffflm).$ 
\end{theorem}
Thm.~\ref{thm:uflvsffl} implies that \ufl{} cannot match the highest fairness achieved by \fflfedavg{}.
See Fig.~\ref{fig:figure2} for the performance tradeoffs.
Similar to Lemma~\ref{lemma:informal_ufl_limit}, the technical assumptions include high enough heterogeneity of $\rvx$ across the clients (see Sec.~\ref{appendix:multi_clients}). 
% More details are provided in  Sec.~\ref{appendix:multi_clients}. 
% The theorem asserts that, under certain distributional assumptions, by using the optimal local fairness budgets $\veps_i$, \fflfedavg{} can achieve perfect fairness, while \ufl{} cannot. 
\begin{remark}
Extending Thm.~\ref{thm:uflvsffl} (and the analysis of \fflfedavg{}) to general cases where $q_i$'s are different remains open.
However, we conjecture that our lemma holds for more general cases, and we will numerically support our conjecture empirically. 
\end{remark}
\begin{remark}[]
Thm.~\ref{thm:uflvsffl} shows that even with \emph{just} two clients ($I=2$), a constant gap exists between non-federated algorithms and federated algorithms in their fairness performances. 
There is a stark difference between this phenomenon and the well-known gain of federated learning due to an increased sample size, which is generally negligible with a few number of clients. 
% sample complexity gain as 
% there exists even small-scale federation can significantly boost the model fairness. 
Our finding on this untapped gain in fairness can be seen as a new justification for small-scale federated learning, i.e., cross-silo settings.
% It suggests that data diversity is the main reason why we should use federated learning, instead of the increased number of samples. 
% The most evident gain of federated learning is the increased number of samples. 
% From the statistical learning view, this benefit of federated learning might be significant if the number of clients $I$ is large, but this gain will be marginal if $I$ is small. 
% This is particularly the case when federated learning is deployed in a cross-silo setting, where only a handful of organizations may participate in training~\citep{bonawitz2019towards}. 
\end{remark}

\subsection{\fflfedavg{} is strictly worse than CFL}\label{sec:4.1}
% We emphasize that our findings in the previous section does \emph{not} imply that \fflfedavg{} completely closes the problem. 
% The theorem say nothing about the accuracy achieved by the perfectly fair classifier obtained by
% While we showed that \fflfedavg{} can achieve perfect fairness on certain distributions, it is still unclear whether or not this is the case for \emph{every} distribution. 
We first present the optimization problem for CFL, whose achievable fairness region can serve as an upper bound on that of all other decentralized algorithms.
We then show the existence of data distributions on which \fflfedavg{} achieves a strictly worse fairness performance than CFL. 
% This implies a strict gap between the performance tradeoff of \fflfedavg{} and that of CFL.

\paragraph{CFL Optimization} 
We consider the same problem setting as last subsection. 
% To study the fundamental limitation of \fflfedavg{}, 
We now formulate the CFL problem as:
\begin{align}\tag{CFL($\veps$)}\label{prob:cfl}
         &~~~~~~~~~~~~~~~~~~~~~~\min_{f\in\cF} \bP(\haty \neq \rvy), \\ & \st.~|\bP(\haty = 1\mid \rva = 0) - \bP(\haty = 1 \mid \rva = 1)| \leq \veps. 
\end{align}
Denote the solution to~\ref{prob:cfl} as $\fcfl$. 
It is clear that $\fcfl$ achieves the best accuracy-fairness tradeoff, at the cost of no privacy. 
% We consider different conditions from Thm.~\ref{thm:uflvsffl}, and provide 
The following lemma shows that there exists some distribution such that \fflfedavg{} is strictly worse than CFL when the distribution of sensitive attribute $\rva$ is heterogeneous ($q_i$ are not all the same).
\begin{lemma}[(informal) A strict gap between \fflfedavg{} and CFL]\label{lemma:ffl_limit}
When there exist $i\neq j \in [I]~\st. ~q_i \neq q_j$,
there exists a distribution such that 
$\min_{\boldsymbol{\veps}\in [0,1]^I} \disp (\ffflm) > \min_{\veps} \disp(\fcfl) = 0.$
\end{lemma}
\begin{remark}
A strict gap exists for \emph{certain} distributions, but \emph{not for all} distributions.
\end{remark}

\subsection{Numerical Comparisons of tradeoffs}\label{sec:num_comparisons_33}

One limitation of our current theoretical results is that they only compare the maximum achievable fairness.
% Note that such analysis reveals how the tradeoff of accuracy and fairness behaves as the fairness level increases, but it fails at fully characterizing the entire tradeoff curve.  
Extending our theoretical results to fully characterize two-dimensional tradeoffs is non-trivial, so we leave it as future work.
Instead, we numerically solve each of the optimization problems and visualize the tradeoff curves.

Shown in Fig.~\ref{fig:figure2} are the tradeoff curves for two clients cases.
% Here, we consider two clients cases.
We let $\rvx \mid \rva = 0, \rvi = 0 \sim \cN(3, 1), \rvx \mid \rva = 1, \rvi = 0 \sim \cN(5, 1), \rvx \mid \rva = 0, \rvi = 1 \sim \cN(1, 1), \rvx \mid \rva = 1, \rvi = 1 \sim \cN(-1, 1), \rva \mid \rvi = 0\sim \Ber(q_0), \rva \mid \rvi=1 \sim \Ber(q_1)$, $\eta(x) = \frac{1}{1+e^{-x}}$, and vary the values of $q_0$ and $q_1$.
Note that $|q_1 - q_0|$ captures the heterogeneity of the sensitive data $\rva$, which increases from left to right. 
First, one can observe that \ufl{}$<$\fflfedavg{}$<$ CFL in terms of the achievable fairness range, as predicted by our theory. 
Beyond our theory, we also observe an increasing gap between the tradeoff curves as the data heterogeneity increases. % Theoretical understanding of this phenomenon remains open.

\section{Proposed Algorithm: FedFB}\label{sec:fedfb}

Our findings in Sec.~\ref{sec:thy} imply that federated learning is necessary, but the current \fedavg{}-based approach might not be the best approach. 
Can we design a federated learning algorithm that is strictly better than \fedavg{}-based approaches?
In this section, we propose a new federated learning algorithm for fair learning, which we dub \fedfb{} (short for Federated FairBatch). 
Our approach is based on the state-of-the-art (centralized) fair learning algorithm \fairbatch{} (\fb{} for short)~\citep{roh2021fairbatch} and comes with a few desirable theoretical guarantees. 
% Later in Sec.~\ref{sec:experiments}, we empirically show that \fedfb{} outperforms \fflfedavg{} and closely matches the CFL's tradeoff on various datasets. 

\paragraph{Centralized \fb{}} Let us first review how \fb{} works in the centralized setting when demographic parity is considered. 
Assume there are $A$ sensitive attributes. 
% We first describe how we improve \fairbatch{} in the centralized setting. 
Let $\boldsymbol{w}$ be the model parameters, $\ell(\rvy, \haty; \boldsymbol{w})$ be the loss function, $n_{y,a} := |\set{(\rvx, \rvy, \rva): \rvy = y, \rva = a}|$ be the number of samples in group $a$ with label $y$, and $n_{\star,a} := |\set{(\rvx, \rvy, \rva):\rva = a}|$ be the number of samples in group $a$. 
Let $L_{y,a}(\boldsymbol{w}) = \sum_{\rvy = y, \rva = a} \ell(\rvy, \haty; \boldsymbol{w})$, $L^\prime_{y,a}(\boldsymbol{w}) := n_{y,a} L_{y,a}(\boldsymbol{w})/n_{\star,a}$.
Then, for $a \in [A]$, we define $$F_a(\boldsymbol{w}) := -L_{0,0}^\prime(\boldsymbol{w})+L_{1,0}^\prime(\boldsymbol{w})+ L_{0,a}^\prime(\boldsymbol{w})-L_{1,a}^\prime(\boldsymbol{w})+ \tfrac{n_{0,0}}{n_{\star,0}} - \tfrac{n_{0,a}}{n_{\star,a}}.$$
Then, we can show the following proposition.
\begin{proposition}[Necessary and sufficient condition for demographic parity]\label{prop:dp}
Consider 0-1 loss: $\ell(\mathrm{y},\hat{\mathrm{y}}; \boldsymbol{w}) = \indicator{\mathrm{y}\neq \hat{\mathrm{y}}(\boldsymbol{w})}$, where $\indicator{\cdot}$ is the indicator function. Then, 
\begin{align} \label{eq:dp}
F_a(\boldsymbol{w}) = 0, \forall a\in [A]
\end{align}
is a necessary \& sufficient condition for demographic parity.
\end{proposition}

Note that Proposition~\ref{prop:dp} allows us to measure demographic disparity using subgroup-specific losses when the 0-1 loss is used. 
Inspired by this observation, \fb{} uses $F_a(\boldsymbol{w})$ as a surrogate of the unfairness measure even for non-$0-1$ loss functions.
With this surrogate, \fb{} solves the following bi-level optimization. 
% Let $\boldsymbol{\lambda}=(\lambda_0,\dots,\lambda_{A-1})$ be the weights attached to different subgroups. We formulate the reweighting task into the following bi-level optimization problem.
\begin{align}\label{dp:w}
& ~~~~~~~~~~\min_{\boldsymbol{\lambda}\in[0, 2\frac{n_{\star,0}}{n}]\times \dots \times[0, 2\frac{n_{\star,A-1}}{n}]} \sum_{a=1}^{A-1}\big(F_a(\boldsymbol{w}_{\boldsymbol{\lambda}})\big)^2\\
& \boldsymbol{w}_{\boldsymbol{\lambda}} = \arg\min_{\boldsymbol{w}} \sum_{a=0}^{A-1}\left[\lambda_a L_{0,a}^\prime(\boldsymbol{w}) + (2\frac{n_{\star,a}}{n}-\lambda_a)L_{1,a}^\prime(\boldsymbol{w})\right],   
\end{align}
Here, $\boldsymbol{\lambda}=(\lambda_0,\dots,\lambda_{A-1})$ are the outer optimization variables, which denote adaptive coefficient to loss occurring from each subgroup. 
Given $\boldsymbol{\lambda}$, the inner optimization problem minimizes a reweighted loss function.
The intuition here is that if one carefully chooses the coefficients for each subpopulation loss term, the weighted risk minimization will result in a fair classifier, even if it is unconstrained.

In particular, \citet{roh2021fairbatch} proposed the following iterative algorithms to solve the outer optimization problem:
% The algorithm reduces to the following simple yet intuitive algorithm. 
% The algorithm starts with equal weights for two different groups. 
% After training a model with the initial weights, it computes the sign of the difference between the two positive prediction rates $\bP(\haty = 1\mid \rva = 0) - \bP(\haty = 1 \mid \rva = 1)$. 
% If this quantity is zero, then $\disp$ is zero, so the sample weights are not updated. 
% If this is positive, it decreases the weights for the samples whose $\rvy = 1, \rva = 0$ and increases the weights for the samples whose $\rvy = 1, \rva = 1$ so that after retraining, $\bP(\haty = 1\mid \rva = 0)$ decreases and $\bP(\haty = 1 \mid \rva = 1)$ increases. 
% And vice versa for the other case.
\begin{align}\label{eq:update_dp}
 & \mu_0(\boldsymbol{\lambda}) = - \sum_{a=1}^{A-1}F_a(\boldsymbol{w}_{\boldsymbol{\lambda}}), 
 \mu_a (\boldsymbol{\lambda}) = F_a(\boldsymbol{w}_{\boldsymbol{\lambda}}), \text{ for } a\neq 0,\\
 & \lambda_a \leftarrow \lambda_a + \frac{\alpha}{\Vert\boldsymbol{\mu}(\boldsymbol{\lambda})\Vert_2}\mu_a (\boldsymbol{\lambda}), \text{ for } a \in [A], \alpha\text{ is the step size}.
\end{align}
% The update direction of $\boldsymbol{\lambda}$ is $\boldsymbol{\mu}(\boldsymbol{\lambda}) = (\mu_0(\boldsymbol{\lambda}), \dots, \mu_{A-1}(\boldsymbol{\lambda}))$.
All together, \fb{} alternates this update equation (for the outer optimization) and the inner-loop updates (e.g., via SGD) to solve the bi-level optimization problem\footnote{Technically speaking, the algorithm shown here is a slightly generalized version of the original \fb{}.
In particular, it works for more than two sensitive groups, while the original \fb{} algorithm only works for binary groups.
However, we will still call ours \fb{}.
}. 

\paragraph{\fedfb{}: Federated Learning + \fb{}}
Recall that under the \fedavg{} protocol, clients periodically share their locally trained model parameters with the server. 
Our proposed algorithm is based on the following simple observation:

\emph{The bi-level structure of \fb{} naturally fits the hierarchical structure of federated learning.}

To see this, observe that the update equation given in \eqref{eq:update_dp} only requires the knowledge about all $\mu_a$, which is the function of $F_a$'s.
By definition of $F_a$, it is the sum of the locally computed $F_a$ values across all the clients.
Thus, this can be securely collected by the central aggregator via secure aggregation.
% note that if the clients also share their subgroup-specific losses, then the centralized server can immediately reconstruct the difference between the two positive prediction rates, measured on the entire data distribution. 
Once the $F_a$ values are securely aggregated at the central server, the central aggregator can update the per-group coefficients $\boldsymbol{\lambda}$ using the exact same equation \eqref{eq:update_dp}.

This allows us to run the \fb{} algorithm even on decentralized data.
More specifically, we modify the FedAvg protocol so that each client shares not only its intermediate model parameters but also its $\{F_a\}_{a \in [A]}$ with the central coordinator.
Once the central server receives the securely aggregated model parameters and $F_a$ values, it performs both (1) model averaging and (2) reweighting coefficients updates ($\boldsymbol{\lambda}$).
The server then broadcasts the averaged model parameter together with the updated coefficients, which are then used for the subsequent round of local training with a reweighted loss function.
Note that we update $\boldsymbol{\lambda}$ every $k$ communication rounds.
\begin{algorithm2e}[t]
 \SetKwBlock{Server}{ServerExecutes:}{ }
 \SetKwBlock{Compute}{ClientUpdate$(i,\boldsymbol{w}, \boldsymbol{\lambda})$:}{}
  \SetKwInOut{Output}{output}

  \Compute{
 Update local $\boldsymbol{w}^{(i)}$ according to sample weights $\boldsymbol{\lambda}$\;
 Compute local $F_a^{(i)}$ for all $a \in [A]$\;
 % $ L_{y,a}^{(i)}\leftarrow \sum_{(\rvi,\rvy,\rva) = (i,y,a)} \ell(\haty, \rvy;\boldsymbol{w})$, $\forall (y,a)$\;
%  Send $\boldsymbol{w}^{(i)}, F_{a}^{(i)}$ for all $(y,a)$ to server via SecAgg;
 }{}
 
 \Server{
\For{each round $t$}{
Wait until clients perform their updates\;
$\boldsymbol{w} \leftarrow $ SecAgg$(\{\boldsymbol{w}^{(i)}\})$\; % _{i\in[I]}
$\boldsymbol{F}_{a} \leftarrow $ SecAgg$(\{F_{a}^{(i)}\}) - (I-1)(\frac{n_{0,0}}{n_{\star, 0}} - \frac{n_{0,a}}{n_{\star, a}})$\; % _{i,a\in[I]\times[A]}
\uIf{$t \%  k = 0$}{
$\boldsymbol{\lambda} \leftarrow \texttt{Update}(\boldsymbol{\lambda},\boldsymbol{F}_{a})$\;
Broadcast $\boldsymbol{\lambda}$ to clients\;
}
Broadcast $\boldsymbol{w}$ to clients\;
}
\Output{$\boldsymbol{w}, \boldsymbol{\lambda}$}
  }{}
  \caption{\fedfb{}$(k,t)$}
 \label{alg:fedfb}
\end{algorithm2e} 

This algorithm, which we dub \fedfb{}, is formally described in Alg.~\ref{alg:fedfb}.
% Note that the update rule for group weights (denoted by $\boldsymbol{\lambda}$ in the pseudocode) only requires the sum of the group losses, which enables secure aggregation.
% In short, we present the detailed description of the algorithm, which consists of the local training algorithm with reweighted samples, the model/loss aggregation protocol, the group-weight update algorithm, and the model/group-weight distribution protocol.
While this description is specific to the demographic parity, \fedfb{} can be used for other fairness notions including equal opportunity, equalized odds, and client parity.
The only difference would be how to design the $F_a$ function, which is the surrogate for the unfairness measure. 
See Sec.~\ref{appendix:fedfb} for more details. 

\paragraph{Convergence of \fedfb{}} The following theorem shows that \fedfb{} converges to the optimal reweighting coefficients.
The proof is based on the analysis tools for federated learning and \fb{} -- See Thm.~\ref{prop:formal_conv_dp} for more details.
\begin{theorem}[(informal) Partial convergence guarantee of \fedfb{}]
Let the output of \fedfb{}$(k, t)$ be $\boldsymbol{\lambda}_{k,t}$. 
Assume $A=2$.
Then, under certain conditions, $\lim_{t \to \infty}\lim_{k \to \infty}\boldsymbol{\lambda}_{k,t} = \boldsymbol{\lambda}^\star$ for some $\boldsymbol{\lambda}^\star \in \arg\min_{\boldsymbol{\lambda}} \sum_a [F_a(\boldsymbol{w}_{\boldsymbol{\lambda}})]^2$. 
\end{theorem}
% We made appropriate changes to the algorithm (with theoretical guarantees) so that it can also handle more general cases. 
% Thus, we use \fb{} by default for fair learning.
% See Lemma~\ref{lemma:dp} for details.

\paragraph{Additional privacy leakage}
Under \fedfb{}, clients exchange real-valued $F_a$'s with the server in addition to the model parameters. 
To limit the information leakage, we develop a variant of \fedfb{}, which exchanges the quantized loss values.
For instance, ``\fedfb{}(10bits)'' means each loss value is uniformly quantized using 10 bits. Such a quantization scheme limits the amount of additional information shared in communication round, at the cost of perturbed coefficient updates.

\section{Experiments}\footnote{Our implemtation is available in \url{https://github.com/yzeng58/Improving-Fairness-via-Federated-Learning}. }\label{sec:experiments}
In this section, we numerically study the performance of \ufl{}, \fflfedavg{}, and CFL for more general cases, and evaluate the empirical performance of \fedfb{}. In each simulation study, we report the summary statistics across five replications. 
Similar to the experimental settings used in \citep{roh2020frtrain},
we train all algorithms using a two-layer ReLU neural network with four hidden neurons to evaluate the performance of \fedfb{} for the non-convex case.
Due to the space limit, we defer (1) logistic regression results, (2) the performance tradeoffs of \fedfb{} as a function of the number of clients, and (3) more implementation details to  Sec.~\ref{appendix:experiment}.

\subsection{Limitation of \ufl{} on general cases}\label{exp:ufl_limit}
The first experiment examines the fairness range of \ufl{} under various Gaussian distributions, which does not satisfy the conditions of Corollary~\ref{cor:informal_gaussian}.
To do so, we numerically solve the optimization problems~\ref{prob:ufl}.
See Sec.~\ref{appendix:UFL_FFL_CFL} for more details.
% For instance, if the variance of two clients is similar, then the conditions do not hold.
We conjecture that the same phenomenon holds for more general distributions, and corroborate our conjecture with numerical experiments.
Shown in Fig.~\ref{fig:figure1}(b,c) are the numerically computed lower bound on \ufl{}'s achievable fairness.
In particular, for (b), we let $\rvx \mid \rva = 0 , \rvi = 0 \sim \cN(10, 0.2^2), \rvx \mid \rva = 1 , \rvi = 0 \sim \cN(9.8, 0.2^2), \rvx \mid \rva = 0 , \rvi = 1 \sim \cN(0.2, 0.2^2), \rvx \mid \rva = 1 , \rvi = 1 \sim \cN(0, 0.2^2), \rva \sim \Ber(0.2)$, and for (c), we let $\rvx \mid \rva = 0 , \rvi = 0 \sim \cN(3, 1), \rvx \mid \rva = 1 , \rvi = 0 \sim \cN(5, 1), \rvx \mid \rva = 0 , \rvi = 1 \sim \cN(1, 1), \rvx \mid \rva = 1 , \rvi = 1 \sim \cN(-1, 1), \rva \sim \Ber(0.5)$. 
For both cases, we set $\eta(x) = \frac{1}{1+e^{-x}}$. 
It is easy to check that these distributions do \emph{not} satisfy the conditions of Corollary~\ref{cor:informal_gaussian}.
In particular, the distribution (b) corresponds to the case that the same group is favored on both clients, and the positive rates of each group in different clients are distinctive.
The distribution (c) represents the case that different groups are favored on two clients. 
% Fig.~\ref{fig:figure1}(b) and Fig.~\ref{fig:figure1}(c) plot the global fairness violation $\veps$ versus the local unfairness budget $\veps_0, \veps_1$ for \ufl{} for the two distributions, respectively.  
In both cases, \ufl{} fails to achieve perfect fairness, \ie, $\delta > 0$. 
We also observe that $\delta$ is large on the distribution (c).
This supports our conjecture that \ufl{}'s fairness performance is strictly limited not only on certain data distributions but also on more general ones. 
\subsection{Accuracy-fairness tradeoffs of \ufl{}, \fflfedavg{} and CFL}\label{sec:5.2-3clients}

The second experiment extends the experiments conducted in Sec.~\ref{sec:num_comparisons_33}.
We assess the relationship between the data heterogeneity and the gap between the three fair learning scenarios with three clients.
As shown in Fig.~\ref{fig:figure2_3clients}, \fflfedavg{} is observed to achieve a strictly worse tradeoff than CFL and a strictly higher maximum fairness value than \ufl{}. 
% \ufl{} in terms of the range of achievable fairness. 
The results corroborate the benefit and limitation of \fedavg{}-based federated learning in improving fairness.
A very interesting observation is that \ufl{} is observed to obtain a strictly higher accuracy than \fflfedavg{}.
Indeed, this could be attributed to the fact that the average of locally fair models might not be fair to any sub-distribution, while \ufl{} at least ensures that each component of the mixture classifier is fair on some sub-distribution. 

\subsection{FedFB evaluation on demographic parity}\label{sec:exp_fedfb_dp}
\begin{table}[t]
\caption{\textbf{Comparison of accuracy and DP disparity on the synthetic, Adult, COMPAS, and Bank datasets.} \fedfb{} outperforms the other approaches on most of the tested datasets, sometimes nearly matching the performance of CFL. Note that \fflfedavg{} sometimes gets a strictly worse performance than \fedavg{}. This is because the average of fair models may not be fair at all.}
\label{tab:dp_mlp}
\begin{center}
\begin{small}
\begin{sc}
\tiny{
\begin{tabularx}{\textwidth}{ccc@{\extracolsep{\fill}}c@{\extracolsep{\fill}}c@{\extracolsep{\fill}}c@{\extracolsep{\fill}}c@{\extracolsep{\fill}}c@{\extracolsep{\fill}}c@{\extracolsep{\fill}}c@{\extracolsep{\fill}}r}
\toprule
\multicolumn{2}{c}{Property}&& \multicolumn{2}{c}{Synthetic} & \multicolumn{2}{c}{Adult} & \multicolumn{2}{c}{COMPAS} & \multicolumn{2}{c}{Bank}\\
Private&Fair&Method & Acc.($\uparrow$) & DP Disp.($\downarrow$)& Acc.($\uparrow$) & DP Disp.($\downarrow$)& Acc.($\uparrow$) & DP Disp.($\downarrow$)& Acc.($\uparrow$) & DP Disp.($\downarrow$) \\
\midrule

\cmark&\xmark  &\fedavg{} &.886$\pm$.003 & .406$\pm$.009 & .829$\pm$.012& .153$\pm$.022& .655$\pm$.009& .167$\pm$.037&   .898$\pm$.001& .026$\pm$.003 \\ \midrule
\cmark& \cmark&\ufl{}&.727$\pm$.194 & .248$\pm$.194 & .825$\pm$.008& .034$\pm$.028& .620$\pm$.019& .088$\pm$.055&   .892$\pm$.002& .014$\pm$.006 \\
% \centering
\cmark&\cmark &\fflfedavg{}&.823$\pm$.102 & .305$\pm$.131 & .801$\pm$.043& .123$\pm$.071& .595$\pm$.005& \textbf{.059$\pm$.009}&   .893$\pm$.000& .017$\pm$.001 \\
\cmark&\cmark&\textbf{FedFB (Ours)} &.725$\pm$.012 &\textbf{.051$\pm$.018} & .804$\pm$.001& \textbf{.028$\pm$.001}& .606$\pm$.019& .086$\pm$.029&  .883$\pm$.000& \textbf{.000$\pm$.000} \\
\midrule
% \agnosticfair{} & .657$\pm$.029 & .032$\pm$.044 & .767$\pm$.004 & .003$\pm$.005 & .541$\pm$.000 & \textbf{.000$\pm$.000} & .883$\pm$.000 & \textbf{.000$\pm$.000} \\
\xmark&\cmark&CFL&.726$\pm$.009 & .028$\pm$.016 & .814$\pm$.002& .010$\pm$.004& .616$\pm$.033& .036$\pm$.028&   .883$\pm$.000& .000$\pm$.000 \\

\bottomrule
\end{tabularx} }
\end{sc}
\end{small}
\end{center}
\end{table}
We assess the empirical performance of \fedfb{} for non-convex cases on four datasets and the performance of \fedfb{} under different data heterogeneity. 
We focus on demographic parity and report DP disparity $= \max_{a\in[A]}|\bP(\haty = 1 \mid \rva = a) - \bP(\haty = 1)|$, where $A$ is the number of groups.
Note that this is slightly different from the definition we used in the previous sections, which was specific for the binary sensitive attribute.

%, which was used specifically for the case of a binary sensitive attribute.
\paragraph{Baselines} 
We employ three types of baselines: (1) decentralized non-fair training (\fedavg{}); (2) decentralized fair training (\ufl{}, \fflfedavg{}, \fairfed{}, \agnosticfair{}s); (3) centralized fair training (CFL).
% We compare our \fedfb{} with four baselines: \fedavg{}, \ufl{}, \fflfedavg{} and CFL. 
Here, for the fairer comparison and better performance, \ufl{}, \fflfedavg{}, and CFL are implemented based on \fb{}, which is the state-of-the-art fair learning algorithm on centralized data. 
% For CFL, we run one inner loop update per one outer loop update as suggested in \citet{roh2021fairbatch}.
% However, we observed that this setting sometimes does \emph{not} result in the best possible fairness performance, explaining why CFL is observed to achieve lower fairness than the others later in this section on some datasets. 
% One possible explanation is that CFL rearranges the weights after every batch update of model parameters, while the model parameters are not close to $\boldsymbol{w}$ enough. 
Note that \ufl{} is the most private, CFL is the least private, and \fedavg{}, \fflfedavg{}, \fairfed{} and \fedfb{} are somewhere in-between as they share some statistics at each communication round but not directly reveal the data.
\agnosticfair{} is close to CFL, as it exchanges information every local update.
%  and \agnosticfair{} exchanges information for each local update. 
% In particular, \fedavg{} and \fflfedavg{} are slightly more private than \fedfb{}. 
% To have a fairer comparison between \fflfedavg{} and \fedfb{}, we also equalize their differential privacy guarantees ~\citep{dwork2008differential} and compare their performances. See Sec.~\ref{appendix:experiment} for more details. 
% We also report the performance of \fairfed{}, a recently proposed algorithm for achieving demographic parity for binary sensitive groups in the federated setting~\citep{ezzeldin2021fairfed}, and \agnosticfair{} in Sec.~\ref{appendix:experiment}.
\begin{table}[t]
\caption{\textbf{Comparison of accuracy and DP disparity on the synthetic dataset with varying heterogeneity.} \fedfb{} achieves good performance on all the tested levels of heterogeneity.  
This is because \fedfb{} is designed to mimic the operation of CFL, whose performance is independent of data heterogeneity. 
}
\label{tab:dp_dh}
\begin{center}
\begin{small}
\begin{sc}
\begin{tabularx}{\textwidth}{l@{\extracolsep{\fill}}c@{\extracolsep{\fill}}c@{\extracolsep{\fill}}c@{\extracolsep{\fill}}c@{\extracolsep{\fill}}c@{\extracolsep{\fill}}c@{\extracolsep{\fill}}r}
\toprule
& \multicolumn{2}{c}{Low Data Heterogeneity} & \multicolumn{2}{c}{Medium Data Heterogeneity}& \multicolumn{2}{c}{High Data Heterogeneity}\\
Method & Acc.($\uparrow$) & DP Disp.($\downarrow$)& Acc.($\uparrow$) & DP Disp.($\downarrow$)& Acc.($\uparrow$) & DP Disp.($\downarrow$) \\
\midrule
% \fedavg{}   &.886$\pm$.002 & .416$\pm$.004 & .886$\pm$.003 & .406$\pm$.009& .883$\pm$.003& .402$\pm$.018\\ \midrule
% \ufl{} &.728$\pm$.193 & .254$\pm$.195 &.727$\pm$.194 & .248$\pm$.194& .729$\pm$.195& .256$\pm$.193 \\
% \fflfedavg{}&.775$\pm$.081 & .165$\pm$.179 &.823$\pm$.102 & .305$\pm$.131 & .789$\pm$.200& .397$\pm$.065\\
\textbf{FedFB (Ours)} &.669$\pm$.040 & .058$\pm$.042 &.725$\pm$.012 &.051$\pm$.018  & .703$\pm$.012& .013$\pm$.008\\
% \midrule
% \agnosticfair{} & .657$\pm$.029 & .032$\pm$.044 & .657$\pm$.029 & .032$\pm$.044 & .657$\pm$.029 & .032$\pm$.044  \\
CFL&.726$\pm$.009 & .028$\pm$.016 & .726$\pm$.009 & .028$\pm$.016 & .726$\pm$.009 & .028$\pm$.016\\
\bottomrule
\end{tabularx}
\end{sc}
\end{small}
\end{center}
\end{table}

 \begin{wrapfigure}{r}{0.35\textwidth}
  \begin{center}
    \includegraphics[width=0.35\textwidth]{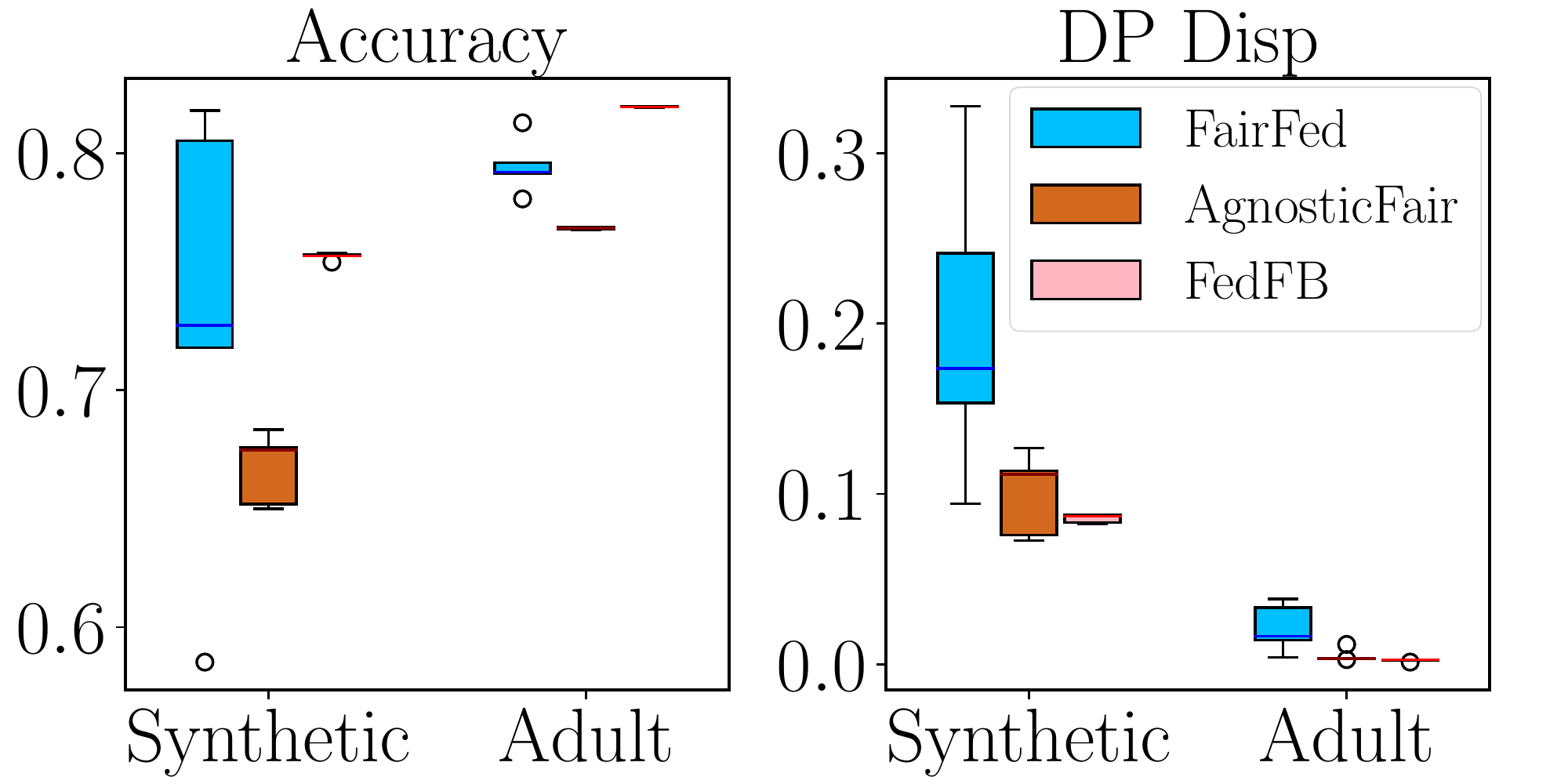}
  \end{center}
\caption{
    \textbf{Performance comparison in terms of accuracy and DP disparity on logistic regression.} 
    \fedfb{} achieves good accuracy and fairness, compared to \fairfed{} and \agnosticfair{}. 
    }
    \label{fig:fairfed}
        \vspace{-.2in}
\end{wrapfigure}

\paragraph{Datasets} \textbf{(synthetic)} We follow \citet{roh2021fairbatch} for data generation, but with a slight modification to make the dataset more unbalanced. 
To study the empirical relationship between accuracy, fairness, and data heterogeneity, 
we split the dataset in different ways to obtain desired levels of data heterogeneity. 
More details are given in Sec.~\ref{appendix:synthetic}.
\textbf{(real)} We use three benchmark datasets: Adult~\citep{dataset:adult} with 48,842 samples, COMPAS~\citep{compas} with 7,214 samples, and Bank~\citep{bank} with 45,211 samples. 
We follow \citet{du2020fairfl}'s method to preprocess and split Adult into two clients and \citet{jiang2020biascorrecting}'s method to preprocess COMPAS and Bank. 
Then, we split COMPAS into two clients based on age and split Bank into three clients based on the loan decision. Note that all the datasets are split in heterogeneous ways.

\paragraph{Results} 
% and leave the results for logistic regression in Sec.~\ref{appendix:experiment}. 
Table~\ref{tab:dp_mlp} reports the test accuracy and DP disparity of four baselines and \fedfb{}. 
% We see a substantial fairness improvement obtained by \fedfb{}.
As expected, the resulting fairness level of \fedfb{} is close to that of CFL.
Besides, we observe the poor performance of \ufl{} and \fflfedavg{}, which is due to the fundamental limitation of \ufl{} and \fflfedavg{}.
%, and the exponentially increasing complexity of hyperparameter selection. 
Table~\ref{tab:dp_dh} reports the accuracy and fairness of each method under different data heterogeneity.
FedFB is observed to be robust to data heterogeneity. 
This agrees with our expectation as \fedfb{} closely mimics the operation of the centralized \fb{}, which is not affected by data heterogeneity. 
We make a more thorough comparison between CFL and \fedfb{} by plotting the accuracy-fairness tradeoff curves in Fig.~\ref{fig:comparison}(b). 
To demonstrate the performance gain does not come at the cost of privacy loss, we restrict \fedfb{} to only exchange 10 bits of information per communication round.
Fig.~\ref{fig:comparison}(b) shows that \fedfb{} still performs well with quantized communications.

As suggested in \citet{ezzeldin2021fairfed}, we use logistic regression to compare our approach with \fairfed{} and \agnosticfair{}. Since \fairfed{} is only applicable to single binary sensitive attribute cases, we report the performance on synthetic and Adult datasets.
Fig.~\ref{fig:fairfed} shows that \fedfb{} achieves the best accuracy-fairness performance among the three methods with a much smaller privacy cost compared to \agnosticfair{}. 
For a fair comparison, we provide the experiment results under the same setting as \citet{ezzeldin2021fairfed}, and full comparison with \agnosticfair{} in Sec.~\ref{appendix:experiment}.

\subsection{FedFB evaluation on client parity}
\begin{figure}
    \centering
    \includegraphics[width=\textwidth]{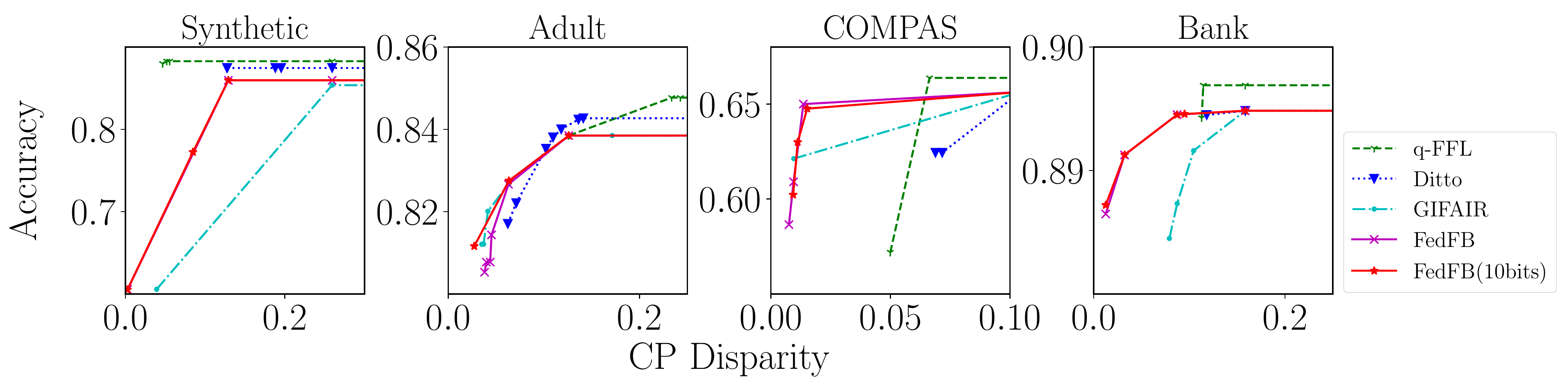}
    \caption{\textbf{Comparison of accuracy and Client Parity (CP) disparity on the synthetic, Adult, COMPAS, and Bank datasets.} Although our algorithm is not specifically designed for CP, it closely matches the state-of-the-art accuracy-CP performance.} 
    \label{fig:cptradeoff}
\end{figure}

We evaluate the performance of \fedfb{} in achieving client parity (CP) and compare it with the state-of-the-art algorithms for CP.
We will measure CP disparity $= \max_{i \neq j \in [I]} |L^{(i)} - L^{(j)}|$, where $L^{(i)}$ is the loss in $i$th client. % see Sec.~\ref{appendix:cp} for more detail.

\paragraph{Baselines} We consider \gifair{}~\citep{yue2021gifair}, \qffl{}~\citep{li2020fair}, \ditto{}~\citep{li2021ditto}, and the unconstrained baseline \fedavg{}~\citep{McMahan17FedAvg}. 
% \gifair{} and \qffl{} are the most similar ones to \fedfb{}. 
Similar to \fedfb{}, both \gifair{} and \qffl{} propose a modified aggregation protocol, under which clients share some additional information with the central server, which then accordingly adjusts the objective function used for the next round of training. 
The key difference is that while \fedfb{} optimizes the coefficients for the primary objective terms (\ie{} sample reweighting) by solving a bi-level optimization problem, \gifair{} updates the coefficient for the penalty terms, and \qffl{} implicitly updates the weights on each objective term based on nonlinear behaviors of polynomial functions, which is equivalent to the $\alpha$-fairness algorithm used in networking~\citep{mo2000fair}. 
The \ditto{} algorithm combines multitask learning with federated learning to learn a personalized classifier for each client, improving the accuracy of the clients with low accuracy.
% Note that \fedavg{} and \ditto{} exchange only model parameters at communication rounds, while \qffl{}, \gifair{}, and \fedfb{} exchange additional information to boost model fairness.

% \paragraph{Datasets} We use the same datasets as Sec.~\ref{sec:exp_fedfb_dp}, but we split the datasets in a different way here. More precisely, we split the datasets according to their sensitive attributes to simulate the same setting as assumed by \gifair{}, \qffl{}, and \ditto{}. Therefore, the synthetic dataset, Adult, COMPAS, and Bank are split into two, two, four, and five clients, respectively. 

\paragraph{Datasets} We use the same datasets as before, but split the datasets according to their sensitive attributes to simulate the same setting as assumed by \gifair{}, \qffl{}, and \ditto{}. 

\paragraph{Results}
% Table~\ref{tab:cp_mlp} shows that \fedfb{} offers competitive and stable performances in mitigating the model bias while sacrificing less accuracy. 
% Although \qffl{} achieves better accuracy and fairness on the synthetic data, \fedfb{} outperforms it in both accuracy and fairness on the other three datasets. 
% To demonstrate the performance gain does not come at the cost of privacy loss, we make the communication rounds of the baselines five times larger than \fedfb{}.
% %, and run FFL algorithms on the synthetic dataset.
% One can see from Fig.~\ref{fig:convergence_rate} that \fedfb{} still converges to the best (training) CP disparity, though the convergence is somewhat slowed down.

Fig.~\ref{fig:cptradeoff} shows that \fedfb{} offers competitive and stable performances in mitigating the model bias, especially in the high fairness region. 
Although \qffl{} achieves better accuracy and fairness on the synthetic data, under strict fairness constraint, \fedfb{} and its private variant nearly achieve the highest accuracy on the other three datasets. 
% To demonstrate the performance gain does not come at the cost of privacy loss, we leverage the private version of \fedfb{} and only exchange 10 bits' information per communication round.

\section{Conclusions}
\paragraph{Summary} We have investigated how to achieve group fairness under a decentralized setting. 
We developed a theoretical framework for decentralized fair learning algorithms and analyzed the performance of \ufl{}, \fflfedavg{}, and CFL.
As a result, we showed that (1) federated learning can significantly boost model fairness even with only a handful number of participating clients, and (2) \fedavg{}-based federated fair learning algorithms are strictly worse than the oracle upper bound of CFL.
To close the gap, we propose \fedfb{}.
Our extensive experimental results demonstrate that our proposed solution \fedfb{} achieves state-of-the-art performances.
%, while still ensuring data privacy.

% For the sake of scalability and computing speed, the current federated learning protocol is designed to ignore the data owned by clients with low computing power. 
% This will make the effective data distribution biased towards the surviving clients, magnifying the model unfairness issue. 
% We identify this as the first possible future research direction. 
\paragraph{Open questions} 
\textbf{(Theory)} 
% While we characterized some fundamental limits on tradeoffs of various approaches, There still remains a large number of open questions. 
First, 
% as we briefly mentioned in Sec.~\ref{sec:num_comparisons_33},
full theoretical characterization of accuracy-fairness tradeoff still remains open. 
% Our current theoretical results only study the extreme ends of the tradeoff curves.
% Moreover, as shown in Sec.~\ref{sec:5.2-3clients}, some of our experimental results reveal a highly nontrivial phenomenon.
% More precisely, with three clients, \ufl{} may achieve strictly better accuracy than \fflfedavg{} in a certain fairness range. 
% Studying this phenomenon and identifying the exact relationship between various learning algorithms is an interesting open problem. 
Furthermore, a three-way tradeoff between accuracy, fairness, and privacy remains widely open. 
% Therefore, another potential interesting research direction is the fundamental tradeoff between accuracy, fairness, and privacy. 
% We leave these directions for future study.
\textbf{(Algorithm)} It remains open whether or not \fedfb{} can handle different fairness notions such as \emph{proportional fairness}~\citep{zhang2020hierarchically,lyu2020collaborative,lyu2020fairfl}.

\section*{Acknowledgement}
This work was supported in part by NSF Award DMS-2023239, NSF/Intel Partnership on Machine Learning for Wireless Networking Program under Grant No. CNS-2003129, and the Understanding and Reducing Inequalities Initiative of the University of Wisconsin-Madison, Office of the Vice Chancellor for Research and Graduate Education with funding from the Wisconsin Alumni Research Foundation.
\bibliography{references}
\bibliographystyle{abbrvnat}

\appendix

\newpage
\section{Appendix - \ufl{}, \fflfedavg{}, CFL analysis}\label{appendix:UFL_FFL_CFL}
In this section, we provide the concrete analysis for \ufl{}, \fflfedavg{}, and CFL. 
For illustration purposes, we will start by discussing two clients cases, and then extend the analysis into more clients cases in Sec.~\ref{appendix:multi_clients}. 
We begin with the analysis of CFL in Sec.~\ref{appendix:cfl}. Then we analyze \ufl{} and \fflfedavg{}. To be specific, in Sec.~\ref{appendix:ufl}, we analyze the limitation of \ufl{} and present the formal version of~Lemma~\ref{lemma:informal_ufl_limit} under two clients cases and Corollary~\ref{cor:informal_gaussian}. 
In Sec.~\ref{appendix:ffl}, we analyze \fflfedavg{} and compare it with \ufl{} and CFL, then we present the formal two-client version of Thm.~\ref{thm:uflvsffl} and Lemma~\ref{lemma:ffl_limit}. 
All the multi-client statements are included in Sec.~\ref{appendix:multi_clients}
We summarize the commonly used notations in Table~\ref{tab:notation} .

\begin{table}
\begin{center}
\begin{tabular}[width = \textwidth]{c c c c c c}
 \hline
 \textbf{Symbol} & \textbf{Meaning} & \textbf{Symbol} & \textbf{Meaning} & \textbf{Symbol} & \textbf{Meaning}\\ [0.5ex]
 \hline
 $\mathrm{x}$ & feature &  $\indicator{\cdot}$ & indicator function & $\varepsilon_i$ & bias in client $i$\\
 $\mathrm{a}$ & sensitive attribute &  $f$ & randomized classifier &  $q$ & $\mathrm{a}\sim \text{Bern}(q)$\\
 $\mathrm{i}$ & client index & $\mathcal{P}_a^{(i)}$ & distribution & $\text{DP Disp}(f)$ & unfairness of $f$\\
 $\hat{\mathrm{y}}$ & predicted class & $\eta(x)$ & $\mathbb{P}(\rvy=1 | \rvx=x)$ &  $f_{\varepsilon_0,\varepsilon_1}^{\text{\ufl{}}}$ & \ufl{} classifier\\
 \hline
\end{tabular}
 
\end{center}
\caption{Commonly used notations.} 
\label{tab:notation}
\end{table} 

\subsection{CFL analysis}\label{appendix:cfl}
In this section, we analyze the CFL classifier $\fcfl$ given in~\ref{prob:cfl}. 
We mainly derive the solution of~\ref{prob:cfl} in Lemma~\ref{lemma:solution}. 
In Lemma~\ref{lemma:unique_solution} and Lemma~\ref{lemma:different_sign} we summarize the properties of $\fcfl$. For a given classifier $f$, we define the mean difference to be
\[\text{MD}(f) = \bP(\haty = 1\mid \rva = 0) - \bP(\haty = 1 \mid \rva = 1).\]

\begin{lemma}\label{lemma:solution}
Let $q\in(0,1)$. Define $g(\cdot): [-\max(q,1-q), \max(q,1-q)] \rightarrow [-1,1]$ as 
\begin{equation}
    g(\lambda) = \int_{[\eta^{-1}(\frac{1}{2} - \frac{\lambda}{2(1-q)}), +\infty]} \,\diffdchar\cP_0 - \int_{[\eta^{-1}(\frac{1}{2} + \frac{\lambda}{2q}), +\infty]} \,\diffdchar \cP_1,
\end{equation} 
then $\fcfl = \{\indicator{s(x,a)>0} + \alpha \indicator{s(x,a) = 0} : \alpha \in [0,1]\}$, where $s(x,0) = 2\eta(x) -1 + \frac{\lambda}{1-q}$, $s(x,1) = 2\eta(x) - 1 -\frac{\lambda}{q}$, $\lambda = g^{-1}(\sign(g(0))\min\{\veps, |g(0)|\})$. Here we denote the indicator function as $\indicator{E}:$ $\indicator{E}=1$ if $E$ is true, zero otherwise. 
\end{lemma}

\begin{proof}
The proof is similar as \citet{menon2018cost}. To solve~\ref{prob:cfl}, we first write the error rate and the fairness constraint as a linear function of $f$. Let $p_a(\cdot)$ be the pdf of $\cP_a$, where $a = 0,1$. 
Denote the joint distribution of $\rvx$ and $\rva$ as $p_{\rvx,\rva}(x,a)$. Note that 
\begin{equation}\label{eq:acc}
    \begin{aligned}
    & \bP(\haty \neq \rvy) \\ %  & = \bP(\rvy = 0)\bE_{\rvx, \rva\mid \rvy= 0} f(\rvx,\rva) + \bP(\rvy = 1)\bE_{\rvx, \rva \mid \rvy = 1}[1-f(\rvx,\rva)]\\
    =& \int_{\cX}\sum_{a\in \cA} \left[f(x,a)(1-\eta(x)) +  (1-f(x,a))\eta(x) \right]p_{\rvx,\rva}(x,a) \,\diffdchar x \\
    =&\bE_{\rvx, \rva} f(\rvx, \rva) (1-2\eta(\rvx)) + \bP(\rvy = 1)
    \end{aligned}
\end{equation}
and 
\begin{equation}\label{eq:fair}
    \begin{aligned}
    & \bP(\haty =1\mid \rva = 0) - \bP(\haty=1\mid \rva = 1) \\
    % =& \bE_{\rvx \mid \rva = 0}f(\rvx, 0) - \bE_{\rvx \mid \rva = 1}f(\rvx, 1)\\
    =&\int_\cX f(x,0)p_0 (x) \,\diffdchar x - \int_\cX f(x,1)p_1 (x) \,\diffdchar x\\
    =& \int_{\cX}\sum_{a\in \cA} \indicator{a = 0}f(x,0)\frac{p_{\rvx,\rva}(x,a)}{\bP(\rva = 0)}\,\diffdchar x
     - \int_{\cX}\sum_{a\in \cA}  \indicator{a = 1}f(x,1)\frac{p_{\rvx,\rva}(x,a)}{\bP(\rva = 1)}\,\diffdchar x
     \\
    =& \bE_{\rvx, \rva} \left[f(\rvx,0)\frac{\indicator{\rva = 0}}{1-q} - f(\rvx,1)\frac{\indicator{\rva = 1}}{q}\right].
    \end{aligned}
\end{equation}
Consequently, our goal becomes solving 
\begin{equation}\label{prob:lemma1}
    \begin{array}{c}
         \min_{f\in\cF}  \bE_{\rvx, \rva} f(\rvx, \rva) (1-2\eta(\rvx)) + \bP(\rvy = 1)\\
         \st.~|\bE_{\rvx, \rva} \left[f(\rvx,0)\frac{\indicator{\rva = 0}}{1-q} - f(\rvx,1)\frac{\indicator{\rva = 1}}{q}\right]| \leq \veps.
    \end{array}
\end{equation}
Denote the function that minimizes the error rate (ERM) as $\tdf\in\cF$. It is easy to see that,
\begin{equation}
    \tdf(x) \in \left\{\indicator{\eta(x) > 1/2} + \alpha \indicator{\eta(x) = 1/2}: \alpha \in [0,1]\right\}.
\end{equation}

Next, consider the following three cases. 
In particular, we provide the proof for $|\md(\tdf)| \leq \veps$ and $\md(\tdf) > \veps$. 
The proof for $\md(\tdf) < -\veps$ is similar as the proof for $\md(\tdf) > \veps$.
\begin{case}
$|\md(\tdf)| \leq \veps$: ERM is already fair.\end{case}
The solution to~\eqref{prob:lemma1} and~\ref{prob:cfl} is $\tdf$.

\begin{case}
$\md(\tdf) > \veps$: ERM is favoring group 0 over group 1. \end{case}
We will show that solving~\eqref{prob:lemma1} is equivalent to solving an unconstrained optimization problem. 

First, we will prove by contradiction that the solution $f^\star \in \cF$ to~\eqref{prob:lemma1} satisfies $\md(f^\star) = \veps$. 
We use $f^\star \in \cF$ to denote the solution of~\eqref{prob:lemma1}. 
Suppose the above claim does not hold. Then we have $\md(f^\star) < \veps$. 
To show the contradiction, we construct a $f^\prime \in \cF$ that satisfies the fairness constraint and has a lower error rate than that of $f^\star$. 
Let $f^\prime $ be a linear combination of $f^\star$ and $\tdf$:
\begin{equation}
    f^\prime = a f^\star + (1-a) \tdf,
\end{equation}
where $a = \frac{\md(\tdf) - \veps}{\md(\tdf) - \md(f^\star)} \in (0,1)$. 
Then we obtain 
\begin{equation}
    \md(f^\prime) = a \md(f^\star) + (1-a) \md(\tdf) = \frac{\veps (\md(\tdf) - \md(f^\star))}{\md(\tdf) - \md(f^\star)} = \veps.
\end{equation}
Denote the error rate $\bP\set{\haty \neq \rvy}$ as $e: \cF \rightarrow [0,1]$. 
Then the error rate of $f^\prime$ is
\begin{equation}
    e(f^\prime) = a e(f^\star) + (1-a)e(\tdf) < e(f^\star),
\end{equation}
which is inconsistent to the optimality assumption of $f^\star$. 
Therefore, $\md(f^\star) = \veps$. 

Now, solving~\ref{prob:cfl} is equivalent to solving 
\begin{equation}
    \begin{array}{c}
         \min_{f\in\cF}  \bE_{\rvx, \rva} f(\rvx, \rva) (1-2\eta(\rvx)) + \bP(\rvy = 1)\\
         \st.~|\bE_{\rvx, \rva} \left[f(\rvx,0)\frac{\indicator{\rva = 0}}{1-q} - f(\rvx,1)\frac{\indicator{\rva = 1}}{q}\right]| = \veps.
    \end{array}
\end{equation}
Furthermore, the optimization problem above is also equivalent to
\begin{eqnarray}
         &\min_{f\in\cF}&  \bE_{\rvx, \rva} f(\rvx, \rva) (1-2\eta(\rvx)) - \lambda \bE_{\rvx, \rva} \left( f(\rvx,0)\frac{\indicator{\rva = 0}}{1-q} - f(\rvx,1)\frac{\indicator{\rva = 1}}{q}\right) \label{eq:lemma1_min}\\
         &&~~~~~~~\st.~|\bE_{\rvx, \rva} \left[f(\rvx,0)\frac{\indicator{\rva = 0}}{1-q} - f(\rvx,1)\frac{\indicator{\rva = 1}}{q}\right]| = \veps \label{eq:lemma1_st},
\end{eqnarray}
for all $\lambda \in \bR$. 

Next, our goal is to select a suitable $\lambda$ such that the constrained optimization problem above becomes an unconstrained problem, \ie, we will select a suitable $\lambda$ such that the minimizer to the unconstrained optimization problem~\eqref{eq:lemma1_min}
satisfies equality constraint~\eqref{eq:lemma1_st}. 

Note that 
\begin{align}
    & \bE_{\rvx, \rva} f(\rvx, \rva) (1-2\eta(\rvx))  - \lambda \bE_{\rvx, \rva} \left( f(\rvx,0)\frac{\indicator{\rva = 0}}{1-q} - f(\rvx,1)\frac{\indicator{\rva = 1}}{q}\right) \\
    = &\bE_{\rvx, \rva} f(\rvx, \rva) \left(1-2\eta(\rvx)  - \lambda  \frac{\indicator{\rva = 0}}{1-q} + \lambda \frac{\indicator{\rva = 1}}{q}\right),
\end{align}
then the solution to unconstrained optimization problem~\eqref{eq:lemma1_min} is 
\begin{equation}
    \bar{f} \in \set{\indicator{s(x,a) > 0} + \alpha \indicator{s(x,a) = 0}: \alpha \in[0,1]},
\end{equation}
where 
\begin{equation}
    s(x,a) = \begin{cases} 
      -1+2\eta(x) + \frac{\lambda}{1-q} & a = 0\\
      -1+2\eta(x) - \frac{\lambda}{q} & a = 1
   \end{cases}.
\end{equation}
\begin{figure}
    \centering
    \includegraphics[width = 0.5\textwidth]{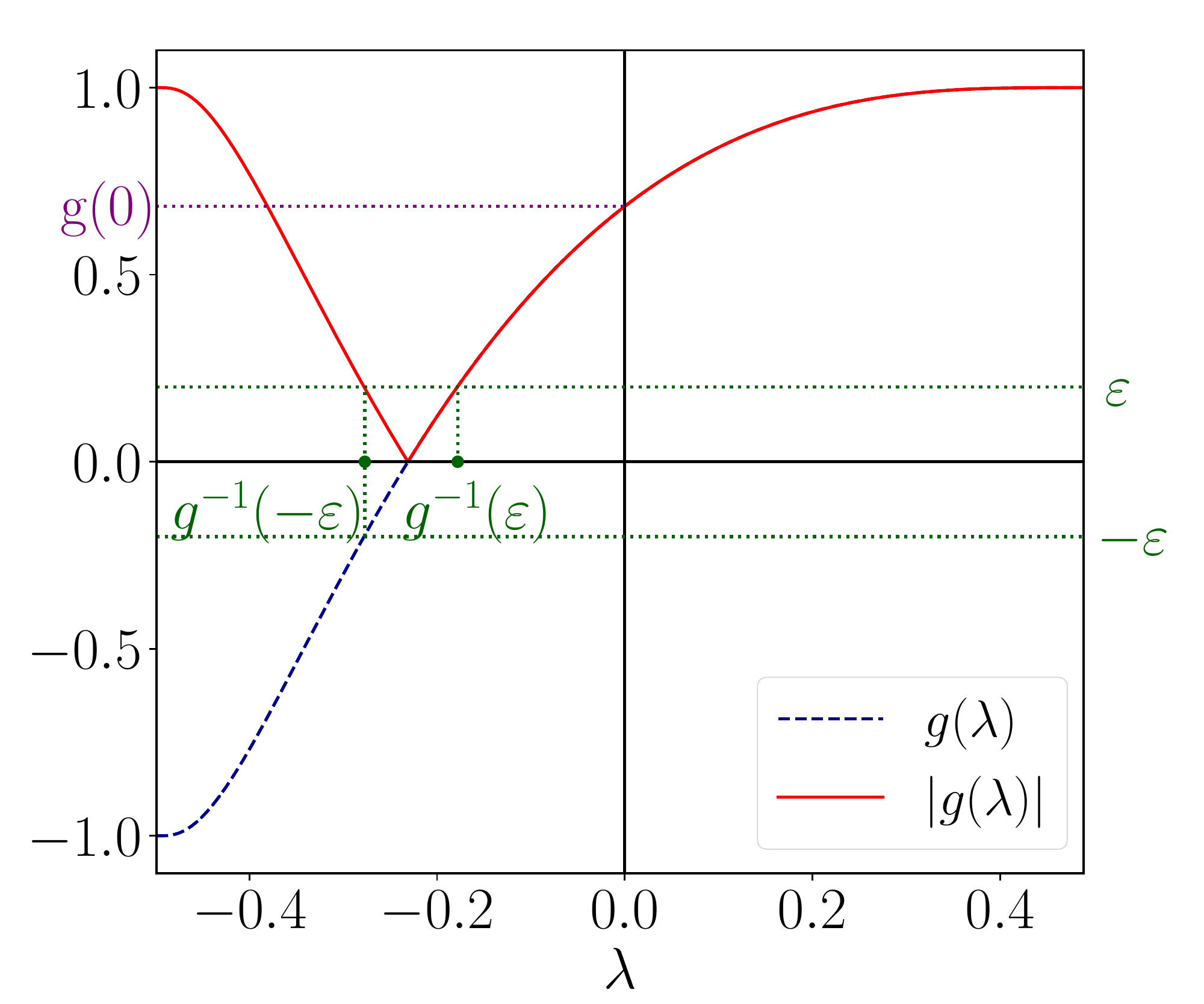}
    \caption{Visualization of $g(\lambda)$ and $|g(\lambda)|$ under \textbf{Case 2:}  $\md(\tdf) > \veps$. When $g(0) = \md(\Tilde{f}) >\veps \geq0$, the corresponding $\lambda$ of the best classifier is $\lambda = g^{-1}(\veps) < 0$.}
    \label{fig:g}
\end{figure}
Since the range of $\eta(x)$ is [0,1], $\bar{f}$ with $\lambda > \max(q,1-q)$ is no different from $\bar{f}$ with $\lambda = \max(q,1-q)$; $\bar{f}$ with $\lambda < -\max(q,1-q)$ is no different from $\bar{f}$ with $\lambda = -\max(q,1-q)$. 
Therefore, the only thing left is to find the $\lambda \in [-\max(q,1-q), \max(q,1-q)]$ such that $\bar{f}$ satisfies the constraint~\eqref{eq:lemma1_st}. 

Now consider the mean difference between the positive rate of two groups:
\begin{equation}\label{eq:g}
    \begin{aligned}
    \md(\bar{f}) & = \bP(\haty = 1\mid \rva = 0) - \bP(\haty = 1\mid \rva = 1) \\
    & = \int_{-\infty}^{+\infty} \indicator{\eta(x) > \frac{1}{2} - \frac{\lambda}{2(1-q)}} \, \diffdchar\cP_0 - \int_{-\infty}^{+\infty} \indicator{\eta(x) > \frac{1}{2} + \frac{\lambda}{2q}} \, \diffdchar\cP_1 \\ 
    & = \int_{[\eta^{-1}(\frac{1}{2} - \frac{\lambda}{2(1-q)}), +\infty]} \,\diffdchar\cP_0 - \int_{[\eta^{-1}(\frac{1}{2} + \frac{\lambda}{2q}), +\infty]} \,\diffdchar\cP_1 \\
    & =g(\lambda).
    \end{aligned}
\end{equation}
Note that $g(\cdot): [-\max(q,1-q), \max(q,1-q)] \rightarrow [-1,1]$ is a strictly monotone increasing function. Consequently, if and only if $\lambda = g^{-1}(\veps)$, $\bar{f}$ satisfies~\eqref{eq:lemma1_st}. 
Recall that optimization problem~\eqref{eq:lemma1_min} with constraint~\eqref{eq:lemma1_st} is equivalent as~\ref{prob:cfl}. 
Thus, let $\lambda = g^{-1}(\veps)$ and $\bar{f}$ is the solution of~\ref{prob:cfl}. 

\begin{case}
$\md(\tdf) < -\veps$: ERM is favoring group 1 over group 0.\end{case}

Similarly, like Case 2, we obtain that the solution~\ref{prob:cfl} is
\begin{equation}
    \bar{f} \in \set{\indicator{s(x,a) > 0} + \alpha \indicator{s(x,a) = 0}: \alpha \in[0,1]},
\end{equation}
where 
\begin{equation}
    s(x,a) = \begin{cases} 
      -1+2\eta(x) + \frac{\lambda}{1-q} & a = 0\\
      -1+2\eta(x) - \frac{\lambda}{q} & a = 1
   \end{cases},
\end{equation}
and $\lambda = g^{-1}(-\veps)$.
Combining all the cases above yields the desired conclusion. The proof is now complete.
\end{proof}

\begin{remark}
Select $\alpha = 0$, then the solution to~\ref{prob:cfl} can be written as 
\begin{equation}\label{eq:thresholdform}
    f(x,a) = \begin{cases} 
      \indicator{\eta(x) > \frac{1}{2} - \frac{\lambda}{2(1-q)}} & a= 0\\
      \indicator{\eta(x) > \frac{1}{2} + \frac{\lambda}{2q}} & a = 1
   \end{cases}.
\end{equation}
Therefore, Lemma~\ref{lemma:solution} implies that the best classifier of the CFL problem is equivalent to simply applying a constant threshold to the class-probabilities for each value of the sensitive feature. 
\end{remark}

Lemma~\ref{lemma:solution} suggests the following property of the solution to~\ref{prob:cfl}.
\begin{lemma}\label{lemma:unique_solution}
If $f$ and $g$ are two solutions to~\ref{prob:cfl}, then $f=g$ almost everywhere. 
\end{lemma}

For illustration purposes, we denote 
\begin{equation}\label{lambda}
\lbdcfl=g^{-1}(\sign(g(0))\min\{\veps, |g(0)|\}).    
\end{equation}
Below we summarize some useful properties of $\lbdcfl$. 

\begin{lemma}\label{lemma:different_sign}
The sign of $\lbdcfl$ and $\md(\fcfl)$ are determined by $g(0)$. 
\begin{enumerate} 
    \item If $\veps<|g(0)|$, then $\md(\fcfl) = \lbdcfl \neq 0$, and $g(\lbdcfl) = \sign(g(0))\veps$. If $|g(0)|\leq \veps,$ then $\lbdcfl = 0$ and $\md(\fcfl) = g(\lbdcfl) = g(0)$.

    \item If $g(0)>0$ or $g(0)<0$, then for any $\veps\geq 0$, we have $\lambda \leq 0$ or $\lambda \geq 0$, respectively.
\end{enumerate}
\end{lemma}

\begin{proof}
The first property follows directly from the definition of $\lbdcfl$. Next, we prove the second property.

When $\veps > |g(0)|$, we have $\lambda=0$ and the first property holds. 
When $g(0)>\veps>0$, by the definition of $\lbdcfl$, we have $\lambda = g^{-1}(\veps) < g^{-1}(g(0)) =0$. 
When $g(0)<-\veps<0$, we have $\lambda = g^{-1}(-\veps) > g^{-1}(g(0)) = 0$. Combining all the cases above yields the desired conclusion.
\end{proof}

\subsection{\ufl{} analysis}\label{appendix:ufl}
With the analysis of CFL in Sec.~\ref{appendix:cfl}, in this section, we analyze the \ufl{} classifier~\ref{prob:ufl} for the case of two clients. For illustration purpose, with $I=2$, we denote $\fufl=\fuflm=(f_0^{\veps_0}+f_1^{\veps_1})/2$.
In Sec.~\ref{appendix:ufl_setting} we introduce some notations for the \ufl{} classifier $\fufl$ that follows from Lemma~\ref{lemma:solution}. 
In Sec.~\ref{appendix:ufl_limit} we analyze the limitation of $\fufl$ as stated in Sec.~\ref{sec:necessity}. 
To be more specific, we present the two clients' version of Lemma~\ref{lemma:informal_ufl_limit}, formal version of Corollary~\ref{cor:informal_gaussian} and conclude their proof. 
In Sec.~\ref{appendix:ufl_cfl}, we analyze the performance gap between $\fufl$ and CFL classifier $\fcfl$.

\subsubsection{Problem setting}\label{appendix:ufl_setting}
 By Lemma~\ref{lemma:solution}, the solution to~\ref{prob:ufl} is
\begin{equation}
        f_i^{\veps_i}(x,a) = \begin{cases} 
      \indicator{\eta(x) > \frac{1}{2} - \frac{\lbdufl}{2(1-q)}} & a= 0\\
      \indicator{\eta(x) > \frac{1}{2} + \frac{\lbdufl}{2q}} & a = 1
   \end{cases},
\end{equation}
where the associated $\lbdufl$ is defined as
\begin{equation}\label{lambda_i}
\lbdufl = g_i^{-1}(\sign(g_i(0)) \min(\veps_i, |g_i(0)|))     
\end{equation}
and
\[g_i(\lambda) = \int_{[\eta^{-1}(\frac{1}{2} - \frac{\lambda}{2(1-q)}), +\infty)} \, \diffdchar\cP_0^{(i)} - \int_{[\eta^{-1}(\frac{1}{2} + \frac{\lambda}{2q}), +\infty)} \, \diffdchar\cP_1^{(i)}.\] 
Note that $g_i(\lambda)$ is the mean difference on $i$th client $\bE_{\rvx \sim \cP_0^{(i)}}f(\rvx, 0) - \bE_{\rvx \sim \cP_1^{(i)}}f(\rvx, 1)$ of the classifier of the form~\eqref{eq:thresholdform}. Now, the demographic disparity for $\fufl$ can be written as
\begin{equation}\label{dp_\ufl{}}
    \begin{split}
       \text{DP Disp}(\fufl)    &  = \left|\frac{1}{2}\sum_{i=0}^1 \left[ \bP(\haty = 1\mid \rva = 0, \rvi = i) - \bP(\haty = 1 \mid \rva = 1, \rvi = i) \right]\right| \\
       & = \left|\frac{1}{4}\sum_{i,j=0}^1 \left[\bE_{\rvx \sim \cP_0^{(i)}}f_j^{\veps_j}(\rvx, 0) - \bE_{\rvx \sim \cP_1^{(i)}}f_j^{\veps_j}(\rvx, 1)  \right]\right|  \\
    &  = |\frac{1}{4}\left(g_0(\lbduflo) +  g_0(\lbdufll)+g_1(\lbduflo)+g_1(\lbdufll)\right)|. 
    \end{split}
\end{equation}
For ease of notation, we define local mean difference on $i$th client as $\text{MD}_i(f) = \bP(\haty = 1\mid \rva = 0, \rvi = i) - \bP(\haty = 1 \mid \rva = 1, \rvi = i)$ and local demographic disparity on $i$th client as $\text{DP Disp}_i(f) = |\text{MD}_i(f)|$, where $i = 0,1$. 
Since the overall distribution of the data samples is $\rvx \mid \rva = a \sim \cP_a = \cP_a^{(0)}/2 + \cP_{a}^{(1)}/2$, $a=0,1$, $g$ (see the definition of $g$ in Lemma~\ref{lemma:solution}) and $g_0,g_1$ has the following relation:
\[g(\lambda) = \frac{1}{2}g_0(\lambda) + \frac{1}{2} g_1(\lambda).\]

\subsubsection{Limitation of \ufl{}}\label{appendix:ufl_limit}
In this section, we mainly analyze the limitation of \ufl{} in Lemma~\ref{lemma:ufl_limit}, which shows that $\fufl$ can not achieve $0$ demographic disparity in certain cases. Corollary~\ref{cor:gaussian} is a specific example of Lemma~\ref{lemma:ufl_limit}.
\begin{lemma}[Formal version of Lemma~\ref{lemma:informal_ufl_limit} under two clients cases]\label{lemma:ufl_limit}
Let $q\in(0,1)$. Let $c = \min\{|g_0(0)|, |g_1(0)|\}$. Define $\psi: [0,c]\times[0,c] \rightarrow [-1,1]$ as
\begin{equation}\label{def:psi}
\begin{split}
 \psi(\veps_0, \veps_1) = \md(\fufl)   &= \frac{1}{4}g_0(g_1^{-1}(\sign(g_1(0)) \veps_1)) + \frac{1}{4}g_1(g_0^{-1}(\sign(g_0(0))\veps_0)) \\
    & + \frac{1}{4}\sign(g_0(0)) \veps_0 + \frac{1}{4} \sign(g_1(0))\veps_1.
\end{split}
\end{equation}
 If $g_0(0)g_1(0) < 0$ and $\psi(\veps_0, \veps_1)(g_0(0) + g_1(0)) > 0 $ for all $\veps_0, \veps_1 \in [0, c]$, then for all $\veps_0, \veps_1 \in [0,1]$, $\text{DP Disp}(\fufl) \geq \delta= \min\{|\psi(\veps_0, \veps_1)|: \veps_0, \veps_1 \in [0, c]\} > 0$.
\end{lemma}
\begin{proof}
Define $\delta = \min\{|\psi(\veps_0, \veps_1)|: \veps_0, \veps_1 \in [0, c]\}$. 
The goal is to show that the demographic disparity has a positive lower bound. 
Note that the mean difference can be expressed as
\begin{equation}\label{\ufl{}_dp}
   \md(\fufl) = \frac{1}{4}\left(g_0(\lbduflo) +  g_0(\lbdufll)+g_1(\lbduflo)+g_1(\lbdufll)\right). 
\end{equation}
In the following proof, we will show, the mean difference cannot reach 0.

Without any loss of generality, assume $|g_0(0)| < |g_1(0)|$. 
First we consider $g_1(0)>0$. We will discuss $g_1(0)<0$ later. 
By $g_0(0)g_1(0)<0$ and $\psi(\veps_0, \veps_1)(g_0(0)+ g_1(0)) > 0$ for all $\veps_0, \veps_1 \in [0, c]$, we have $g_0(0)<0$ and $\psi(\veps_0, \veps_1) > 0$ for all $\veps_0, \veps_1 \in [0, c]$. 

First, we will prove that \ufl{} achieves its lowest mean difference when $\veps_0, \veps_1 \in [0, c]$. 
In what follows, we consider five different cases to derive the desired result.

\begin{case}
$\veps_0 > |g_0(0)|, \veps_1 > |g_1(0)|$: ERM is fair on both clients.
\end{case} By~\eqref{lambda_i}, we have $\lbduflo = \lbdufll = 0$. 
Recall $g_i(\cdot)$ is a monotone increasing function, we combine $g_1(0)>0$ and Lemma~\ref{lemma:different_sign} to have $g_0(g_1^{-1}(0)) < g_0(0) < 0.$
Applying the above conclusion yields 
\[\eqref{\ufl{}_dp} = \frac{1}{2}g_0(0) +  \frac{1}{2}g_1(0) > 2\left(\frac{1}{4}g_0(0) + \frac{1}{4}g_1(0) + \frac{1}{4}g_0(g_1^{-1}(0))\right) = 2\psi(g_0(0), 0)\geq \delta.\]

\begin{case}
$\veps_0 \leq |g_0(0)|, \veps_1 > |g_1(0)|$: ERM is unfair on client 0, but fair on client 1.
\end{case} 
Applying~\eqref{lambda_i} results in $\lbdufll=0$. 
By the fact that $g_i(\cdot)$ is a strictly monotone increasing function, we have $\lbduflo=g_0^{-1}(-\veps_0) > g_0^{-1}(g_0(0)) = 0.$ 
Applying the above conclusion yields
\begin{equation}
    \begin{aligned}
       ~\eqref{\ufl{}_dp}  =& - \frac{1}{4} \veps_0 + \frac{1}{4} g_1(0) + \frac{1}{4}g_0(0) + \frac{1}{4}g_1(\lbduflo)~~~\big(\lbduflo>0, g_1(\lbduflo) > g_1(0), g_0(0) < -\veps_0\big) \\
        >& \frac{1}{2}g_0(0)+\frac{1}{2}g_1(0) > 2\psi(g_0(0), 0) \geq \delta.
    \end{aligned}
\end{equation}
\begin{case}
$\veps_0 \leq |g_0(0)|, \veps_1 \leq |g_1(0)|$: ERM is unfair on both client 0 and client 1. 
\end{case} 
Applying~\eqref{lambda_i} we have $\lbduflo = g_0^{-1}(-\veps_0), \lbdufll=g_1^{-1}(\veps_1)$.
Then we have
\begin{align}
  ~\eqref{\ufl{}_dp} & = \frac{1}{4} \left(-\veps_0 + \veps_1 + g_0(g_1^{-1}(\veps_1)) + g_1(g_0^{-1}(-\veps_0))\right) \\
   & \geq \frac{1}{4} \left(-\veps_0 + \veps_0 + g_0(g_1^{-1}(\veps_0)) + g_1(g_0^{-1}(-\veps_0))\right) =  \psi(\veps_0, \veps_0) \geq \delta. 
\end{align}

\begin{case}
$\veps_0 > |g_0(0)|, \veps_1 \leq |g_0(0)| $: ERM is fair on client 0 and very unfair on client 1.
\end{case} 
By~\eqref{lambda_i}, we have $\lbduflo = 0, \lbdufll = g_1^{-1}(\veps_1) > g_1^{-1}(0)$. 
Then we obtain
\begin{align}
~\eqref{\ufl{}_dp}   &  = \frac{1}{4}\left(g_0(0)+g_0(\lbdufll) +g_1(0)+\veps_1\right) \\
    &> (g_0(0) + g_0(g_1^{-1}(0)) + g_1(0))/4 = \psi(g_0(0),0) \geq \delta.
\end{align}

\begin{case}
$\veps_0 > |g_0(0)|, |g_0(0)| \leq \veps_1 < |g_1(0)|$: ERM is fair on client 0 and unfair on client 1.
\end{case} 
Applying~\eqref{lambda_i} implies $ \lbdufll = g^{-1}_1(\veps_1) > g_1^{-1}(0)$. 
Therefore,
\begin{align}
~\eqref{\ufl{}_dp}  =   &  \frac{1}{4}\left(g_0(0) + g_0(g_1^{-1}(\veps_1)) + \veps_1 + g_1(0)\right)  \\
    &>(g_0(0) + g_1(0)+ g_0(g_1^{-1}(0)) + \veps_1)/4 > \psi(g_0(0), 0) \geq \delta.
\end{align}
Combining all the cases above, we conclude that when $g_1(0)>0$ $\text{DP Disp}(\fufl) \geq \delta = \min\{|\psi(\veps_0, \veps_1)|: \veps_0, \veps_1 \in [0, c]\} > 0$ for all $\veps_0, \veps_1 \in [0,1]$.

Now we consider $g_1(0)<0$, by the setting $|g_0(0)| < |g_1(0)|$ and the assumption $g_0(0)g_1(0)<0$ and $\psi(\veps_0,\veps_1)(g_0(0)+g_1(0))>0$, we have $g_0(0)>0$ and $\psi(\veps_0,\veps_1)<0$ for all $\veps_0,\veps_1 \in [0,c]$. 
Following similar computation above, Case 1 - Case 5 become:

Case 1. $\veps_0 > |g_0(0)|, \veps_1 > |g_1(0)|$. Now we have $0<g_0(0)<g_0(g_1^{-1}(0))$, thus 
\[\eqref{\ufl{}_dp} < 2\left(\frac{1}{4}g_0(0) + \frac{1}{4}g_1(0) + \frac{1}{4}g_0(g_1^{-1}(0))\right) = 2\psi(g_0(0), 0)\leq -\delta.\]

Case 2. $\veps_0 \leq |g_0(0)|, \veps_1 > |g_1(0)|$. Now we have $g_0(0)>\veps_0$, $g_1(\lbduflo) < g_1(0)$, thus
\[\eqref{\ufl{}_dp} < \frac{1}{2}g_0(0)+ \frac{1}{2}g_1(0)< 2\psi(g_0(0),0)\leq -\delta.\]

Case 3. In this case we have
\[\eqref{\ufl{}_dp}  = \frac{1}{4} \left(\veps_0 - \veps_1 + g_0(g_1^{-1}(-\veps_1)) + g_1(g_0^{-1}(\veps_0))\right) = \psi(\veps_0, \veps_1) \leq -\delta.\] 

Case 4. Now we have $\lbdufll = g_1^{-1}(-\veps_1) < g_1^{-1}(0)$, thus
\[\eqref{\ufl{}_dp} < (g_0(0) + g_0(g_1^{-1}(0)) + g_1(0))/4 = \psi(g_0(0),0) \leq -\delta.\]

Case 5. Now we have $ \lbdufll = g^{-1}_1(-\veps_1) < g_1^{-1}(0)$, thus
\[\eqref{\ufl{}_dp}  =     (g_0(0) + g_1(0)+ g_0(g_1^{-1}(0)) - \veps_1)/4 < \psi(g_0(0), 0) \leq - \delta.\]
Then we conclude the proof.

\end{proof}

\begin{remark}
Note that $c$ is the smallest local demographic disparity the ERM achieves on clients. 
The condition $g_0(0)g_1(0)<0$ implies that the ERM is favoring different groups in different clients. 
The condition $\psi(\veps_0, \veps_1)(g_0(0) + g_1(0)) > 0$ for all $\veps_0, \veps_1 \in [0,c]$ implies that $\fufl$ favors the same group as ERM when the constraint is very tight. If the conditions above hold, Lemma~\ref{lemma:ufl_limit} suggests that there exists a lower bound of all the demographic disparity that \ufl{} can achieve. 
In particular, if the conditions above hold, \ufl{} fails to achieve perfect demographic parity.  
\end{remark}

Among the conditions of Lemma~\ref{lemma:ufl_limit}, $g_0(0)g_1(0)<0$ can be satisfied by the distribution with high data heterogeneity. 
To demonstrate the condition $\psi(\veps_0, \veps_1)(g_0(0) + g_1(0)) > 0$ for all $\veps_0, \veps_1 \in [0,c]$ can be satisfied, we consider a limiting Gaussian case. 
The following corollary serves as an example that satisfies the conditions of Lemma~\ref{lemma:ufl_limit}, and provides a more explicit expression of the lowest demographic disparity \ufl{} can reach in the Gaussian case.

\begin{corollary}[Formal version of Corollary~\ref{cor:informal_gaussian}]\label{cor:gaussian}
Let $q=0.5$, $\eta(x) = \frac{1}{1+\text{e}^{-x}}, \cP_a^{(i)} = \cN(\mu_a^{(i)}, \sigma^{(i)2})$ where $(\mu_0^{(0)} - \mu_1^{(0)}) (\mu_0^{(1)} - \mu_1^{(1)}) < 0$. 
Then DP Disp$(\fufl) \geq \delta \approx \frac{1}{4}|g_0(0) + g_1(0)| > 0$ for all $\veps_0, \veps_1\in[0,1]$ if one of the following condition holds: 
\begin{enumerate}
    \item $\sigma^{(0)}  \gg \sigma^{(1)}, |\mu_0^{(0)}|, |\mu_1^{(0)}|, |\mu_0^{(1)}|, |\mu_1^{(1)}|$  and $\mu_0^{(1)} > \mu_1^{(1)}$: client 0 has much larger variance than client 1, and client 1 is favoring group 0;
    \item $\sigma^{(1)}  \gg \sigma^{(0)}, |\mu_0^{(0)}|, |\mu_1^{(0)}|, |\mu_0^{(1)}|, |\mu_1^{(1)}|$  and $\mu_0^{(0)} > \mu_1^{(0)}$: client 1 has much larger variance than client 0, and client 0 is favoring group 0.
\end{enumerate}
\end{corollary}
\begin{proof}
In this example, note that local mean difference function of $\lambda$ can be written as:
\begin{equation}\label{g_i_example}
 g_i(\lambda) = \Phi(\frac{\eta^{-1}(\frac{1}{2} + \lambda) - \mu_1^{(i)}}{\sigma^{(i)}}) - \Phi(\frac{\eta^{-1}(\frac{1}{2} - \lambda) - \mu_0^{(i)}}{\sigma^{(i)}})   ,
\end{equation}
where $\Phi(\cdot)$ is the CDF of the standard Gaussian distribution.

We only provide the proof for condition 1, and the proof for condition 2 is similar. 
Assume condition 1 holds. By $(\mu_0^{(0)} - \mu_1^{(0)}) (\mu_0^{(1)} - \mu_1^{(1)}) < 0$ and $\mu_0^{(1)} - \mu_1^{(1)} > 0$, we have $\mu_0^{(0)} - \mu_1^{(0)} < 0$. 
By~\eqref{g_i_example}, we have $g_0(0) < 0$ and $g_1(0) > 0$. 
Consequently, combining Lemma~\ref{lemma:different_sign} and~\eqref{def:psi} yields
\[\psi(\veps_0, \veps_1) = \frac{1}{4} \left(g_0(g_1^{-1}(\veps_1)) + g_1(g_0^{-1}(-\veps_0)) - \veps_0 + \veps_1 \right).\] 
First we show that $\psi(\veps_0,\veps_1)$ reachs its minimum either at $(c,0)$ or $(0,c)$ by taking the derivative of $\psi$, where $c = \min\{|g_0(0)|,|g_1(0)|\}$ is the smallest local demographic disparity. 
And then, we will estimate the minimum of $\psi$ on $[0,1]\times[0,1]$.

We take the derivative of $\psi$ with respect to $\veps_i$ and get
\begin{equation}
    \frac{\partial \psi}{\partial \veps_i}(\veps_0, \veps_1) = \sign(g_i(0)) \left(1+ \frac{g_{1-i}^\prime(g_i^{-1}( \sign(g_i(0)) \veps_i))}{g_i^\prime(g_i^{-1}(\sign(g_i(0)) \veps_i ))}\right) / 4.
\end{equation}

By condition 1, we have $g_0(0) = \Phi(-\frac{\mu_1^{(0)}}{\sigma^{(0)}}) - \Phi(-\frac{\mu_0^{(0)}}{\sigma^{(0)}}) \approx 0,$ thus $|g_0(0)|\ll |g_1(0)|$ and $c = |g_0(0)|$. Since $g_i$ are increasing function, $i=0,1$, we have $g_i^\prime(\cdot)>0$, and thus $\frac{\partial \psi}{\partial \veps_0}<0,\frac{\partial \psi}{\partial \veps_1}>0$. Therefore, $\psi$ reaches its extreme value at $(0,c)$ and $(c,0)$. 

Now, we evaluate 
\[\psi(0,c) = (g_0(g_1^{-1}(c)) + g_1(g_0^{-1}(0)) - g_0(0))/4> \frac{g_1(0)-g_0(0)+g_0(g_1^{-1}(c))}{4},\] where the inequality comes from $g_0(0)<0$ to have $g_1(g_0^{-1}(0))>g_1(0)$. 
And at $(c,0)$, we have
\[\psi(c,0) = (g_0(g_1^{-1}(0)) + g_1(0)+g_0(0))/4.\] 

In what follows, we will show that $\min\set{\psi(c,0), \psi(0,c)} \approx \delta = \frac{g_1(0)+g_0(0)}{4}$ 
by proving that $ 0>g_0(g_1^{-1}(c)) > g_0(g_1^{-1}(0))\approx 0$. 

Consider $\psi(c,0)$ and $\psi(0,c)$. Since $g_1(0)>c,g_0(0)<0$, we have 
\[g_0(g_1^{-1}(0))<g_0(g_1^{-1}(c))<g_0(0)<0.\]

Therefore, the only thing left is to show $ g_0(g_1^{-1}(0))\approx 0$. 
We divide the rest of the proof into the following three cases. 

\begin{case}
$\mu_1^{(1)} < \mu_0^{(1)}< 0$: on client 1, the local classifier is favoring group 0 over group 1, and the positive rate of both groups are under $\frac{1}{2}$.
\end{case}
Clearly, under this case, we have $\int_{[\eta^{-1}(\frac{1}{2}),\infty)} \,\diffdchar  \cP_0^{(1)} = \int_{[0,\infty)} \,\diffdchar  \cP_0^{(1)} < \frac{1}{2}$. 
We select $\lambda^\prime<0$ such that $\eta^{-1}(\frac{1}{2}+\lambda^\prime) = \mu_1^{(1)}$. Then we have $\int_{[\eta^{-1}(\frac{1}{2}+\lambda^\prime),\infty)} \,\diffdchar  \cP_1^{(1)} = \int_{[\mu_1^{(1)},\infty)} \,\diffdchar  \cP_1^{(1)}= \frac{1}{2}$, while $\int_{[\eta^{-1}(\frac{1}{2}-\frac{\lambda^\prime}{2(1-q)}),\infty)} \,\diffdchar  \cP_0^{(1)} <0$. 
Thus we get $g_1(\lambda^\prime) < 0$. Combining $g_1(0)>0$ and intermediate value theorem results in $\lambda^\prime<g^{-1}_1(0) < 0$.  Then we obtain
\begin{equation}\label{Conclusion in case1}
    \begin{split}
\mu_1^{(1)} & = \eta^{-1}(\frac{1}{2} + \lambda^\prime)  <\eta^{-1}(\frac{1}{2} + g^{-1}_1(0)) < 0 \\
   -\mu_{1}^{(1)}     &=\eta^{-1}(\frac{1}{2}-\lambda^\prime)>\eta^{-1}(\frac{1}{2}-g^{-1}_1(0))>0 ~~~~~\text{($\eta(x)-\frac{1}{2}$ is odd)}
    \end{split}.
\end{equation}

By plugging~\eqref{Conclusion in case1} into~\eqref{g_i_example}, we have the other side of $g_0(g_1^{-1}(0)) < 0$ is bounded by $ g_0(\lambda^\prime) = \Phi(\frac{\mu_1^{(1)} - \mu_1^{(1)}}{\sigma^{(1)}}) - \Phi(\frac{-\mu_1^{(1)} - \mu_0^{(1)}}{\sigma^{(1)}})$. Since $\sigma^{(1)} \gg |\mu_1^{(1)}|, |\mu_1^{(0)}|$, we get $g_0(g^{-1}_1(0)) \approx 0$.

\begin{case}
$0<\mu_1^{(1)} < \mu_0^{(1)}$: on client 1, the local classifier is favoring group 0 over group 1, and the positive rate of both groups are above $\frac{1}{2}$.
\end{case}
This proof of this case is similar to Case 1.

\begin{case}
$\mu_1^{(1)} < 0 < \mu_0^{(1)} $: with respect to the local classifier trained by client 1, the positive rate of group 0 is above $\frac{1}{2}$ while that of group 1 is under $\frac{1}{2}$.
\end{case}
Without any loss of generality, we assume $|\mu_1^{(1)}|<|\mu_0^{(1)}| $. Select $\lambda^{\prime\prime}
<0$ such that $\eta(\frac{1}{2}-\lambda^{\prime\prime})=\mu_0^{(1)}.$ Clearly $\int_{[\eta^{-1}(\frac{1}{2}-\lambda^{\prime\prime}),+\infty)} \,\diffdchar  \cP_0^{(1)} = \int_{[\mu_0^{(1)},+\infty)} \,\diffdchar  \cP_0^{(1)} = \frac{1}{2}$, while
\[\int_{[\eta^{-1}(\frac{1}{2}+\lambda^{\prime\prime}),+\infty)} \,\diffdchar  \cP_1^{(1)} = \int_{[-\mu_0^{(1)},+\infty)} \,\diffdchar  \cP_1^{(1)} >\int_{[-\mu_1^{(1)},+\infty)} \,\diffdchar  \cP_1^{(1)} > \frac{1}{2}.\]
Consequently, we get $g_1(\lambda^{\prime\prime}) < 0$ and $\lambda^{\prime\prime}<g_1^{-1}(0) < 0$. Then we draw the same conclusion as~\eqref{Conclusion in case1}. Therefore, $g_0(g_1^{-1}(0)) \approx 0$.

Combining all three cases above, we get $\delta >0$. 
Then applying Lemma~\ref{lemma:ufl_limit} we complete the proof.
\end{proof}

\subsubsection{Comparison between \ufl{} and CFL}\label{appendix:ufl_cfl}
In this section, we compare the performance of \ufl{} and CFL. In Lemma~\ref{lemma:ufl=CFL_condition} and Lemma~\ref{lemma:veps>g(0)}, we illustrate the conditions for \ufl{} to have the same performance as CFL. In Lemma~\ref{lemma:max\ufl{}}, Lemma~\ref{lemma:uflcfl} and Lemma~\ref{lemma:ufl_q_different}, we illustrate the scenarios when CFL outperforms \ufl{}.

To do the comparison, first we introduce some additional notations. Define the accuracy of a classifier $f$ as $$\text{Acc}(f) = \bP(\haty = \rvy) = \bP(\rvy = 0)\bE_{\rvx, \rva\mid \rvy= 0}[1-f(\rvx,\rva)] + \bP(\rvy = 1)\bE_{\rvx, \rva \mid \rvy = 0}f(\rvx,\rva).$$ 

Given the required global demographic disparity $\veps$, define the performance of $\fufl$ as:
\begin{equation}
    \text{\ufl{}}(\veps_0, \veps_1; \veps) = \begin{cases} 
      \text{Acc}(\fufl) & \text{DP Disp}(\fufl) \leq \veps \\
      0 & \ow.
   \end{cases},
\end{equation}

and define performance of $\fcfl$ as: $$\text{CFL}(\veps) = \text{Acc}(f^{\text{CFL}}_{\veps}).$$ 

Now we are able to compare the performance between \ufl{} and CFL with the metric $\text{\ufl{}}(\veps_0, \veps_1; \veps)$ and $\text{CFL}(\veps)$. 
In particular, we will show that, under some mild conditions, $\max_{\veps_0, \veps_1}\text{\ufl{}}(\veps_0, \veps_1; \veps) < \text{CFL}(\veps)$, which implies the gap between \ufl{} and CFL is inevitable. 

We begin with the following two lemmas, which describe the cases that $\text{\ufl{}}(\veps_0, \veps_1; \veps) = \text{CFL}(\veps)$.

\begin{lemma}\label{lemma:ufl=CFL_condition}
Let $q\in(0,1)$. Given an \ufl{} classifier $\fufl$ such that $\text{DP Disp}(\fufl) \leq \veps$ and a CFL classifier $\fcfl$, we have $\text{\ufl{}}(\veps_0, \veps_1; \veps) \leq \text{CFL}(\veps)$. 
The equality holds if and only if $\lbduflo=\lbdufll=\lbdcfl$, where $\lbdcfl$ is defined in~\eqref{lambda}, $\lbduflo,\lbdufll$ are defined in~\eqref{lambda_i}.
\end{lemma}

\begin{proof}
The proof is straightforward.
Clearly, since $\fcfl$ is the optimizer to~\ref{prob:cfl}, we have $\text{\ufl{}}(\veps_0, \veps_1; \veps) \leq \text{CFL}(\veps)$. 
By Lemma~\ref{lemma:solution}, according to the form of the solution to~\ref{prob:cfl}, $\fufl$ is the solution to~\ref{prob:cfl} if and only if $\lbduflo=\lbdufll=\lbdcfl$. Thus complete the proof.
% When $\lambda_0=\lambda_1=\lambda$, we have $f_0^{\veps_0} = f_1^{\veps_1}=\fcfl$, then clearly $\text{Acc}(\fufl) = \text{CFL}(\veps)$.

% When $\text{Acc}(\fufl) = \text{CFL}(\veps)$, $\fufl$ is a solution to~\ref{prob:cfl}. We prove $\lambda_0=\lambda_1$ by contradiction argument. Suppose $\lambda_0\neq \lambda_1$, without any loss of generality, let $\lambda_0<\lambda_1$. Then $f_0^{\veps_0}(x,0) = 0$ for $\eta(x)\leq \frac{1-\frac{\lambda_0}{1-q}}{2}$, and $f_1^{\veps_1}(x,0)=1$ for $\eta(x)>\frac{1-\frac{\lambda_1}{1-q}}{2}$. Thus we have
% \[\fufl(x,0) = \frac{1}{2}[f_0^{\veps_0}(x,0) + f_1^{\veps_1}(x,0)]=\frac{1}{2} \text{  for  }    \frac{1-\frac{\lambda_1}{1-q}}{2} <\eta(x)\leq \frac{1-\frac{\lambda_0}{1-q}}{2}.\]
% Thus by Lemma~\ref{lemma:unique_solution}, $\fufl$ is not a solution to~\ref{prob:cfl}. We conclude $\lambda_0=\lambda_1$ by contradiction.

% Again by Lemma~\ref{lemma:unique_solution}, we have $f_{\veps_0,\veps_1}^{\text{\ufl{}}} = \fcfl$ a.e, then we conclude $\lambda_0=\lambda_1=\lambda$.
\end{proof}

\begin{lemma}\label{lemma:veps>g(0)}
Let $q\in(0,1)$. If the ERM is already fair, \ie, $\veps\geq |g(0)|$, then
\[\max_{\veps_0,\veps_1} \text{\ufl{}}(\veps_0, \veps_1; \veps) = \text{CFL}(\veps).\]

\end{lemma}

\begin{proof}
Since the ERM is already fair, $\fcfl$ is ERM$= \indicator{\eta(x) > 1/2}$. Therefore, we take $\veps_0 = \veps_1 = 1$, and $\fufl$ also equals to ERM. Thus we conclude the lemma.

% When $\veps\geq |g(0)|$, by Lemma~\ref{lemma:different_sign} we have $\lambda=0$. Then we pick $\veps_0=\veps_1=1$. Clearly $\veps_0 \geq |g_0(0)|$ and $\veps_1 \geq |g_1(0)|$. Again by Lemma~\ref{lemma:different_sign}, we have $\lambda_0=\lambda_1=0$. By Lemma~\ref{lemma:ufl=CFL_condition} we have $\text{Acc}(\fufl) = \text{CFL}(\veps)$ with $\veps_0=\veps_1=1$. Then we conclude the lemma.

\end{proof}

The next two lemmas describes the cases that $\text{\ufl{}}(\veps_0, \veps_1; \veps) < \text{CFL}(\veps)$.

\begin{lemma}\label{lemma:max\ufl{}}
Let $q\in(0,1)$. If $g_0(0)g_1(0) < 0$, $\max_{\veps_0, \veps_1} \text{\ufl{}}(\veps_0, \veps_1; \veps) <  \text{CFL}(\veps)$ for all $\veps < |g(0)|$.
\end{lemma}
\begin{proof}
In this proof, we only consider the case that $g_1(0)>0$, $|g_1(0)| \geq |g_0(0)|$. 
The proof for $g_1(0)>0$ or $|g_1(0)| \geq |g_0(0)|$ is similar. 
Next, we divide the proof into two cases. 
\begin{case}
$\max_{\veps_0, \veps_1} \text{\ufl{}}(\veps_0, \veps_1; \veps) = 0$: \ufl{} cannot achieve $\veps$ global demographic disparity.
\end{case}
 The conclusion holds.

\begin{case}
$\max_{\veps_0, \veps_1} \text{\ufl{}}(\veps_0, \veps_1; \veps) > 0$: \ufl{} can achieve $\veps$ global demographic disparity.
\end{case} 
Since $\veps < |g(0)|$, by Lemma~\ref{lemma:different_sign}, we have $\lbdcfl\neq 0$. 
Next, we solve $\fufl$ by solving the local version of~\ref{prob:cfl}. 
Combining Lemma~\ref{lemma:different_sign},  $g_1(0)>0$ and $g_0(0)<0$ yields $\lbduflo \geq 0, \lbdufll \leq 0$. If $\lbduflo = \lbdufll$, then $\lbduflo = \lbdufll = 0 \neq \lbdcfl$. Thus, we conclude the lemma by applying Lemma~\ref{lemma:ufl=CFL_condition}.
\end{proof}
\begin{remark}
Lemma~\ref{lemma:max\ufl{}} implies that if ERM is favoring different groups in different clients, there exists an inevitable gap between the performance of \ufl{} and that of CFL. 
\end{remark}

\begin{lemma}\label{lemma:uflcfl}
Let $q\in(0,1)$. Let $\tau = \min\big\{ |g(0)|, \max\{\sign(g(0))g(g_0^{-1}(0)), \sign(g(0))g(g_1^{-1}(0))\}  \big\}$. 

If $g_0(0)g_1(0) > 0$, we have
\begin{equation}
        \max_{\veps_0, \veps_1} \text{\ufl{}}(\veps_0, \veps_1; \veps)  
        \begin{cases} 
      = \text{CFL}(\veps) & \text{for all } \veps \geq \tau\\
      <  \text{CFL}(\veps) & \ow.
   \end{cases}.
\end{equation}
\end{lemma}
\begin{proof}
Without any loss of generality, assume $g_0(0), g_1(0) > 0$. Then by Lemma~\ref{lemma:different_sign}, we have $\lbduflo\leq 0, \lbdufll \leq 0$ for all $\veps_0, \veps_1 \in [0,1]$, and $g(0) = (g_0(0) + g_1(0))/2 > 0$.

To study the performance of \ufl{} when $g_0(0)g_1(0) < 0$, recall that we use $f_i^{\veps_i}$ to denote the local classifier trained by client $i$ in \ufl{} analysis. 
Therefore, $f_i^0$ is the local classifier trained by client $i$ that achieves perfect local fairness. 

Next, we discuss two cases to prove the result. In Case 1, we will show that $\tau = 0$, and then prove that $\max_{\veps_0, \veps_1} \text{\ufl{}}(\veps_0, \veps_1; \veps) = \text{CFL}(\veps)$; in Case 2, we will show that $\tau > 0$, and then prove that $\max_{\veps_0, \veps_1} \text{\ufl{}}(\veps_0, \veps_1; \veps) < \text{CFL}(\veps)$ when $\veps < \tau$.

\begin{case}
$f_0^0 = f_1^0$: the two local classifiers that achieve perfect local fairness are equal.
\end{case} 

When $\veps\geq g(0)$, the conclusion holds by directly applying Lemma~\ref{lemma:veps>g(0)}. Therefore, in what follows, we focus on the case that $\veps < g(0)$.

Next, we will first, show that when $\veps = 0$, $\lambda_0^{\mathrm{\ufl{}}_0}=\lambda_0^{\mathrm{\ufl{}}_1} = \lambda_0^{\mathrm{CFL}}$, which implies that $f_0^0 = f_1^0 = f_0^{\mathrm{CFL}}$.

Since $f_0^0 = f_1^0$, we have $\lambda_0^{\mathrm{\ufl{}}_0} = \lambda_0^{\mathrm{\ufl{}}_1}$, and thus $g_0^{-1}(0) = g_1^{-1}(0)$. 
Consequently, we get
$$ g(\lambda_0^{\mathrm{\ufl{}}_0}) = g(\lambda_0^{\mathrm{\ufl{}}_1})= g(g_0^{-1}(0)) = \frac{g_0(g_0^{-1}(0)) + g_1(g_1^{-1}(0))}{2} = 0 = g(\lambda_0^{\mathrm{CFL}}).$$ 
By the monotonicity of $g$, we have $\lambda_0^{\mathrm{\ufl{}}_0}=\lambda_0^{\mathrm{\ufl{}}_1} = \lambda_0^{\mathrm{CFL}}$ and $\tau = 0$. 
Next, we will show that, for $\veps \neq 0$, there also exists $\veps_0, \veps_1 \in [0,1]$ such that $\lbduflo = \lbdufll = \lbdcfl$, which implies $f_{\veps_0}^{\mathrm{\ufl{}}_0} = f_{\veps_0}^{\mathrm{\ufl{}}_1} = \fcfl$. 

Consider $\veps \neq 0$. Select $\veps_i = g_i(\lbdcfl), i=0,1.$ 
By Lemma~\ref{lemma:solution} and the monotonicity of $g$, we have $ \lambda_0^{\mathrm{CFL}} < \lbdcfl <0 $. Therefore, $\veps_i = g_i(\lbdcfl) > g_i(\lambda_0^{\mathrm{CFL}}) =  0$. 
By Lemma~\ref{lemma:solution}, we get $\lbdufl = g_i^{-1}(\veps_i)$ for $i = 0,1$. By the selection of $\veps_i$, we have $\lbdcfl= g_i^{-1}(\veps_i)$. 
Therefore, $\lbdcfl = \lbduflo = \lbdufll$ when $\veps \neq 0$. 

Combining all the discussion above yields $\fufl = \fcfl$ for all $\veps < g(0)$, thus we conclude $\max_{\veps_0, \veps_1} \text{\ufl{}}(\veps_0, \veps_1; \veps) = \text{CFL}(\veps)$ for all $\veps < g(0)$. Consequently, the lemma holds under Case 1. 

\begin{case}
$f_0^0 \neq f_1^0$: the two local classifiers that achieve perfect local fairness are different.
\end{case} 
The key idea of this proof is: when $\md_0(\fcfl), \md_1(\fcfl)\geq 0$, then we can always select $\veps_0 = g_0(\lbdcfl) = \md_0(\fcfl) > 0, \veps_1 = g_1(\lbdcfl) = \md_1(\fcfl) > 0$ such that $\lbduflo = g_0^{-1}(\veps_0) = \lbdcfl$ and $\lbdufll = g_1^{-1}(\veps_1) = \lbdcfl$, thus $\fufl = \fcfl$; when $\md_0(\fcfl)\md_1(\fcfl)=g_0(\lbdcfl)g_1(\lbdcfl)<0$, however, by Lemma~\ref{lemma:different_sign}, for all $\veps_0, \veps_1 \in [0,1]\times [0,1]$, we have $g_0(\lbduflo)g_1(\lbdufll)>0>g_0(\lbdcfl)g_1(\lbdcfl)$, thus there exist $i\in\set{0,1}$ such that $\lbdufl\neq \lbdcfl$ and $\fufl\neq\fcfl$. 
Next, we will give rigorous proof. 

Since $f_0^0 \neq f_1^0$, we have $\lambda_0^{\mathrm{\ufl{}}_0} \neq \lambda_0^{\mathrm{\ufl{}}_1}$. 
By Lemma~\ref{lemma:solution}, we get $g_0^{-1}(0) \neq g_1^{-1}(0)$. 
Without any loss of generality, assume $g_0^{-1}(0) < g_1^{-1}(0)$, which implies $g_1(g_0^{-1}(0)) < g_1(g_1^{-1}(0))=0$ and $g_0(g_0^{-1}(0)) = 0 < g_0(g_1^{-1}(0))$. Thus we get
\[g_1(g_0^{-1}(0)) < 0 < g_0(g_1^{-1}(0)).\]

Combining the inequality above and $g=\frac{g_0+g_1}{2}$ yields
\begin{equation}\label{gg_0 and gg_1}
g(g_0^{-1}(0)) = g_1(g_0^{-1}(0)) < 0 = g(g^{-1}(0)) <  g_0(g_1^{-1}(0)) = g(g_1^{-1}(0)).
\end{equation}
Thus we have
\[\tau = \max\{\sign(g(0))g(g_0^{-1}(0)), \sign(g(0))g(g_1^{-1}(0))\} = g(g_1^{-1}(0)).\]

%From the monotonicity of $g$, we get
%\[g_0^{-1}(0) < g^{-1}(0) < g_1^{-1}(0).\]
When $\veps \geq |g(0)|$, by Lemma~\ref{lemma:veps>g(0)}, clearly we have $\max_{\veps_0, \veps_1} \text{\ufl{}}(\veps_0, \veps_1; \veps) = \text{CFL}(\veps)$.

For the other case $\veps < |g(0)|$, by Lemma~\ref{lemma:different_sign} we have $\lambda = g^{-1}(\veps)$. Similar to Case 1, in order to achieve $\lbduflo=\lbdufll=\lbdcfl$, we select $\veps_i = g_i(\lbdcfl)$.

When $\veps < g(g_1^{-1}(0))$,  
\[g_1(\lbdcfl) = g_1(g^{-1}(\veps)) < g_1(g^{-1}(g(g_1^{-1}(0)))) = 0.\] 
Since $g_1(0)>0$, by Lemma~\ref{lemma:different_sign} we have $g_1(\lbdufll)\geq 0$, from the monotonicity of $g_1$ we conclude $\lbdufll\neq \lbdcfl$. 
By Lemma~\ref{lemma:ufl=CFL_condition}, we have $\max_{\veps_0, \veps_1} \text{\ufl{}}(\veps_0, \veps_1; \veps) < \text{CFL}(\veps)$ for all $\veps \leq \tau$. 

When $\veps \geq g(g_1^{-1}(0))$, applying~\eqref{gg_0 and gg_1} we have 
\[g_i(0)>\veps_i = g_i(\lambda) = g_i(g^{-1}(\veps)) \geq g_i(g^{-1}(g(g_i^{-1}(0)))) = 0 ,\] 
where the first inequality comes from Case 1. Thus, $\lambda_i=g^{-1}(\veps_i)=\lambda$ as desired, and we obtain $\fufl = \fcfl$. 

Combining both cases above yields the desired conclusion.
\end{proof}
\begin{remark}
Recall that we use $f_i^{\veps_i}$ to denote the local classifier trained by client $i$ in \ufl{} analysis. 
In the expression of $\tau$:
\begin{equation}
    \min\big\{ |g(0)|, \max\{\sign(g(0))g(g_0^{-1}(0)), \sign(g(0))g(g_1^{-1}(0))\}  \big\},
\end{equation}
$|g(0)|$ is the demographic disparity of ERM, $\sign(g(0))g(g_0^{-1}(0))$ is the demographic disparity of local classifier $f_0^{0}$, and $\sign(g(0))g(g_1^{-1}(0))$ is the demographic disparity of local classifier $f_1^{0}$. 
According to the proof of Lemma~\ref{lemma:uflcfl}, we obtain $\max\{\sign(g(0))g(g_0^{-1}(0)), \sign(g(0))g(g_1^{-1}(0)) \}> 0$ if and only if two local classifiers which achieves perfect local fairness is equal, \ie, $f_0^{0} = f_1^{0}$.
Therefore, Lemma~\ref{lemma:uflcfl} implies that, if the ERM is favoring the same group on different clients and the two local classifiers which achieve perfect local fairness are unequal, then \ufl{} performs strictly worse than CFL when the required demographic disparity is smaller than a certain value.
\end{remark}

So far we assume $\rva \sim \Ber(q)$ for both client 0 and client 1. When both clients do not share the same $q$, we can conclude that CFL outperforms \ufl{} in the following lemma.

\begin{lemma}\label{lemma:ufl_q_different}
Assume $\rva \sim \Ber(q_i)$ in client $i$, and $q_0 \neq q_1 \in (0,1)$. Then $\max_{\veps_0, \veps_1} \text{\ufl{}}(\veps_0, \veps_1; \veps) < \text{CFL}(\veps)$ for all $\veps < |g(0)|$.
\end{lemma}
\begin{proof}
We assemble the dataset from two clients to have $\rvx \mid \rva = a \sim \cP_a = \frac{q_0}{q_0+q_1}\cP_a^{(0)} + \frac{q_1}{q_0+q_1}\cP_{a}^{(1)}$, $a=0,1$. 
We let $\fcfl$ be the solution to~\ref{prob:cfl}, with $q=\frac{q_0+q_1}{2}$:
\[\fcfl = \indicator{s(x,a) > 0},\] 
\[\text{where } s(x,0) = 2\eta(x) - 1 + \frac{\lbdcfl}{1-q}, \quad s(x,1) = 2\eta(x) - 1 - \frac{\lbdcfl}{q} .\]

Given an \ufl{} classifier $\fufl=(f_0^{\veps_0}+f_1^{\veps_1})/2$ such that $\disp(\fufl)\leq \veps$, the solution reads
\[f_i^{\veps_i} = \indicator{s_i(x,a) > 0},\] 
\[\text{where } s_i(x,0) = 2\eta(x) - 1 + \frac{\lbdufl}{1-q_i}, \quad s_i(x,1) = 2\eta(x) - 1 - \frac{\lbdufl}{q_i} .\]
We prove the lemma by contradiction argument. If $\text{Acc}(\fufl)=\text{CFL}(\veps)$, then $\fufl$ is a solution to~\ref{prob:cfl} with $q=\frac{q_0+q_1}{2}$. 
Since $\veps < |g(0)|$, by Lemma~\ref{lemma:different_sign} we have $\lbdcfl \neq 0$. Without any loss of generality, assume $\lbdcfl<0$. Below we discuss three cases.

\begin{case}
$\lbduflo=\lbdufll=0$: the \ufl{} classifier is ERM. 
\end{case}
In this case, $f_0^{\veps_0} = f_1^{\veps_1}$. 
We have $\fufl(x,0) = 1$ for $\eta(x)>\frac{1}{2}$, and $\fcfl = 0$ for $\eta(x)<(1-\lbdcfl/(1-q))/2$. 
By Lemma~\ref{lemma:unique_solution}, $\fufl$ is not a solution to~\ref{prob:cfl}.

\begin{case}
$\lbduflo \neq 0$ or $\lbdufll \neq 0$, and $\lbduflo\lbdufll=0$: the \ufl{} classifier is not ERM, but one of the local classifier is ERM.
\end{case}
Without any loss of generality, let $\lbduflo=0$, $\lbdufll < 0$. 
Then $f_0^{\veps_0}(x,0)=1$ for $\eta(x)>\frac{1}{2}$, while $f_1^{\veps_1}(x,0)=0$ for $\eta(x)< (1-\lbduflo(1-q_1))/2.$ 
Thus we get
\[\fufl(x,0)=\frac{1}{2} \text{ for }     \frac{1}{2} <\eta(x)< (1-\lbdufll(1-q_1))/2.\]
By Lemma~\ref{lemma:unique_solution}, $\fufl$ is not a solution to~\ref{prob:cfl}.

\begin{case}
$\lbduflo\lbdufll \neq 0$: The local classifiers are not ERM. 
\end{case}
When $\frac{\lbduflo}{1-q_0} \neq \frac{\lbdufll}{1-q_1}$, without loss of generality, let $\frac{\lbduflo}{1-q_0} > \frac{\lbdufll}{1-q_1}$. 
Then by the same argument in Case 2, we have
\[\fufl(x,0)=\frac{1}{2} \text{ for }    \frac{1-\frac{\lbduflo}{1-q_0}}{2} <\eta(x)< \frac{1-\frac{\lbdufll}{1-q_1}}{2}.\]
By Lemma~\ref{lemma:unique_solution}, $\fufl$ is not a solution to~\ref{prob:cfl}.

When $\frac{\lbduflo}{q_0} \neq \frac{\lbdufll}{q_1}$, similarly, $\fufl$ is not a solution to~\ref{prob:cfl}.

When $\frac{\lbduflo}{q_0} = \frac{\lbdufll}{q_1}$ and $\frac{\lbduflo}{1-q_0} = \frac{\lbdufll}{1-q_1}$, since $\lbduflo\lbdufll\neq 0$, we have
\[\frac{q_0}{\lbduflo} = \frac{q_1}{\lbdufll},\quad \frac{1-q_0}{\lbduflo} = \frac{1-q_1}{\lbdufll},\]
which leads to $\frac{1}{\lbduflo} = \frac{1}{\lbdufll}$ and thus $q_0= q_1$. This contradicts with the assumption that $q_0\neq q_1$.

Combining all three cases yields desired conclusion.

\end{proof}

\subsection{\fflfedavg{} analysis}\label{appendix:ffl}
In this section, we analyze \fflfedavg{} for the case of two clients. For purpose of illustration, with $I=2$, we denote $\fffl=\ffflm$ to be the solution to \ref{prob:ffl}. In Sec.~\ref{appendix:ffl_improve}, we present a formal version of Thm.~\ref{thm:uflvsffl} and show that compared to \ufl{}, \fflfedavg{} has strictly higher fairness. 
In Sec.~\ref{sec:fflbestclassifier} we derive the solution to~\ref{prob:ffl}, and show it is equivalent to \fflfedavg{}.
In Sec.~\ref{appendix:ffl_limit}, we analyze the limitation of \fflfedavg{} and present a formal version of Lemma~\ref{lemma:ffl_limit}.

\subsubsection{Improve fairness via federated learning}\label{appendix:ffl_improve}
Different to \ufl{}, \fflfedavg{} can reach any $\veps$ demographic disparity: 
\begin{theorem}[Formal version of Thm.~\ref{thm:uflvsffl} under two clients cases]\label{thm:fflfair}
Let $q\in(0,1)$. For all $\veps \in [0,1]$, there exists $\veps_0, \veps_1\in[0,1]$ such that $\text{DP Disp}(\fffl) \leq \veps$. Thus under the condition in Lemma~\ref{lemma:ufl_limit}, we have
\[\min_{\veps_0, \veps_1 \in [0,1]} \disp (\fufl) > \min_{\veps_0, \veps_1 \in [0,1]} \disp(\fffl) = 0 .\]
\end{theorem}
\begin{proof}
For any $\veps \in [0,1]$, let $\veps_0 = \veps_1 = \veps$. 
Then the global DP disparity becomes
\begin{equation}
    \begin{aligned}
    \text{DP Disp}(\fffl) &= |\bE_{\rvx \mid \rva = 0} f(\rvx,0) - \bE_{\rvx \mid \rva = 1} f(\rvx, 1)| \\
    & = |(\int_{\cX} f(x,0) \,d\cP_{0}^{(0)} + \int_{\cX} f(x,0) \,d\cP_{0}^{(1)})/ 2 \\
    & - (\int_{\cX} f(x,1) \,d\cP_{1}^{(0)} + \int_{\cX} f(x,1) \,d\cP_{1}^{(1)})/ 2| \\
    & =  |\text{MD}_0(\fffl) /2 + \text{MD}_1(\fffl) /2| \\
    & \leq \text{DP Disp}_0(\fffl)/2 + \text{DP Disp}_1(\fffl)/2\\
    & \leq (\veps_0 + \veps_1) /2 = \veps.
    \end{aligned}
\end{equation}
\end{proof}

\subsubsection{The best classifier of \fflfedavg{}}\label{sec:fflbestclassifier}
For \fflfedavg{}, we directly consider multi-client cases. 
To visualize the gap between \ufl{}, \fflfedavg{}, and CFL, in our numerical experiments, we draw finite samples from Gaussian distribution, and then we optimize the empirical risk with the fairness constraint to obtain the classifier trained by \fflfedavg{}. 
The following lemma provides the solution to~\ref{prob:ffl} when $\cX$ is finite. 

\begin{lemma}\label{lemma:FFL_soln}
For finite $\cX$, the solution to~\ref{prob:ffl} is given by
\begin{align}
   f(x,a)    &  =  \indicator{\sum_{i\in [I]} s_i(x,a)p^{(i)}_a(x) > 0} , \label{solution to ffl}
\end{align}
where $s_i(x,a) = 2\eta(x)-1+I\lambda_i \frac{\indicator{\rva = 0}}{1-q} - I\lambda_i \frac{\indicator{\rva = 1}}{q}$, for certain $\lambda_0,\dots \lambda_{I-1}\in\bR$.
\end{lemma}

\begin{proof}
This proof is based on \citet{menon2018cost}. 
The key idea of this proof is to use the Lagrangian approach. 
Before we applying the Lagrangian approach, we will show that~\ref{prob:ffl} is expressible as a linear program, and thus the strong duality holds.

Since $\cX$ is finite, $f$ is a vector of finite dimension. Based on the proof of Lemma~\ref{lemma:solution}, the error rate can be written as 
$$
\bP(\haty \neq \rvy) = \sum_{x\in \cX, a\in\cA} f(x,a)(1-2\eta(x))\bP(\rvx = x, \rva = a) + \bP(\rvy = 1),
$$
and the fairness constraints can be written as 
\begin{equation}
\begin{split}
    & \bP(\haty = 1\mid \rva = 0, \rvi = i) - \bP(\haty = 1\mid \rva = 1, \rvi = i) \\
    &= \sum_{x\in\cX} \left[f(x,0)\frac{\bP(\rvx = x\mid \rva = 0, \rvi = i)}{\bP(\rva = 0)} - f(x,1)\frac{\bP(\rvx = x\mid \rva = 1, \rvi = i)}{\bP(\rva = 1)} \right],
\end{split}
\end{equation}
for $i \in [I]$. Let $u(x,a) = (1-2\eta(x))\bP(\rvx = x, \rva = a)$, $u^\prime = \bP(\rvy = 1)$, and $v_i(x,a) = (\indicator{a=0} - \indicator{a=1})\bP(\rvx = x \mid \rva = a, \rvi = i)/\bP(\rva = a)$, for $x\in\cX, a\in\cA, i \in [I]$. Note that $u, v_0, v_1$ are vectors of the same dimension of $f$. For ease of notation, we allow $\leq$ to be applied to pairs of vectors in an element-wise manner.
Therefore, the optimization is 
\begin{align}
    \min_{f}& ~ u^\top f + u^\prime \\
    \st.~& v_i^\top f \leq \veps_i \\
    & 0\leq f \leq 1,
\end{align}
which is a linear objective with linear constraint. Therefore, the strong duality holds for~\ref{prob:ffl}. Next, we apply Lagrangian approach to solve the~\ref{prob:ffl}.

Recall that
\begin{equation}
    \begin{aligned}
 \bP(\haty \neq \rvy) 
    =&\frac{1}{I}\bE_{\rva} \bE_{\rvx\sim \cP^{(i)}_{\rva}} f(\rvx, \rva) (1-2\eta(\rvx))  + \bP(\rvy = 1),
    \end{aligned}
\end{equation}
and 
\begin{equation}
    \begin{aligned}
    & \bP(\haty =1\mid  \rva = 0, \rvi=i) - \bP(\haty=1\mid \rva = 1,\rvi=i) \\
    =& \bE_{\rva} \bE_{\rvx \sim \cP^{(i)}_{\rva}} \left[f(\rvx,0)\frac{\indicator{\rva = 0}}{1-q} - f(\rvx,1)\frac{\indicator{\rva = 1}}{q}\right].
    \end{aligned}
\end{equation}

By strong duality, for $\lambda_0^\prime,\dots, \lambda_{2I-1}^\prime \geq 0$, the corresponding Lagrangian version of~\ref{prob:ffl} is 
\begin{align}\label{eq:lagrangianffl}
   &  \min_{f\in\cF} \bP(\haty \neq \rvy)   - \sum_{i\in [I]}\Big[ \lambda_{2i}^\prime [\bP(\haty = 1\mid \rva = 0, \rvi = i) - \bP(\haty = 1 \mid \rva = 1, \rvi = i)-\veps_i]  \\
    & \quad \quad \quad \quad \quad \quad +\lambda_{2i+1}^\prime [\veps_i-\bP(\haty = 1\mid \rva = 0, \rvi = i) - \bP(\haty = 1 \mid \rva = 1, \rvi = i)] \Big]
\end{align} 
Let $\lambda_i = \lambda_{2i}^\prime-\lambda_{2i+1}^\prime$, $i\in [I]$, then we get
\begin{align}
~\eqref{eq:lagrangianffl}=&  \min_{f\in\cF} \bE_{\rva} \bE_{\rvx\sim \cP^{(i)}_{\rva}}       \bigg[\frac{1}{I}f(\rvx, \rva) (1-2\eta(\rvx)) -\sum_{i\in [I]}\Big(\lambda_i  f(\rvx,0)\frac{\indicator{\rva = 0}}{1-q} - \lambda_i        f(\rvx,1)\frac{\indicator{\rva = 1}}{q} \Big) \bigg] \\
=& \min_{f\in\cF} \int_{\cX} \sum_{a\in \cA} -\frac{1}{I}f(x, a) [\sum_{i\in [I]}s_i(x,a) p^{(i)}_{a}(x) ]\diffdchar x . 
\end{align}
where $s_i$ is defined in Lemma~\ref{lemma:FFL_soln}. Thus the above equation reaches its minimum at 
\[f(x,a) = \indicator{\sum_{i\in [I]}s_i(x,a) p^{(i)}_{a}(x)> 0}.\]
\end{proof}
\begin{remark}
Based on the proof of Lemma~\ref{lemma:FFL_soln},~\ref{prob:ffl} is equivalent to solving 
\begin{equation}
    \min_{f\in\cF}\bP(\haty \neq \rvy) - \sum_{i=0}^{I-1}\lambda_i(\bP(\haty = 1\mid \rva = 0, \rvi = i) - \bP(\haty = 1 \mid \rva = 1, \rvi = i)).
\end{equation}
Under certain conditions (assumptions 1 to 4 in \citet{li2020on}), we have solving~\ref{prob:ffl} is equivalent to minimizing $$\bP(\haty \neq \rvy \mid \rvi = i) - \lambda_i(\bP(\haty = 1\mid \rva = 0, \rvi = i) - \bP(\haty = 1 \mid \rva = 1, \rvi = i))$$ locally and applying \fedavg{} (Theorem 1 in \citet{li2020on}).
\end{remark}

\subsubsection{Comparison of \fflfedavg{} and CFL}\label{appendix:ffl_limit}
\begin{lemma}[Formal version of Lemma~\ref{lemma:ffl_limit} under two clients cases]\label{lemma:formal_ffl_limit}
When $\rva \mid \rvi = 0\sim \Ber(0), \rva \mid \rvi=1 \sim \Ber(1)$ and $\disp(f_1^{\text{CFL}})> 0$, we have
\[\min_{\veps_0, \veps_1} \disp (\fffl) =  \disp(f_1^{\text{CFL}}) > \min_{\veps} \disp(\fcfl) = 0.\]
\end{lemma}

\begin{proof}
Since $\rva \mid \rvi = 0\sim \Ber(0), \rva \mid \rvi=1 \sim \Ber(1)$, the constraints in~\ref{prob:ffl} vanish. When $\veps=1$, the constraint in~\ref{prob:cfl} always holds and thus also vanishes. Thus in such scenario the solution to~\ref{prob:ffl} becomes $f_1^{\text{CFL}}$. Then from the assumption we have
\[\disp(\fffl) = \disp(f_1^{\text{CFL}}) > \disp(f_0^{\text{CFL}})=0. \]

\end{proof}

\subsection{Extension to multi-client cases}\label{appendix:multi_clients}
In this subsection, we perform the analysis of \ufl{} and \fflfedavg{} for the multi-client cases. 
We present a more general version of Lemma~\ref{lemma:ufl_limit}, Thm.~\ref{thm:fflfair} and Lemma~\ref{lemma:formal_ffl_limit}.

%Assume $I$ clients, we use same setting and notations as in Sec.~\ref{sec:necessity}. Here we let $\rva \sim \Ber(q_i)$ for client $i$.

%We define the multi-clients version of \ufl{} and \fflfedavg{} as
%\begin{align}\label{prob:ufl_multi}
%    \min_{f\in\cF} \bP(\haty \neq \rvy \mid \rvi = i),~
%         \st.~|\bP(\haty = 1\mid \rva = 0, \rvi = i) - \bP(\haty = 1 \mid \rva = 1, \rvi = i)| \leq \veps_i,
%\end{align}
%and
%\begin{equation}\label{prob:ffl_multi}
%    \begin{aligned}
%         &~~~~~~~~~~~~~~~~~~~~~~~~~~~~~~~~~~\min_{f\in\cF}  \bP(\haty \neq \rvy)  \\
 %        \st.\quad & |\bP(\haty = 1\mid \rva = 0, \rvi = i) - \bP(\haty = 1 \mid \rva = 1, \rvi = i)| %\leq \veps_i. 
 %   \end{aligned}
%\end{equation}
%Here $i\in [I]$. Denote $\boldsymbol{\veps} = (\veps_0,\dots,\veps_{I-1})$. Let $\fuflm$ and $\ffflm$ be the solution of~\eqref{prob:ufl_multi} and~\eqref{prob:ffl_multi} respectively. The CFL problem is still denoted as~\ref{prob:cfl}.

The following lemma shows the fundamental limitation of \ufl{}:
\begin{lemma}[Formal version of Lemma~\ref{lemma:informal_ufl_limit}]\label{lemma:ufl_limit_multi}
Let $q\in(0,1)$. Consider a partition which divides $I$ clients into two subsets. Denote the mixture distribution of each subset as $\rvx \mid \rva = a , \rvj = j \sim \Tilde{\cP}_a^{(j)}$, where $j$ is the index of the subset, and $\Tilde{\cP}_a^{(j)}$ is a distribution, for $a,j=0,1$. Similar to two clients case, define $\Tilde{g}_j (\lambda) = \int_{[\eta^{-1}(\frac{1}{2} - \frac{\lambda}{2(1-q)}), +\infty)} \, \diffdchar\Tilde{\cP}_0^{(j)} - \int_{[\eta^{-1}(\frac{1}{2} + \frac{\lambda}{2q}), +\infty)} \, \diffdchar\Tilde{\cP}_1^{(j)}$. Consider the case that $q_i = q$ for all $i \in [I]$ and $q \in (0,1)$. Denote the proportion of the two subset as $J_0$ and $J_1$, where $J_0, J_1 >0 $ and $J_0 + J_1 = 1$. Let $c = \min\set{|\Tilde{g}_0 (0)|, |\Tilde{g}_1(0)|}$. Define $\Tilde{\psi}: [0,c]\times[0,c] \rightarrow [-1,1]$ as
\begin{equation}\label{def:psi_mul}
\begin{split}
 \Tilde{\psi}(\Tilde{\veps}_0, \Tilde{\veps}_1) &= J_0 J_1\Tilde{g}_0(\Tilde{g}_1^{-1}(\sign(\Tilde{g}_1(0)) \Tilde{\veps}_1)) + J_0 J_1\Tilde{g}_1(\Tilde{g}_0^{-1}(\sign(\Tilde{g}_0(0))\Tilde{\veps}_0)) \\
    & + J_0^2\sign(\Tilde{g}_0(0)) \Tilde{\veps}_0 + J_1^2 \sign(\Tilde{g}_1(0))\Tilde{\veps}_1.
\end{split}
\end{equation}
 If there exists a partition such that $\Tilde{g}_0(0)\Tilde{g}_1(0) < 0$ and $\Tilde{\psi}(\Tilde{\veps}_0, \Tilde{\veps}_1)(\Tilde{g}_0(0) + \Tilde{g}_1(0)) > 0 $ for all $\Tilde{\veps}_0, \Tilde{\veps}_1 \in [0, c]$, then for all $\Tilde{\veps}_0, \Tilde{\veps}_1 \in [0,1]$, $\text{DP Disp}(\fuflm) \geq \Tilde{\delta}= \min\{|\Tilde{\psi}(\Tilde{\veps}_0, \Tilde{\veps}_1)|: \Tilde{\veps}_0, \Tilde{\veps}_1 \in [0, c]\} > 0$.
\end{lemma}
\begin{proof}
By Lemma~\ref{lemma:ufl_limit}, we conclude the achievable fairness range of \ufl{} is strictly smaller than that of CFL. Therefore, pooling the datasets in one subset and perform fair learning can clearly achieve a wider range of fairness than perform fair learning on each client individually. Thus, we can consider the two subsets as two clients with  uneven amounts of data, which is almost the same case Lemma~\ref{lemma:ufl_limit} considers. Therefore, we follow the same proof idea as Lemma~\ref{lemma:ufl_limit} to prove our claim.

Denote the assembled classifier trained from two pooled datasets as $\fuflp$. 
Note that the mean difference can be expressed as
\begin{equation}\label{\ufl{}_dp_p}
   \md(\fuflp) = J_0^2\tg_0(\lbduflop) + J_0J_1\tg_0(\lbdufllp) + J_0J_1\tg_1(\lbduflop) + J_1^2\tg_1(\lbdufllp). 
\end{equation}
In the following proof, we will show, the mean difference cannot reach 0.

Without any loss of generality, assume $|\tg_0(0)| < |\tg_1(0)|$ and  $\tg_1(0)>0$. 
By $\tg_0(0)\tg_1(0)<0$ and $\tpsi(\tveps_0, \tveps_1)(\tg_0(0)+ \tg_1(0)) > 0$ for all $\tveps_0, \tveps_1 \in [0, c]$, we have $\tg_0(0)<0$ and $\tpsi(\tveps_0, \tveps_1) > 0$ for all $\tveps_0, \tveps_1 \in [0, c]$. Without any loss of generality, assume $J_0 \leq J_1$.

First, we will prove that \ufl{} achieves its lowest mean difference when $\tveps_0, \tveps_1 \in [0, c]$. 
In what follows, we consider five different cases to derive the desired result.

\begin{case}
$\tveps_0 > |\tg_0(0)|, \tveps_1 > |\tg_1(0)|$: ERM is fair on both clients.
\end{case} By~\eqref{lambda_i}, we have $\lbduflop = \lbdufllp = 0$. 
Recall $\tg_i(\cdot)$ is a monotone increasing function, combine $\tg_1(0)>0$ and Lemma~\ref{lemma:different_sign}, and thus $\tg_0(\tg_1^{-1}(0)) < \tg_0(0) < 0.$
Applying the above conclusion yields 
\[\eqref{\ufl{}_dp_p} = J_0\tg_0(0) +  J_1\tg_1(0) > \frac{1}{J_1} \left(J_0^2\tg_0(0) + J_1^2\tg_1(0) + J_0J_1\tg_0(\tg_1^{-1}(0))\right) = \frac{1}{J_1} \tpsi(\tg_0(0), 0)\geq \tdelta.\]

\begin{case}
$\tveps_0 \leq |\tg_0(0)|, \tveps_1 > |\tg_1(0)|$: ERM is unfair on client 0, but fair on client 1.
\end{case} 
Applying~\eqref{lambda_i} results in $\lbdufll=0$. 
By the fact that $\tg_i(\cdot)$ is a strictly monotone increasing function, we have $\lbduflop = \tg_0^{-1}(-\tveps_0) > \tg_0^{-1}(\tg_0(0)) = 0.$ 
Applying the above conclusion yields
\begin{equation}
    \begin{aligned}
       ~\eqref{\ufl{}_dp_p}  =& - J_0^2 \tveps_0 + J_1^2  \tg_1(0) + J_0J_1 \tg_0(0) + J_0J_1 \tg_1(\lbduflop) \\
        >& J_0\tg_0(0)+J_1\tg_1(0) > \frac{1}{J_1}\tpsi(\tg_0(0), 0) \geq \tdelta.
    \end{aligned}
\end{equation}
In the first equality we used $\lbduflop>0, \tg_1(\lbduflop) > \tg_1(0), \tg_0(0) < -\tveps_0$.
\begin{case}
$\tveps_0 \leq |\tg_0(0)|, \tveps_1 \leq |\tg_1(0)|$: ERM is unfair on both client 0 and client 1. 
\end{case} 
Applying~\eqref{lambda_i} we have $\lbduflop = \tg_0^{-1}(-\tveps_0), \lbdufllp=\tg_1^{-1}(\tveps_1)$.
Then we have
\begin{align}
  ~\eqref{\ufl{}_dp_p} & = -J_0^2\tveps_0 + J_1^2\tveps_1 + J_0J_1\tg_0(\tg_1^{-1}(\tveps_1)) + J_0J_1\tg_1(\tg_0^{-1}(-\tveps_0)) \\
   & \geq -J_0^2\tveps_0 + J_1^2\tveps_0 + J_0J_1\tg_0(\tg_1^{-1}(\tveps_0)) + J_0J_1\tg_1(\tg_0^{-1}(-\tveps_0)) =  \tpsi(\tveps_0, \tveps_0) \geq \tdelta. 
\end{align}

\begin{case}
$\tveps_0 > |\tg_0(0)|, \tveps_1 \leq |\tg_0(0)| $: ERM is fair on client 0 and very unfair on client 1.
\end{case} 
By~\eqref{lambda_i}, we have $\lbduflop = 0, \lbdufllp = \tg_1^{-1}(\tveps_1) > \tg_1^{-1}(0)$. 
Then we obtain
\begin{align}
~\eqref{\ufl{}_dp_p}   &  = J_0^2\tg_0(0) + J_0J_1\tg_0(\lbdufllp) + J_0J_1\tg_1(0) + J_1^2\tveps_1  \\
    & > J_0^2\tg_0(0) + J_0J_1\tg_0(\tg_1^{-1}(0)) + J_1^2\tg_1(0) = \tpsi(\tg_0(0),0) \geq \tdelta.
\end{align}

\begin{case}
$\tveps_0 > |\tg_0(0)|, |\tg_0(0)| \leq \tveps_1 < |\tg_1(0)|$: ERM is fair on client 0 and unfair on client 1.
\end{case} 
Applying~\eqref{lambda_i} implies $ \lbdufllp = \tg^{-1}_1(\tveps_1) > \tg_1^{-1}(0)$. 
Therefore,
\begin{align}
~\eqref{\ufl{}_dp_p}  =   & J_0^2\tg_0(0) + J_0J_1\tg_0(\tg_1^{-1}(\tveps_1)) + J_1^2\tveps_1 + J_0J_1\tg_1(0)  \\
    &> J_0^2 \tg_0(0) + J_0J_1\tg_1(0)+ J_0J_1\tg_0(\tg_1^{-1}(0)) + J_1^2\tveps_1  > \tpsi(\tg_0(0), 0) \geq \tdelta.
\end{align}
Combining all the cases above, we conclude that when $\tg_1(0)>0$, $\text{DP Disp}(\fuflp) \geq \tdelta = \min\{|\tpsi(\tveps_0, \tveps_1)|: \tveps_0, \tveps_1 \in [0, c]\} > 0$ for all $\tveps_0, \tveps_1 \in [0,1]$.

\end{proof}
\begin{remark}
Based on the proof above, we can conclude, for the cases with multiple clients, the fundamental limitation of \ufl{} still exists. 
\end{remark}

The following theorem shows that \fflfedavg{} can reach any $\veps$ DP disparity:
\begin{theorem}[Generalized version of Thm.~\ref{thm:fflfair}]\label{thm:fflfair_multi}
Let $q_i=q\in (0,1)$ for all $i\in [I]$. For all $\veps \in [0,1]$, let $\veps_i \leq \veps$ for all $i\in [I]$, then $\text{DP Disp}(\ffflm) \leq \veps$. Thus under the condition in Lemma~\ref{lemma:ufl_limit_multi}, we have 
\[\min_{\boldsymbol{\veps} \in [0,1]^I} \disp (\fuflm) > \min_{\boldsymbol{\veps} \in [0,1]^I} \disp(\ffflm) = 0 .\]
\end{theorem}
\begin{proof}
When $\veps_i\leq \veps$ the global DP disparity becomes
\begin{equation}
   \begin{aligned}
   &  \text{DP Disp}(\ffflm) = |\bE_{\rvx \mid \rva = 0} f(\rvx,0) - \bE_{\rvx \mid \rva = 1} f(\rvx, 1)| \\
    & =  |(\sum_{i=0}^{I-1} \int_{\cX} f(x,0) \,d\cP_{0}^{(i)} )/ I  - (\sum_{i=0}^{I-1} \int_{\cX} f(x,1) \,d\cP_{1}^{(i)} )/ I| \\
    & =  |\sum_{i=0}^{I-1} \text{MD}_i(\ffflm) /I | \leq \sum_{i=0}^{I-1} \text{DP Disp}_i(\ffflm)/I  \leq \sum_{i=0}^{I-1} \veps_i/I = \veps.
    \end{aligned}
\end{equation}
\end{proof}

The following theorem shows the limitation of \fflfedavg{}:
\begin{lemma}[Generalized version of Lemma~\ref{lemma:formal_ffl_limit}]\label{lemma:ffl_limit_multi}
Let $\rva \mid \rvi = i\sim \Ber(0)$ or $\rva \mid \rvi = i\sim \Ber(1)$ for all $i\in [I]$. When $\disp(f_1^{\text{CFL}})> 0$, we have
\[\min_{\boldsymbol{\veps}\in [0,1]^{I}} \disp (\ffflm) =  \disp(f_1^{\text{CFL}}) > \min_{\veps} \disp(\fcfl) = 0.\]
\end{lemma}

\begin{proof}
Since $\rva \mid \rvi = i\sim \Ber(0)$ or $\rva \mid \rvi = i\sim \Ber(1)$, the constraints in~\eqref{prob:ffl} vanish. When $\veps=1$, the constraint in~\ref{prob:cfl} always holds and thus also vanishes. Thus in such scenario the solution to~\eqref{prob:ffl} becomes $f_1^{\text{CFL}}$. Then from the assumption we have
\[\disp(\ffflm) = \disp(f_1^{\text{CFL}}) > \disp(f_0^{\text{CFL}})=0. \]

\end{proof}

\section{Appendix - FedFB analysis and algorithm description}\label{appendix:fedfb}
% We first provide a brief review of the \fb{} algorithm.
% \fb{} solves a bi-level optimization problem to learn a fair classifier on centralized data.
% The inner optimization problem solves a weighted empirical risk minimization problem where samples from different groups are reweighted by different weights.
% The outer optimization problem optimizes the weights used for the inner problem, with the goal of minimizing the unfairness of the classifier.
% \fb{} works for various group fairness definitions including demographic parity, equalized odds, and equalized opportunity. 
% For the case of demographic parity, the algorithm reduces to the following simple yet intuitive algorithm. 
% The algorithm starts with equal weights for two different groups. 
% After training a model with the initial weights, it computes the sign of the difference between the two positive prediction rates $\bP(\haty = 1\mid \rva = 0) - \bP(\haty = 1 \mid \rva = 1)$. 
% If this quantity is zero, then $\disp$ is zero, so the sample weights are not updated. 
% If this is positive, it decreases the weights for the samples whose $\rvy = 1, \rva = 0$ and increases the weights for the samples whose $\rvy = 1, \rva = 1$ so that after retraining, $\bP(\haty = 1\mid \rva = 0)$ decreases and $\bP(\haty = 1 \mid \rva = 1)$ increases. 
% And vice versa for the other case.

In this section, we provide our bi-level optimization formulation for \fedfb{} for four fairness notions: demographic parity, equal opportunity, equalized odds and client parity, and design the corresponding update rule. This development can also be applied to centralized case. 
Then, we provides more details of how we incorporate \fb{} with federated learning. 

To explain how to optimize the weights of different groups, we introduce some necessary notations first.
Recall that in the setting of Sec.~\ref{sec:fedfb} we assume $A$ sensitive attributes, here we further assume $I$ clients.
Then the sample has attributes $(\mathrm{x},\mathrm{y},\mathrm{a},\mathrm{i})$, where $\mathrm{x} \in \mathcal{X}$ is the input feature, $\mathrm{y}\in \{0,1\}$ is the label, $\mathrm{a}\in \mathcal{A}= [A]$ is the sensitive attribute and $\mathrm{i}\in [I]$ is the index of the client that the sample belongs to.
% Let $n$ be the total number of training samples, $A$ be the total number of sensitive attributes and $I$ be the total number of clients.
We further denote $n_{y,a}^{(i)}: = |\{(\mathrm{x},\mathrm{y},\mathrm{a},\mathrm{i}):\mathrm{y} = y, \mathrm{a} = a, \mathrm{i}=i\}|$ be the number of samples belong to client $i$ of label $y$ and sensitive attribute $a$. 
And we further define the local version of $L_{y,a}$ as $L_{y,a}^{(i)}(\boldsymbol{w}) := \sum_{(\mathrm{x},\mathrm{y},\mathrm{a},\mathrm{i}):\mathrm{y} = y, \mathrm{a}=a, \mathrm{i} = i}\ell (\mathrm{y}, \hat{\mathrm{y}};\boldsymbol{w}) / n_{y,a}^{(i)}$.
\subsection{FedFB \wrt demographic parity}\label{appendix:fedfb_dp}
We first provide the proof for Proposition~\ref{prop:dp}. 
\begin{proof}[Proof of Proposition~\ref{prop:dp}]
We denote by $\mathbb{P}$ the empirical probability. The demographic parity is satisfied when $\mathbb{P}(\hat{\mathrm{y}}=1\mid \mathrm{a} = 0) = \mathbb{P}(\hat{\mathrm{y}} = 1\mid \mathrm{a} = a)$ holds for all $a\in [A]$. Thus,
\begin{equation}
\begin{aligned}
    &\mathbb{P}(\hat{\mathrm{y}}=1, \mathrm{y}=0 \mid \mathrm{a}=0)+\mathbb{P}(\hat{\mathrm{y}}=1, \mathrm{y}=1 \mid \mathrm{a}=0)\\
    =&\mathbb{P}(\hat{\mathrm{y}}=1, \mathrm{y}=0 \mid \mathrm{a}=a)+\mathbb{P}(\hat{\mathrm{y}}=1, \mathrm{y}=1 \mid \mathrm{a}=a)
\end{aligned}
\end{equation}
For 0-1 loss, we have $\ell(|1-y|, \cdot) = 1-\ell(y, \cdot)$, thus
\begin{equation}
    \begin{aligned} 
    & \frac{1}{n_{\star, 0}} \sum_{(\mathrm{x},\mathrm{y},\mathrm{a}): \mathrm{y}=0, \mathrm{a}=0}\left(1-\ell\left(\mathrm{y}, \hat{\mathrm{y}};\boldsymbol{w}\right)\right)+ \frac{1}{n_{\star, 0}} \sum_{(\mathrm{x},\mathrm{y},\mathrm{a}): \mathrm{y}=1, \mathrm{a}=0} \ell\left(\mathrm{y}, \hat{\mathrm{y}};\boldsymbol{w}\right) \\
    =& \frac{1}{n_{\star, a}}\sum_{(\mathrm{x},\mathrm{y},\mathrm{a}): \mathrm{y}=0, \mathrm{a}=a}\left(1-\ell\left(\mathrm{y}, \hat{\mathrm{y}};\boldsymbol{w}\right)\right)+\frac{1}{n_{\star, a}} \sum_{(\mathrm{x},\mathrm{y},\mathrm{a}): \mathrm{y}=1, \mathrm{a}=a} \ell\left(\mathrm{y}, \hat{\mathrm{y}};\boldsymbol{w}\right) .
    \end{aligned}
\end{equation}
By replacing $\sum_{(\mathrm{x},\mathrm{y},\mathrm{a}): \mathrm{y}=y, \mathrm{a}=a} \ell\left(\mathrm{y}, \hat{\mathrm{y}};\boldsymbol{w}\right) = n_{y,a}L_{y,a}(\boldsymbol{w})$, we have
\begin{equation}
\begin{aligned}
& \frac{n_{0,0}}{n_{\star,0}}(1-L_{0,0}(\boldsymbol{w})) + \frac{n_{1,0}}{n_{\star,0}}L_{1,0}(\boldsymbol{w}) \\ = &  \frac{n_{0,a}}{n_{\star,a}}(1-L_{0,a}(\boldsymbol{w})) + \frac{n_{1,a}}{n_{\star,a}}L_{1,a}(\boldsymbol{w}) .
\end{aligned}
\end{equation}
\end{proof}

We make the following assumption to our loss function for showing the partial convergence guarantee of \fb{}.
\begin{assumption}\label{ass:dp}
$L'_{y,a}(\cdot)$ is twice differentiable for all $y\in\{0,1\}$, $a\in [A]$, and
\begin{equation}\label{ass_on_dp}
\sum_{a=0}^{A-1} \left[ \lambda_a\nabla^2L_{0,a}^\prime(\boldsymbol{w}) + (2\frac{n_{\star,a}}{n}-\lambda_a)\nabla^2L_{1,a}^\prime(\boldsymbol{w}) \right] \succ 0 \text{ for all } \lambda \in \Lambda.    
\end{equation}
\end{assumption}
If $L^\prime_{y,a}(\boldsymbol{w}_{\boldsymbol{\lambda}})$ is convex for all $y\in\{0,1\},a\in [A]$, the condition~\eqref{ass_on_dp} holds unless for all $a$, $L_{0,a}^\prime(\cdot), L_{1,a}^\prime(\cdot)$ share their stationary points, which is very unlikely~(see Remark 1 in~\citet{roh2021fairbatch}).

Recall the bi-level optimization problem in \eqref{dp:w}, we denote the outer objective function to be $F_{\mathrm{dp}}(\boldsymbol{\lambda})=\sum_{a=1}^{A-1}\big(F_a(\boldsymbol{w}_{\boldsymbol{\lambda}})\big)^2$. The following lemma provides a decreasing direction of $F_{\mathrm{dp}}$, which inspired us to design the update rule (\ref{eq:update_dp}) of $\boldsymbol{\lambda}$.

\begin{lemma}[Decreasing direction of $F_{\mathrm{dp}}$]\label{lemma:dp}
If Assumption~\ref{ass:dp} holds, then on the direction $\boldsymbol{\mu}(\boldsymbol{\lambda}) = (\mu_0(\boldsymbol{\lambda}),\dots,\mu_{A-1}(\boldsymbol{\lambda}))$ where
\begin{equation}\label{dp:mu}
\begin{split}
 \mu_0(\boldsymbol{\lambda})    & = - \sum_{a=1}^{A-1} F_a(\boldsymbol{w}_{\boldsymbol{\lambda}}) \\
 \mu_a (\boldsymbol{\lambda})  &= F_a(\boldsymbol{w}_{\boldsymbol{\lambda}})
\end{split}    
\end{equation}
for all $a\in \{1,\dots,A-1\}$, we have $\boldsymbol{\mu}(\boldsymbol{\lambda})\cdot \nabla F_{\mathrm{dp}}(\boldsymbol{\lambda}) \leq 0$, and the equality holds if only if $\boldsymbol{\mu}(\boldsymbol{\lambda})=\boldsymbol{0}$. 
\end{lemma}

\begin{proof}
We compute the derivative as
\begin{equation}
    \begin{split}
      \frac{\partial F_{\text{dp}}(\boldsymbol{\lambda})}{\partial \lambda_j}    & =2\sum_{a=1}^{A-1} \Big[\Big( -L_{0,0}^\prime(\boldsymbol{w}_{\boldsymbol{\lambda}})+L_{1,0}^\prime(\boldsymbol{w}_{\boldsymbol{\lambda}}) + L_{0,a}^\prime(\boldsymbol{w}_{\boldsymbol{\lambda}})-L_{1,a}^\prime(\boldsymbol{w}_{\boldsymbol{\lambda}})+ \frac{n_{0,0}}{n_{\star,0}} - \frac{n_{0,a}}{n_{\star,a}} \Big) \\
      & \Big(-\nabla L^\prime_{0,0}(\boldsymbol{w}_{\boldsymbol{\lambda}}) + \nabla L^\prime_{1,0}(\boldsymbol{w}_{\boldsymbol{\lambda}}) + \nabla L^\prime_{0,a}(\boldsymbol{w}_{\boldsymbol{\lambda}})-\nabla L^\prime_{1,a}(\boldsymbol{w}_{\boldsymbol{\lambda}}) \Big) \Big]  \frac{\partial \boldsymbol{w}_{\boldsymbol{\lambda}}}{\partial \lambda_j}     . \label{dp:pF_pl}
    \end{split}
\end{equation}
Note that $\boldsymbol{w}_{\boldsymbol{\lambda}}$ is the minimizer to$~\eqref{dp:w}_{\text{inner}}$, we have
\[\sum_{a=0}^{A-1} \Big[\lambda_a \nabla L_{0,a}^\prime(\boldsymbol{w}_{\boldsymbol{\lambda}}) + (2\frac{n_{\star,a}}{n}-\lambda_a) \nabla L_{1,a}^\prime(\boldsymbol{w}_{\boldsymbol{\lambda}}) \Big]=0.         \]
We take the $\lambda_j$ derivative to the above equation and have
\[ \nabla L_{0,j}^\prime(\boldsymbol{w}_{\boldsymbol{\lambda}}) - \nabla L_{1,j}^\prime(\boldsymbol{w}_{\boldsymbol{\lambda}}) +  \sum_{a=0}^{A-1} \Big[\lambda_a \nabla^2 L_{0,a}^\prime(\boldsymbol{w}_{\boldsymbol{\lambda}}) + (2\frac{n_{\star,a}}{n}-\lambda_a) \nabla^2 L^\prime_{1,a}(\boldsymbol{w}_{\boldsymbol{\lambda}}) \Big]\frac{\partial \boldsymbol{w}_{\boldsymbol{\lambda}}}{\partial \lambda_j} = 0  .\]
Thus we get
\begin{equation}\label{dp:pw_pl}
 \frac{\partial \boldsymbol{w}_{\boldsymbol{\lambda}}}{\partial \lambda_j} = \Big(\sum_{a=0}^{A-1} \Big[\lambda_a \nabla^2 L_{0,a}^\prime(\boldsymbol{w}_{\boldsymbol{\lambda}}) + (2\frac{n_{\star,a}}{n}-\lambda_a) \nabla^2 L^\prime_{1,a}(\boldsymbol{w}_{\boldsymbol{\lambda}}) \Big]\Big)^{-1}   [\nabla L_{1,j}^\prime(\boldsymbol{w}_{\boldsymbol{\lambda}}) - \nabla L_{0,j}^\prime(\boldsymbol{w}_{\boldsymbol{\lambda}})].        
\end{equation}

Then on the direction $\boldsymbol{\mu}(\boldsymbol{\lambda})$ given by~\eqref{dp:mu}, we combine~\eqref{dp:pF_pl} and~\eqref{dp:pw_pl} to have
\begin{align}
& \boldsymbol{\mu}(\boldsymbol{\lambda}) \cdot \nabla F_{\text{dp}}(\boldsymbol{\lambda})   \\
    & = 2 \bigg(\sum_{a=1}^{A-1} \Big[\Big( -L_{0,0}^\prime(\boldsymbol{w}_{\boldsymbol{\lambda}})+L_{1,0}^\prime(\boldsymbol{w}_{\boldsymbol{\lambda}}) + L_{0,a}^\prime(\boldsymbol{w}_{\boldsymbol{\lambda}})-L_{1,a}^\prime(\boldsymbol{w}_{\boldsymbol{\lambda}})+ \frac{n_{0,0}}{n_{\star,0}} - \frac{n_{0,a}}{n_{\star,a}} \Big) \\
 & \quad \Big(-\nabla L^\prime_{0,0}(\boldsymbol{w}_{\boldsymbol{\lambda}}) + \nabla L^\prime_{1,0}(\boldsymbol{w}_{\boldsymbol{\lambda}}) + \nabla L^\prime_{0,a}(\boldsymbol{w}_{\boldsymbol{\lambda}})-\nabla L^\prime_{1,a}(\boldsymbol{w}_{\boldsymbol{\lambda}}) \Big) \Big] \bigg)  \\
    &\quad \Big(\sum_{a=0}^{A-1} \Big[\lambda_a \nabla^2 L_{0,a}^\prime(\boldsymbol{w}_{\boldsymbol{\lambda}}) + (2\frac{n_{\star,a}}{n}-\lambda_a) \nabla^2 L^\prime_{1,a}(\boldsymbol{w}_{\boldsymbol{\lambda}}) \Big]\Big)^{-1}  \\
    &\quad \bigg(- \sum_{a=1}^{A-1} \Big[\Big( -L_{0,0}^\prime(\boldsymbol{w}_{\boldsymbol{\lambda}})+L_{1,0}^\prime(\boldsymbol{w}_{\boldsymbol{\lambda}}) + L_{0,a}^\prime(\boldsymbol{w}_{\boldsymbol{\lambda}})-L_{1,a}^\prime(\boldsymbol{w}_{\boldsymbol{\lambda}})+ \frac{n_{0,0}}{n_{\star,0}} - \frac{n_{0,a}}{n_{\star,a}} \Big) \\
    & \quad \Big(-\nabla L^\prime_{0,0}(\boldsymbol{w}_{\boldsymbol{\lambda}}) + \nabla L^\prime_{1,0}(\boldsymbol{w}_{\boldsymbol{\lambda}}) + \nabla L^\prime_{0,a}(\boldsymbol{w}_{\boldsymbol{\lambda}})-\nabla L^\prime_{1,a}(\boldsymbol{w}_{\boldsymbol{\lambda}}) \Big) \Big]  \bigg) \leq 0
\end{align}
where we have used Assumption~\ref{ass:dp}, and the equality holds only when $\boldsymbol{\mu}(\boldsymbol{\lambda}) = \boldsymbol{0}$.
\end{proof}

% Inspired by Lemma~\ref{lemma:dp}, in each communication round $t = 0,1, \dots$, we design update rule for $\boldsymbol{\lambda}$ as:
% \begin{equation}\label{eq:update_dp}
% \lambda_a^{(t+1)}=\lambda_a^{(t)} + \frac{\alpha_t}{\Vert\boldsymbol{\mu}(\boldsymbol{\lambda}^{(t)})\Vert_2}\mu_a (\boldsymbol{\lambda}^{(t)}), \text{ for } a \in [A],    
% \end{equation}
% where $\alpha_t$ is the step size. 
% The update rule~\eqref{eq:update_dp} is consistent with our intuition when $A=2$.

Now we introduce how clients collaborate to solve the bi-level optimization problem~\eqref{dp:w}. 

First, we focus on the outer objective function in~$\eqref{dp:w}$ and introduce how clients collaborate to update the weight $\boldsymbol{\lambda}$. Note that the central server can compute $L_{y,a}^\prime(\boldsymbol{w}_{\boldsymbol{\lambda}})$ by weight-averaging the local group loss $L_{y,a}^{(i)}(\boldsymbol{w}_{\boldsymbol{\lambda}})$ sent from clients at communication rounds as
\[L'_{y,a}(\boldsymbol{w}_{\boldsymbol{\lambda}}) = \sum_{i\in[I]} \frac{n^{(i)}_{y,a}}{n_{\star,a}}L^{(i)}_{y,a}(\boldsymbol{w}_{\boldsymbol{\lambda}}),\]
thereby obtain $\boldsymbol{\mu}(\boldsymbol{\lambda})$ and update $\boldsymbol{\lambda}$ by~\eqref{eq:update_dp}. 

Next, we focus on the inner objective function in~$\eqref{dp:w}$ and introduce how clients collaborate to update the model parameters $\boldsymbol{w}_{\boldsymbol{\lambda}}$ using \fedavg{}.
Note that we can decompose the objective function $L(\boldsymbol{w}, \boldsymbol{\lambda})$ into
$L(\boldsymbol{w},\boldsymbol{\lambda}) = \sum_{i\in [I]} L^{(i)}(\boldsymbol{w},\boldsymbol{\lambda})  $, where $
  L^{(i)}(\boldsymbol{w},\boldsymbol{\lambda})  := \sum_{a=0}^{A-1} [\lambda_a n_{0,a}^{(i)} L_{0,a}^{(i)}/n_{\star,a} + (2\frac{n_{\star,a}}{n}-\lambda_a) n_{1,a}^{(i)} L_{1,a}^{(i)}/n_{\star,a} ]$ is the client objective function of client $i$. 
The global objective can be seen as a weighted sum of the client objective function. 
Therefore, we can use \fedavg{} to solve the inner optimization problem. 

Denote $L_{y,a}^{\prime(i)}(\boldsymbol{w}) := \frac{n^{(i)}_{y,a}}{n_{\star,a}}L_{y,a}^{(i)}(\boldsymbol{w})$, we present the pseudocode of \fedfb{} \wrt{} demographic parity in Algorithm~\ref{alg:fedfb_dp}.

\begin{algorithm2e}
\SetAlgoNoLine
 \SetKwInOut{Input}{input}
 \SetKwBlock{Server}{Server executes:}{ }
 \SetKwBlock{Compute}{ClientUpdate$(i,\boldsymbol{w},\boldsymbol{\lambda})$:}{}
  \SetKwInOut{Output}{output}

 \Server{
 \Input{Learning rate $\{\alpha_t\}_{t\in\mathbb{N}}$\;}
Initialize $\lambda_a$ as $\frac{n_{\star,a}}{n}$ for all $a\in[A]\backslash \set{0}$\;
\For{each iteration $t=1,2,\dots$}{
Clients perform updates\;
$\boldsymbol{w}_{\boldsymbol{\lambda}} \leftarrow$ SecAgg$(\set{\boldsymbol{w}^{(i)}})$ for all $i$\;
$F_a \leftarrow$ SecAgg$(\set{F^{(i)}_{a}}) - (I-1)(\frac{n_{0,0}}{n_{\star, 0}} - \frac{n_{0,a}}{n_{\star, a}})$ for all $a, i$\;
$\mu_0 \leftarrow - \sum_{a=1}^{A-1}F_a$\;
$\mu_a \leftarrow F_a, \text{for all }a\in [A]\backslash \{0\}   $\;
\uIf{$t\% k = 0$}{
$\lambda_a    \leftarrow \lambda_a    + \frac{\alpha_t}{\Vert \boldsymbol{\mu}\Vert_2} \mu_a,  \text{ for all }a\in [A]$\;
Broadcast $\boldsymbol{\lambda}$ to clients
}
Broadcast $\boldsymbol{w}_{\boldsymbol{\lambda}}$ to clients;
}
\Output{$\boldsymbol{w}_{\boldsymbol{\lambda}}, \boldsymbol{\lambda}$}
 }{}
 
 \Compute{
 $\boldsymbol{w}^{(i)} \leftarrow $ Gradient descent \wrt objective function $ \sum_{a=0}^{A-1}\left[\lambda_a L_{0,a}^{\prime(i)}(\boldsymbol{w}) + (2\frac{n_{\star,a}}{n}-\lambda_a)L_{1,a}^{\prime(i)}(\boldsymbol{w})\right]$\;
 $F_a^{(i)} \leftarrow -L_{0,0}^{\prime(i)} + L_{1,0}^{\prime(i)} + L_{0,a}^{\prime(i)} - L_{1,a}^{\prime(i)}  +\frac{n_{0,0}}{n_{\star,0}} - \frac{n_{0,a}}{n_{\star,a}}$\;
 Send $\boldsymbol{w}^{(i)}, F_{0,a}^{(i)}(\boldsymbol{w})$ for all $a\in [A]$ to server via a SecAgg protocol;
 }{}

 \caption{\fedfb{}$(k,t)$ \wrt Demographic Parity}
 \label{alg:fedfb_dp}
\end{algorithm2e}

Next, we analyze the convergence performance of \fedfb{}. We need to make the following assumptions on the objective function $L^{(i)}(\boldsymbol{w},\boldsymbol{\lambda})$, $i\in[I]$. For simplicity, we drop the $\boldsymbol{\lambda}$ here and use the notations $L^{(i)}(\boldsymbol{w})$ and $L(\boldsymbol{w})$ instead. We use $\boldsymbol{w}_t^{(i)}$ to denote the model parameters at $t$-th iteration in $i$-th client. The assumptions below are proposed by work~\citet{li2020on}:

\begin{assumption}[Strong convexity, Assumption 1 in~\citet{li2020on}]\label{ass:convex}
$L^{(i)}(\boldsymbol{w})$ is $\mu$-strongly convex for $i\in[I]$, \ie, for all $\boldsymbol{v}$ and $\boldsymbol{w}$, $\boldsymbol{w}$, $L^{(i)}(\boldsymbol{v}) \geq L^{(i)}(\boldsymbol{w}) + (\boldsymbol{v} - \boldsymbol{w})^\top \nabla L^{(i)}(\boldsymbol{w}) + \frac{\mu}{2}\norm{\boldsymbol{v} - \boldsymbol{w}}_2^2 $.
\end{assumption}

\begin{assumption}[Smoothness, Assumption 2 in~\citet{li2020on}]\label{ass:smooth}
$L^{(i)}(\boldsymbol{w})$ is $L$-smooth for $i\in[I]$, \ie, for all $\boldsymbol{v}$ and $\boldsymbol{w}$, $L^{(i)}(\boldsymbol{v}) \leq L^{(i)}(\boldsymbol{w}) + (\boldsymbol{v} - \boldsymbol{w})^\top \nabla L^{(i)}(\boldsymbol{w}) + \frac{L}{2}\norm{\boldsymbol{v} - \boldsymbol{w}}_2^2 $.
\end{assumption}

\begin{assumption}[Bounded variance, Assumption 3 in~\citet{li2020on}]\label{ass:variance}
Let $\xi^{(i)}_t$ be sampled from $i$-th device's local data uniformly at random, where $t\in[T]$, and $T$ is the total number of every client's SGDs. 
The variance of stochastic gradients in each device is bounded: $\mathbb{E}\norm{\nabla L^{(i)}(\boldsymbol{w}_t^{(i)}, \xi_t^{(i)}) - \nabla L^{(i)}(\boldsymbol{w}_t^{(i)})}^2 <\infty$ for $i \in [I]$.
\end{assumption}

\begin{assumption}[Bounded gradients, Assumption 4 in~\citet{li2020on}]\label{ass:gradient}
The expected squared norm of stochastic gradients is uniformly bounded, \ie, $\mathbb{E}\norm{\nabla L^{(i)}(\boldsymbol{w}^{(i)}_t, \xi_t^{(i)})} <\infty$ for all $i\in[I], t\in [T]$.
\end{assumption}

In \fedavg{}, first, the central server broadcasts the lastest model to all clients, then, every client performs local updates for a number of iterations, last, the central server aggregates the local models to produce the new global model(see Algorithm Description in~\citet{li2020on} for more detailed explanation).

Note that we assume $\boldsymbol{\lambda}$ is not updated at each communication round. With the above assumptions, the following theorem shows the convergence of FedFB in the case of two clients.

\begin{theorem}\label{prop:formal_conv_dp}
Consider the case of $A=2$. 
Let Assumption~\ref{ass:convex},~\ref{ass:smooth},~\ref{ass:variance},~\ref{ass:gradient} on $\{L^{(i)}(\cdot,\boldsymbol{\lambda})\}_{i\in[I]}$ and Assumption~\ref{ass:dp} on $\{L^\prime_{y,a}(\cdot)\}_{y,a\in\{0,1\}}$ hold. Choose $\{\alpha_t\}_{t=1}^\infty$ such that $\lim_{t\rightarrow \infty}\alpha_t = 0, \sum_{t=1}^\infty \alpha_t = \infty$. 
Suppose we have sufficiently large number of communication rounds to solve the inner optimization problem between two $\boldsymbol{\lambda}$ update rounds, i.e, $k\to \infty$ in Algorithm~\ref{alg:fedfb_dp}. 
Denote the number of $\boldsymbol{\lambda}$ update as $T$,  $\boldsymbol{\lambda}^{(t)}=(\lambda_0^{(t)},\cdots,\lambda_{A-1}^{(t)})$ as the weight at $t$-th updated round. 
We can find sufficiently large $T$ such that with high probability, applying update rule~\eqref{eq:update_dp} leads to $ |\lambda_a^{(T)} - \lambda_a^\star| \leq \max\Big\{|\lambda_a^{(0)} - \lambda_a^\star| - \sum_{t=1}^T \alpha_t, \alpha_T \Big\} \to 0,$
where $\boldsymbol{\lambda}^\star$ is the global minimizer: $\boldsymbol{\lambda}^\star \in \arg\min_{\boldsymbol{\lambda}} \sum_a [F_a(\boldsymbol{w}_{\boldsymbol{\lambda}})]^2$. .

\end{theorem}

\begin{proof}
We first derive the update rule in the case of $A=2$. 
From~\eqref{eq:update_dp}, we have $\mu_0 = -\mu_1$. Denote 
\[f(\boldsymbol{\lambda}^{(t)}) =-L_{0,0}^\prime(\boldsymbol{w}_{\boldsymbol{\lambda}^{(t)}}) + L_{1,0}^\prime(\boldsymbol{w}_{\boldsymbol{\lambda}}^{(t)}) + L_{0,1}^\prime(\boldsymbol{w}_{\boldsymbol{\lambda}}^{(t)}) - L_{1,1}^\prime (\boldsymbol{w}_{\boldsymbol{\lambda}}^{(t)}) +\frac{n_{0,0}}{n_{\star,0}} - \frac{n_{0,1}}{n_{\star,1}}.\]
Then the update rule becomes
\begin{equation}\label{update_rule_A=2}
    \begin{split}
    \lambda_0^{(t+1)}& = \lambda_0^{(t)}  - \frac{\sqrt{2}}{2}\alpha_t \text{sign}(f(\boldsymbol{\lambda}^{(t)}))     \\
    \lambda_1^{(t+1)}& = \lambda_1^{(t)}  + \frac{\sqrt{2}}{2}\alpha_t \text{sign}(f(\boldsymbol{\lambda}^{(t)})).        
    \end{split}
\end{equation}
We denote $E$ as the number local epochs performed between two communication rounds, and $R$ be the total number of iterations between two $\boldsymbol{\lambda}$ update rounds. Thus $\frac{R}{E}$ is the number of communication rounds between two $\boldsymbol{\lambda}$ update rounds.
Then we apply \fedavg{} with $R$ number of total iterations to solve $\min_{\boldsymbol{w}} L(\boldsymbol{w};\boldsymbol{\lambda}^{(t)})$, and obtain $\boldsymbol{w}_R^{(t)}$. 
Then in the update round from $\boldsymbol{\lambda}^{(t)}$ to $\boldsymbol{\lambda}^{(t+1)}$, by Thm.~1 in \citet{li2020on}, we have 
\begin{equation}
    \begin{aligned}
        & \mathbb{E}[L(\boldsymbol{w}_R^{(t)};\boldsymbol{\lambda}^{(t)})] - L(\boldsymbol{w}_{\boldsymbol{\lambda^{(t)}}}; \boldsymbol{\lambda}^{(t)})  = O\Big(\frac{E^2}{R}\Big),\\
        & \mathbb{E}\norm{\boldsymbol{w}_{R}^{(t)} - \boldsymbol{w}_{\boldsymbol{\lambda}^{(t)}}}^2 
        = O\Big(\frac{E^2}{R}\Big),
    \end{aligned}
\end{equation}
where $\boldsymbol{w}_{\boldsymbol{\lambda}^{(t)}} = \underset{\boldsymbol{w}}{\arg \min} L(\boldsymbol{w};\boldsymbol{\lambda}^{(t)})$. 
By Markov's inequality, with probability $1-\delta$,
\begin{equation}
    \begin{aligned}
        & L(\boldsymbol{w}_R^{(t)};\boldsymbol{\lambda}^{(t)}) - L(\boldsymbol{w}_{\boldsymbol{\lambda}^{(t)}}; \boldsymbol{\lambda}^{(t)}) 
        =   O\Big(\frac{E^2}{\delta R}\Big),\\
        & \norm{\boldsymbol{w}_{R}^{(t)} - \boldsymbol{w}_{\boldsymbol{\lambda}^{(t)}}}^2 
        = O\Big(\frac{E^2}{\delta R}\Big).
    \end{aligned}
\end{equation}
Then taking the union bound over $T$ updating iterations of $\boldsymbol{\lambda}^{(t)}$, the conclusions above hold with probability at least $1-T\delta$ for all $\boldsymbol{\lambda}^{(t)},t=1,2,\cdots,T$. 
Therefore, with sufficiently large $R=R(\delta)$ such that $\frac{E^2}{\delta R}\to 0$, for all $\boldsymbol{\lambda}^{(t)}$ in the $T$ iteration, $|L(\boldsymbol{w}_R^{(t)};\boldsymbol{\lambda}^{(t)}) - L(\boldsymbol{w}_{\boldsymbol{\lambda}^{(t)}} ;\boldsymbol{\lambda}^{(t)})|\ll 1$, $\Vert \boldsymbol{w}_{R}^{(t)} - \boldsymbol{w}_{\boldsymbol{\lambda}^{(t)}}\Vert \ll 1$.

By Lemma~\ref{lemma:dp}, $F_{\text{dp}}(\boldsymbol{\lambda})$ has minimizer $\boldsymbol{\lambda}^\star$ on direction $\boldsymbol{\mu}(\boldsymbol{\lambda})$ such that $\boldsymbol{\mu}(\boldsymbol{\lambda}^\star) = 0$, which implies $F_{\mathrm{dp}} (\boldsymbol{\lambda}^\star) = 0$.
Therefore, $\boldsymbol{\lambda}^\star$ is a global minimizer.
By update rule~\eqref{update_rule_A=2}, we can find large $T>0$ to have 
\begin{equation}
    |\lambda_a^{(T)} - \lambda_a^\star| \leq \max\left\{|\lambda_a^{(0)} - \lambda_a^\star| - \sum_{t=1}^T\alpha_t, \alpha_T \right\} \to 0,
\end{equation}
where $a \in \{0,1\}$.
\end{proof}
\begin{remark}
Note that Thm.~\ref{prop:formal_conv_dp} assumes \fedfb{} does not update $\boldsymbol{\lambda}$ in each communication round and there are infinite rounds of aggregations between two $\boldsymbol{\lambda}$ updating round. However, for computation efficiency, we update $\boldsymbol{\lambda}$ at every communication round in practice.
\end{remark}

\subsection{FedFB \wrt equal opportunity}
Similar to Proposition~\ref{prop:dp}, we design the following bi-level optimization problems to capture equal opportunity:
\begin{align}
\min_{\boldsymbol{\lambda} \in \Lambda} F_{\mathrm{eo}}(\boldsymbol{\lambda})    &= \min_{\boldsymbol{\lambda}\in\Lambda}\sum_{a=1}^{A-1} \left( L_{1,a}(\boldsymbol{w}_{\boldsymbol{\lambda}}) - L_{1,0}(\boldsymbol{w}_{\boldsymbol{\lambda}})\right)^2\\
 \boldsymbol{w}_{\boldsymbol{\lambda}}   &  = \arg\min_{\boldsymbol{w}} \sum_{a=1}^{A-1} \lambda_a L_{1,a}(\boldsymbol{w}) + (\frac{n_{1,\star}}{n} - \sum_{a=1}^{A-1}\lambda_a)L_{1,0}(\boldsymbol{w}) + \frac{n_{0,\star}}{n}L_{0,\star}(\boldsymbol{w}).    \label{eo:w}
\end{align}
Here $\Lambda = \{(\lambda_1, \dots, \lambda_{A-1}): \lambda_1+\dots+\lambda_{A-1} \leq \frac{n_{1,\star}}{n}, \lambda_a \geq 0 \text{ for all } a = 1, \dots, A-1\}$.

For equal opportunity, we make the following assumption:
\begin{assumption}\label{ass:eo}
$L_{y,a}(\cdot)$ is twice differentiable for all $y\in \{0,1\}$, $a\in [A]$, and
\[\sum_{a=1}^{A-1}\lambda_a\nabla^2L_{1,a}(\boldsymbol{w}_{\boldsymbol{\lambda}})+(\frac{n_{1,\star}}{n}-\sum_{a=1}^{A-1}\lambda_a)\nabla^2 L_{1,0}(\boldsymbol{w}_{\boldsymbol{\lambda}}) + \frac{n_{0,\star}}{n}\nabla^2L_{0,\star}(\boldsymbol{w}_{\boldsymbol{\lambda}}) \succ 0\]
for all $\lambda\in \Lambda$.
\end{assumption}

With the above assumption, the following lemma provides the update rule:
\begin{lemma}\label{lemma:eo}
If Assumption~\ref{ass:eo} holds, then on the direction 
\begin{equation}\label{eo:mu}
\boldsymbol{\mu}(\boldsymbol{\lambda}) = (L_{1,1}(\boldsymbol{w}_{\boldsymbol{\lambda}})-L_{1,0}(\boldsymbol{w}_{\boldsymbol{\lambda}}), \dots, L_{1,A-1}(\boldsymbol{w}_{\boldsymbol{\lambda}})-L_{1,0}(\boldsymbol{w}_{\boldsymbol{\lambda}})),
\end{equation}

we have $\boldsymbol{\mu}(\boldsymbol{\lambda}) \cdot \nabla F_{\mathrm{eo}}(\boldsymbol{\lambda}) \leq 0$, and the equality holds if only if $\boldsymbol{\mu}(\boldsymbol{\lambda})=\boldsymbol{0}$.

\end{lemma}
\textbf{Update rule for equal opportunity:}
\begin{equation}\label{update_rule:eo}
\lambda_a^{(t+1)} = \lambda_a^{(t)} + \frac{\alpha_t}{\Vert \boldsymbol{\mu}(\boldsymbol{\lambda}^{(t)})\Vert_2} (L_{1,a}(\boldsymbol{w}_{\boldsymbol{\lambda}^{(t)}}) - L_{1,0}(\boldsymbol{w}_{\boldsymbol{\lambda}^{(t)}})).  
\end{equation}
\begin{proof}[\textbf{Proof of Lemma~\ref{lemma:eo}}]
We compute the derivative as
\begin{equation}
    \begin{split}
      \frac{\partial F_{\mathrm{eo}}(\boldsymbol{\lambda})}{\partial \lambda_j}    & =2\sum_{a=1}^{A-1} (L_{1,a}(\boldsymbol{w}_{\boldsymbol{\lambda}}) - L_{1,0}(\boldsymbol{w}_{\boldsymbol{\lambda}}))(\nabla L_{1,a}(\boldsymbol{w}_{\boldsymbol{\lambda}})  - \nabla L_{1,0}(\boldsymbol{w}_{\boldsymbol{\lambda}}))  \frac{\partial \boldsymbol{w}_{\boldsymbol{\lambda}}}{\partial \lambda_j} . \label{eo:pF_pl}
    \end{split}
\end{equation}
Note that $\boldsymbol{w}_{\boldsymbol{\lambda}}$ is the minimizer to~\eqref{eo:w}, we have
\[ \sum_{a=1}^{A-1} \lambda_a \nabla L_{1,a}(\boldsymbol{w}_{\boldsymbol{\lambda}}) + (\frac{n_{1,\star}}{n} - \sum_{a=1}^{A-1}\lambda_a) \nabla L_{1,0}(\boldsymbol{w}_{\boldsymbol{\lambda}}) + \frac{n_{0,\star}}{n} \nabla L_{0,\star}(\boldsymbol{w}_{\boldsymbol{\lambda}}) = 0.    \]
We take the $\lambda_j$ derivative to the above equation and have
\begin{align}
    & \Big[\sum_{a=1}^{A-1} \lambda_a   \nabla^2 L_{1,a}(\boldsymbol{w}_{\boldsymbol{\lambda}})    +(\frac{n_{1,\star}}{n} - \sum_{a=1}^{A-1}\lambda_a) \nabla^2 L_{1,0}(\boldsymbol{w}_{\boldsymbol{\lambda}}) + \frac{n_{0,\star}}{n} \nabla^2 L_{0,\star}(\boldsymbol{w}_{\boldsymbol{\lambda}})\Big]  \\
    & \nabla L_{1,j}(\boldsymbol{w}_{\boldsymbol{\lambda}}) - \nabla L_{1,0}(\boldsymbol{w}_{\boldsymbol{\lambda}})+ \frac{\partial \boldsymbol{w}_{\boldsymbol{\lambda}}}{\partial \lambda_j} =0 .
\end{align}
Thus we get
\begin{equation}\label{eo:pw_pl}
\begin{split}
 \frac{\partial \boldsymbol{w}_{\boldsymbol{\lambda}}}{\partial \lambda_j}=&    \Big[\sum_{a=1}^{A-1} \lambda_a   \nabla^2 L_{1,a}(\boldsymbol{w}_{\boldsymbol{\lambda}})    +(\frac{n_{1,\star}}{n} - \sum_{a=1}^{A-1}\lambda_a) \nabla^2 L_{1,0}(\boldsymbol{w}_{\boldsymbol{\lambda}}) + \frac{n_{0,\star}}{n} \nabla^2 L_{0,\star}(\boldsymbol{w}_{\boldsymbol{\lambda}})\Big]^{-1}   \\
 &[ \nabla L_{1,0}(\boldsymbol{w}_{\boldsymbol{\lambda}})-\nabla L_{1,j}(\boldsymbol{w}_{\boldsymbol{\lambda}})].  
\end{split}
\end{equation}

Then on the direction $\boldsymbol{\mu}(\boldsymbol{\lambda})$ given by~\eqref{eo:mu}, we combine~\eqref{eo:pF_pl} and~\eqref{eo:pw_pl} to have
\begin{align}
& \boldsymbol{\mu}(\boldsymbol{\lambda}) \cdot \nabla F_{\mathrm{eo}}(\boldsymbol{\lambda})   \\
    & = 2 \Big[\sum_{a=1}^{A-1} (L_{1,a}(\boldsymbol{w}_{\boldsymbol{\lambda}}) - L_{1,0}(\boldsymbol{w}_{\boldsymbol{\lambda}}))(\nabla L_{1,a}(\boldsymbol{w}_{\boldsymbol{\lambda}})-\nabla L_{1,0}(\boldsymbol{w}_{\boldsymbol{\lambda}})) \Big] \\
    & \quad \Big[\sum_{a=1}^{A-1} \lambda_a   \nabla^2 L_{1,a}(\boldsymbol{w}_{\boldsymbol{\lambda}})    +(\frac{n_{1,\star}}{n} - \sum_{a=1}^{A-1}\lambda_a) \nabla^2 L_{1,0}(\boldsymbol{w}_{\boldsymbol{\lambda}}) + \frac{n_{0,\star}}{n} \nabla^2 L_{0,\star}(\boldsymbol{w}_{\boldsymbol{\lambda}})\Big]^{-1} \\
    &\quad \Big[\sum_{a=1}^{A-1}     (L_{1,a}(\boldsymbol{w}_{\boldsymbol{\lambda}})-L_{1,0}(\boldsymbol{w}_{\boldsymbol{\lambda}}))(\nabla L_{1,0}(\boldsymbol{w}_{\boldsymbol{\lambda}}) - \nabla L_{1,a}(\boldsymbol{w}_{\boldsymbol{\lambda}})) \Big] \leq 0,
\end{align}
where we have used Assumption~\ref{ass:eo}, the equality holds only when $\boldsymbol{\mu}(\boldsymbol{\lambda}) = \boldsymbol{0}$.

\end{proof}

We present the \fedfb{} algorithm \wrt{} equal opportunity in Algorithm~\ref{alg:fedfb_eo}.

%  \SetKwBlock{Compute}{ClientUpdate$(i,\boldsymbol{w},\boldsymbol{\lambda})$:}{}
%   \SetKwInOut{Output}{output}

\begin{algorithm2e}
\SetAlgoNoLine
 \SetKwInOut{Input}{input}
 \SetKwBlock{Server}{Server executes:}{ }
\SetKwBlock{Compute}{ClientUpdate$(i,\boldsymbol{w},\boldsymbol{\lambda})$:}{}
  \SetKwInOut{Output}{output}

 \Server{
 \Input{Learning rate $\{\alpha_t\}_{t\in\mathbb{N}}$\;}
Initialize $\lambda_a$ as $\frac{n_{1,a}}{n}$ for all $a\in[A]\backslash \set{0}$\;
\For{each iteration $t=1,2,\dots$}{
Clients perform updates\;
$\boldsymbol{w}_{\boldsymbol{\lambda}} \leftarrow$ SecAgg($\set{\boldsymbol{w}^{(i)}}$) for all $i$\;
$\mu_a\leftarrow$ SecAgg$(\set{\mu_a^{(i)}})$ for all $a\in [A] \backslash \set{0}$\;
% $\mu_a \leftarrow L_{1,a} - L_{1,0}, \rva \in [A]\backslash \set{0} $\;
\uIf{$t\%k = 0$}{
$\lambda_a    \leftarrow \lambda_a    + \frac{\alpha_t}{\Vert \boldsymbol{\mu}\Vert_2} \mu_a,  \text{ for all }a\in [A] \backslash \set{0}$\;
Broadcast $\boldsymbol{\lambda}$ to clients\;
}
Broadcast $\boldsymbol{w}_{\boldsymbol{\lambda}}$ to clients;
}
\Output{$\boldsymbol{w}_{\boldsymbol{\lambda}}, \boldsymbol{\lambda}$}
 }{}
 
 \Compute{
 $\boldsymbol{w}^{(i)} \leftarrow $ Gradient descent \wrt objective function $ \sum_{a=1}^{A-1} \lambda_a L_{1,a}^{(i)}(\boldsymbol{w}) + (\frac{n_{1,\star}}{n} - \sum_{a=1}^{A-1}\lambda_a)L^{(i)}_{1,0}(\boldsymbol{w}) + \frac{n_{0,\star}}{n}L^{(i)}_{0,\star}(\boldsymbol{w})$\;
 Send $\boldsymbol{w}^{(i)},\mu_a^{(i)}$ for all $a\in [A]\backslash\set{0}$ to server via a SecAgg protocol;
 }{}

 \caption{\fedfb{}$(k,t)$ \wrt Equal Opportunity}
 \label{alg:fedfb_eo}
\end{algorithm2e}

\subsection{FedFB \wrt equalized odds}
For equalized odd, we design the following bi-level optimization problem:
\begin{align}
\min_{\boldsymbol{\lambda} \in \Lambda} F_{\mathrm{eod}}(\boldsymbol{\lambda})    &= 
\min_{\boldsymbol{\lambda}\in\Lambda} \sum_{a=1}^{A-1}\left[(L_{1,a}(\boldsymbol{w}_{\boldsymbol{\lambda}}) - L_{1,0}(\boldsymbol{w}_{\boldsymbol{\lambda}}))^2 + (L_{0,a}(\boldsymbol{w}_{\boldsymbol{\lambda}}) - L_{0,0}(\boldsymbol{w}_{\boldsymbol{\lambda}}))^2 \right]\\
 \boldsymbol{w}_{\boldsymbol{\lambda}}   &  = \arg\min_{\boldsymbol{w}} \sum_{a=1}^{A-1} \left(\lambda_{0,a}L_{0,a}(\boldsymbol{w}) + \lambda_{1,a}L_{1,a}(\boldsymbol{w}) \right) \\
 & + (\frac{n_{0,\star}}{n} - \sum_{a=1}^{A-1}\lambda_{0,a})L_{0,0}(\boldsymbol{w}) + (\frac{n_{1,\star}}{n} - \sum_{a=1}^{A-1} \lambda_{1,a}) L_{1,0}(\boldsymbol{w}).    \label{eod:w}
\end{align}
Here 
\begin{align}
 \Lambda    = \{(\lambda_{0,1}, \dots, \lambda_{0,A-1}, \lambda_{1,1}, \dots, \lambda_{1,A-1}):& \sum_{a=1}^{A-1}\lambda_{0,a} \leq \frac{n_{0,\star}}{n}, \quad \sum_{a=1}^{A-1} \lambda_{1,a} \leq \frac{n_{1,\star}}{n}, \\
    & \lambda_{0,a}, \lambda_{1,a} \geq 0, \quad \text{for all }a = 1,\dots, A-1\}.
\end{align}

For equalized odds, we make the following assumption:
\begin{assumption}\label{ass:eod}
$L_{y,a}(\cdot)$ is twice differentiable for all $y\in \{0,1\}$, $a\in [A]$, and
\begin{align}
    & \sum_{a=1}^{A-1}\Big[ (\lambda_{1,a}\nabla^2L_{1,a}(\boldsymbol{w}_{\boldsymbol{\lambda}})  + \lambda_{0,a}\nabla^2L_{0,a}(\boldsymbol{w}_{\boldsymbol{\lambda}}) + (\frac{n_{0,\star}}{n} - \sum_{a=1}^{A-1}\lambda_{0,a})\nabla^2L_{0,0}(\boldsymbol{w}_{\boldsymbol{\lambda}}) \\
    & + (\frac{n_{1,\star}}{n} - \sum_{a=1}^{A-1}\lambda_{1,a})\nabla^2L_{1,0}(\boldsymbol{w}_{\boldsymbol{\lambda}}))\Big] \succ 0
\end{align}
for all $\lambda \in \Lambda$.
\end{assumption}

With the above assumption, the following lemma provides the update rule:
\begin{lemma}\label{lemma:eod}
If Assumption~\ref{ass:eod} holds, then on the direction 

$\boldsymbol{\mu}(\boldsymbol{\lambda}) = (\mu_{0,1}(\boldsymbol{\lambda}), \dots, \mu_{0, A-1}(\boldsymbol{\lambda}), \mu_{1,1}(\boldsymbol{\lambda}), \dots, \mu_{1,A-1}(\boldsymbol{\lambda}))
$, with
\begin{equation}\label{eod:mu}
\mu_{y,a}(\boldsymbol{\lambda}) = L_{y,a}(\boldsymbol{w}_{\boldsymbol{\lambda}}) - L_{y,0}(\boldsymbol{w}_{\boldsymbol{\lambda}}),\quad y\in \{0,1\},\quad a\in [A],
\end{equation}
we have $\boldsymbol{\mu}(\boldsymbol{\lambda})\cdot \nabla F_{\mathrm{eod}}(\boldsymbol{\lambda}) \leq 0$, and the equality holds if only if $\boldsymbol{\mu}(\boldsymbol{\lambda})=\boldsymbol{0}$.

\end{lemma}
\textbf{Update rule for equalized odd:}
\begin{equation}\label{update_rule:eod}
\lambda_{y,a}^{(t+1)} = \lambda_{y,a}^{(t)} + \frac{\alpha_t}{\Vert \boldsymbol{\mu}(\boldsymbol{\lambda}^{(t)})\Vert_2} (L_{y,a}(\boldsymbol{w}_{\boldsymbol{\lambda}^{(t)}}) - L_{y,0}(\boldsymbol{w}_{\boldsymbol{\lambda}^{(t)}})) \text{ for } y\in \{0,1\},a\in [A].    
\end{equation}

\begin{proof}[\textbf{Proof of Lemma~\ref{lemma:eod}}]
We compute the derivative as
\begin{equation}
    \begin{split}
      \frac{\partial F_{\mathrm{eod}}(\boldsymbol{\lambda})}{\partial \lambda_{y,j}}    & =2\sum_{a=1}^{A-1} \big[ (L_{1,a}(\boldsymbol{w}_{\boldsymbol{\lambda}}) - L_{1,0}(\boldsymbol{w}_{\boldsymbol{\lambda}}))(\nabla L_{1,a}(\boldsymbol{w}_{\boldsymbol{\lambda}}) - \nabla L_{1,0}(\boldsymbol{w}_{\boldsymbol{\lambda}})) \\
      & + (L_{0,a}(\boldsymbol{w}_{\boldsymbol{\lambda}}) - L_{0,0}(\boldsymbol{w}_{\boldsymbol{\lambda}}))(\nabla L_{0,a}(\boldsymbol{w}_{\boldsymbol{\lambda}})-\nabla L_{0,0}(\boldsymbol{w}_{\boldsymbol{\lambda}}))\big] \frac{\partial \boldsymbol{w}_{\boldsymbol{\lambda}}}{\partial \lambda_{y,j}}. \label{eod:pF_pl}
    \end{split}
\end{equation}
Note that $\boldsymbol{w}_{\boldsymbol{\lambda}}$ is the minimizer to~\eqref{eod:w}, we have
\begin{align}
    & \sum_{a=1}^{A-1} (\lambda_{0,a} \nabla L_{0,a}(\boldsymbol{w}_{\boldsymbol{\lambda}}) + \lambda_{1,a} \nabla L_{1,a}(\boldsymbol{w}_{\boldsymbol{\lambda}}))\\
    & +  (\frac{n_{0,\star}}{n} - \sum_{a=1}^{A-1}\lambda_{0,a})\nabla L_{0,0}(\boldsymbol{w}_{\boldsymbol{\lambda}}) + (\frac{n_{1,\star}}{n} - \sum_{a=1}^{A-1} \lambda_{1,a}) \nabla L_{1,0}(\boldsymbol{w}_{\boldsymbol{\lambda}}) = 0. 
\end{align}

We take the $\lambda_{y,j}$ derivative to the above equation and have
\begin{align}
    &  \nabla L_{y,j}(\boldsymbol{w}_{\boldsymbol{\lambda}}) -\nabla L_{y,0}(\boldsymbol{w}_{\boldsymbol{\lambda}}) +\Big[\sum_{a=1}^{A-1} (\lambda_{0,a} \nabla^2 L_{0,a}(\boldsymbol{w}_{\boldsymbol{\lambda}}) + \lambda_{1,a} \nabla^2 L_{1,a}(\boldsymbol{w}_{\boldsymbol{\lambda}}) ) \\
    & +  (\frac{n_{0,\star}}{n} - \sum_{a=1}^{A-1}\lambda_{0,a})\nabla^2 L_{0,0}(\boldsymbol{w}_{\boldsymbol{\lambda}}) + (\frac{n_{1,\star}}{n} - \sum_{a=1}^{A-1} \lambda_{1,a}) \nabla^2 L_{1,0}(\boldsymbol{w}_{\boldsymbol{\lambda}})\Big]\frac{\partial \boldsymbol{w}_{\boldsymbol{\lambda}}}{\partial \lambda_{y,j}} =0.
\end{align}
Thus we get
\begin{align}
 \frac{\partial \boldsymbol{w}_{\boldsymbol{\lambda}}}{\partial \lambda_{y,j}}   &  = \Big[\sum_{a=1}^{A-1} (\lambda_{0,a} \nabla^2 L_{0,a}(\boldsymbol{w}_{\boldsymbol{\lambda}}) + \lambda_{1,a} \nabla^2 L_{1,a}(\boldsymbol{w}_{\boldsymbol{\lambda}}) )    \\
    & + (\frac{n_{0,\star}}{n} - \sum_{a=1}^{A-1}\lambda_{0,a})\nabla^2 L_{0,0}(\boldsymbol{w}_{\boldsymbol{\lambda}}) + (\frac{n_{1,\star}}{n} - \sum_{a=1}^{A-1} \lambda_{1,a}) \nabla^2 L_{1,0}(\boldsymbol{w}_{\boldsymbol{\lambda}})\Big]^{-1} \\
    & \quad [\nabla L_{y,0}(\boldsymbol{w}_{\boldsymbol{\lambda}})-\nabla L_{y,j}(\boldsymbol{w}_{\boldsymbol{\lambda}})]. \label{eod:pw_pl}
\end{align}

Then on the direction $\boldsymbol{\mu}(\boldsymbol{\lambda})$ given by~\eqref{eod:mu}, we combine~\eqref{eod:pF_pl} and~\eqref{eod:pw_pl} to have
\begin{align}
& \boldsymbol{\mu}(\boldsymbol{\lambda}) \cdot \nabla F_{\mathrm{eod}}(\boldsymbol{\lambda})   \\
    & = 2 \Big[ \sum_{a=1}^{A-1} [(L_{1,a}(\boldsymbol{w}_{\boldsymbol{\lambda}})-L_{1,0}(\boldsymbol{w}_{\boldsymbol{\lambda}}))(\nabla L_{1,a}(\boldsymbol{w}_{\boldsymbol{\lambda}})-\nabla L_{1,0}(\boldsymbol{w}_{\boldsymbol{\lambda}}))      \\
    & + (L_{0,a}(\boldsymbol{w}_{\boldsymbol{\lambda}}) - L_{0,0}(\boldsymbol{w}_{\boldsymbol{\lambda}}))(\nabla L_{0,a}(\boldsymbol{w}_{\boldsymbol{\lambda}})-\nabla L_{0,0}(\boldsymbol{w}_{\boldsymbol{\lambda}}))] \Big] \\
    & \Big[\sum_{a=1}^{A-1} (\lambda_{0,a} \nabla^2 L_{0,a}(\boldsymbol{w}_{\boldsymbol{\lambda}}) + \lambda_{1,a} \nabla^2 L_{1,a}(\boldsymbol{w}_{\boldsymbol{\lambda}}) )  \\
    & + (\frac{n_{0,\star}}{n} - \sum_{a=1}^{A-1}\lambda_{0,a})\nabla^2 L_{0,0}(\boldsymbol{w}_{\boldsymbol{\lambda}}) + (\frac{n_{1,\star}}{n} - \sum_{a=1}^{A-1} \lambda_{1,a}) \nabla^2 L_{1,0}(\boldsymbol{w}_{\boldsymbol{\lambda}})\Big]^{-1} \\
    &  \Big[  \sum_{a=1}^{A-1} [(L_{1,a}(\boldsymbol{w}_{\boldsymbol{\lambda}})-L_{1,0}(\boldsymbol{w}_{\boldsymbol{\lambda}}))(\nabla L_{1,0}(\boldsymbol{w}_{\boldsymbol{\lambda}})-\nabla L_{1,a}(\boldsymbol{w}_{\boldsymbol{\lambda}}))    \\
    & + (L_{0,a}(\boldsymbol{w}_{\boldsymbol{\lambda}}) - L_{0,0}(\boldsymbol{w}_{\boldsymbol{\lambda}}))(\nabla L_{0,0}(\boldsymbol{w}_{\boldsymbol{\lambda}})-\nabla L_{0,a}(\boldsymbol{w}_{\boldsymbol{\lambda}}))]\Big] \leq  0
\end{align}
where we have used Assumption~\ref{ass:eod}, and the equality holds only when $\boldsymbol{\mu}(\boldsymbol{\lambda}) = \boldsymbol{0}$.

\end{proof}
The full procedure is described in Algorithm~\ref{alg:fedfb_eod}. 

\begin{algorithm2e}
\SetAlgoNoLine
 \SetKwInOut{Input}{input}
 \SetKwBlock{Server}{Server executes:}{ }
 \SetKwBlock{Compute}{ClientUpdate$(i,\boldsymbol{w},\boldsymbol{\lambda})$:}{}
  \SetKwInOut{Output}{output}

 \Server{
 \Input{Learning rate $\{\alpha_t\}_{t\in\mathbb{N}}$\;}
Initialize $\lambda_{y,a}$ as $\frac{n_{y,a}}{n}$ for all $ y\in\set{0,1}, a\in[A]\backslash \set{0}$\;
\For{each iteration $t=1,2,\dots$}{
Clients perform updates\;
$\boldsymbol{w}_{\boldsymbol{\lambda}} \leftarrow$ SecAgg $(\set{\boldsymbol{w}^{(i)}})$ for all $i$\;
$\mu_{y,a} \leftarrow$ SecAgg$(\set{\mu^{(i)}_{y,a}})$ for all $y\in\{0,1\},a\in [A] \backslash \set{0}$\;
\uIf{$t\%k=0$}{
$\lambda_{y,a}    \leftarrow \lambda_{y,a}  + \frac{\alpha_t}{\Vert \boldsymbol{\mu}\Vert_2} \mu_{y,a},  \text{ for all }a\in [A]\backslash \set{0}$\;
Broadcast $\boldsymbol{\lambda}$ to clients\;
}

Broadcast $\boldsymbol{w}_{\boldsymbol{\lambda}}$ to clients\;
}
\Output{$\boldsymbol{w}_{\boldsymbol{\lambda}}$}
 }{}
 
 \Compute{
 $\boldsymbol{w}^{(i)} \leftarrow $ Gradient descent \wrt objective function $\sum_{a=1}^{A-1} \left(\lambda_{0,a}L^{(i)}_{0,a}(\boldsymbol{w}) + \lambda_{1,a}L^{(i)}_{1,a}(\boldsymbol{w}) \right) + (\frac{n_{0,\star}}{n} - \sum_{a=1}^{A-1}\lambda_{0,a})L^{(i)}_{0,0}(\boldsymbol{w}) + (\frac{n_{1,\star}}{n} - \sum_{a=1}^{A-1} \lambda_{1,a}) L^{(i)}_{1,0}(\boldsymbol{w})$\;
Send $\boldsymbol{w}^{(i)}, \mu_{y,a}^{(i)}$ for all $y\in[1], a\in [A]$ to server via a SecAgg protocol;
 }{}

 \caption{\fedfb{}$(k,t)$ \wrt Equalized Odds}
 \label{alg:fedfb_eod}
\end{algorithm2e}

\subsection{FedFB \wrt client parity}\label{appendix:cp}
For client parity, we slightly abuse the notation and define the loss over client $i$ as $L^{(i)}(\boldsymbol{w}) := \sum_{(\mathrm{x},\mathrm{y},\mathrm{a},\mathrm{i}): \mathrm{i}=i}\ell (\mathrm{y}, \hat{\mathrm{y}};\boldsymbol{w}) / n^{(i)}$, with $n^{(i)}: = |\{(\mathrm{x},\mathrm{y},\mathrm{a},\mathrm{i}): \mathrm{i}=i\}|$. 
Note that the $L^{(i)}(\boldsymbol{w})$ here is different from the one in Sec.~\ref{appendix:fedfb_dp}. We design the following bi-level optimization problem:
\begin{align}
& \min _{\boldsymbol{\lambda}\in \Lambda} F_{\mathrm{cp}}(\boldsymbol{\lambda}) = \min _{\boldsymbol{\lambda}\in \Lambda} \sum_{i=1}^{I-1} \left(L^{(i)}\left(\boldsymbol{w}_{\boldsymbol{\lambda}}\right)-L^{(0)}\left(\boldsymbol{w}_{\boldsymbol{\lambda}}\right)\right)^2,     \\ 
& \boldsymbol{w}_{\boldsymbol{\lambda}}=\underset{\boldsymbol{w}}{\arg \min } \sum_{i=1}^{I-1} \lambda^{(i)} L^{(i)}(\boldsymbol{w}) + (1-\sum_{i=1}^{I-1} \lambda^{(i)} )L^{(0)}(\boldsymbol{w}). \label{cp:L}
\end{align}
Here
\[\Lambda = \{(\lambda^{(1)},\dots,\lambda^{(I-1)}): 0\leq \lambda^{(i)} \leq 1,\quad \sum_{i=1}^{I-1} \lambda^{(i)} \leq 1 \}.\]

For client parity, we make the following assumption:
\begin{assumption}\label{ass:cp}
$L^{(i)}(\cdot)$ is twice differentiable for all $i=1,\dots,I-1$, and
\[\sum_{i=1}^{I-1} \lambda^{(i)} \nabla^2 L^{(i)}(\boldsymbol{w}_{\boldsymbol{\lambda}})+(1-\sum_{i=1}^{I-1} \lambda^{(i)})\nabla^2 L^{(0)}(\boldsymbol{w}_{\boldsymbol{\lambda}}) \succ 0.\]
\end{assumption}

With the above assumption, the following lemma provides the update rule:
\begin{lemma}\label{lemma:cp}
If Assumption~\ref{ass:cp} holds, then on the direction
\begin{equation}\label{cp:mu}
\boldsymbol{\mu}(\boldsymbol{\lambda}) = (L^{(1)}(\boldsymbol{w}_{\boldsymbol{\lambda}})-L^{(0)}(\boldsymbol{w}_{\boldsymbol{\lambda}}),\dots,L^{(I-1)}(\boldsymbol{w}_{\boldsymbol{\lambda}})-L^{(0)}(\boldsymbol{w}_{\boldsymbol{\lambda}})),    
\end{equation}
we have $\boldsymbol{\mu}(\boldsymbol{\lambda})\cdot \nabla F_{\mathrm{cp}}(\boldsymbol{\lambda})\leq 0$, and the equality holds if and only if $\boldsymbol{\mu}(\boldsymbol{\lambda})=\boldsymbol{0}$.

\end{lemma}
Then we update $\boldsymbol{\lambda}^{(t)} = (\lambda^{(1)^{(t)}},\dots, \lambda^{(I-1)^{(t)}})$ as follows:

\textbf{Update rule for client parity:}
\begin{equation}\label{update_rule:cp}
\lambda^{(i)^{(t+1)}} = \lambda^{(i)^{(t)}} + \frac{\alpha_t}{\Vert \boldsymbol{\mu}(\boldsymbol{\lambda}^{(t)})\Vert_2} (L^{(i)}(\boldsymbol{w}_{\boldsymbol{\lambda}^{(t)}})-L^{(0)}(\boldsymbol{w}_{\boldsymbol{\lambda}^{(t)}})) \text{ for } i=1,\dots,I-1.   
\end{equation}

\begin{proof}[\textbf{Proof of Lemma~\ref{lemma:cp}}]
We compute the derivative as
\begin{equation}
    \begin{split}
      \frac{\partial F_{\mathrm{cp}}(\boldsymbol{\lambda})}{\partial \lambda^{(j)}}    & = \Big(2\sum_{i=1}^{I-1} [L^{(i)}(\boldsymbol{w}_{\boldsymbol{\lambda}})-L^{(0)}(\boldsymbol{w}_{\boldsymbol{\lambda}})][\nabla L^{(i)}(\boldsymbol{w}_{\boldsymbol{\lambda}}) - \nabla L^{(0)}(\boldsymbol{w}_{\boldsymbol{\lambda}})] \Big) \frac{\partial \boldsymbol{w}_{\boldsymbol{\lambda}}}{\partial \lambda^{(j)}} . \label{pF_pl}
    \end{split}
\end{equation}
Note that $\boldsymbol{w}_{\boldsymbol{\lambda}}$ is the minimizer to~\eqref{cp:L}, we have
\[\sum_{i=1}^{I-1} \lambda^{(i)} \nabla L^{(i)}(\boldsymbol{w}_{\boldsymbol{\lambda}})   + (1-\sum_{i=1}^{I-1} \lambda^{(i)})\nabla L^{(0)}(\boldsymbol{w}_{\boldsymbol{\lambda}})=0.\]

We take the $\lambda^{(j)}$ derivative to the above equation and have
\[\nabla L^{(j)}(\boldsymbol{w}_{\boldsymbol{\lambda}}) + \sum_{i=1}^{I-1} \lambda^{(i)}  \nabla^2 L^{(i)}(\boldsymbol{w}_{\boldsymbol{\lambda}}) \frac{\partial \boldsymbol{w}_{\boldsymbol{\lambda}}}{\partial \lambda_j}   - \nabla L^{(0)}(\boldsymbol{w}_{\boldsymbol{\lambda}}) + (1-\sum_{i=1}^{I-1} \lambda^{(i)}) \nabla^2 L^{(0)}(\boldsymbol{w}_{\boldsymbol{\lambda}}) \frac{\partial \boldsymbol{w}_{\boldsymbol{\lambda}}}{\partial \lambda_j}      =0    .\]
Thus we get
\begin{equation}\label{pw_pl}
 \frac{\partial \boldsymbol{w}_{\boldsymbol{\lambda}}}{\partial \lambda^{(j)}} = \Big[\sum_{i=1}^{I-1} \lambda^{(i)} \nabla^2 L^{(i)}(\boldsymbol{w}_{\boldsymbol{\lambda}})+(1-\sum_{i=1}^{I-1} \lambda^{(i)})\nabla^2 L^{(0)}(\boldsymbol{w}_{\boldsymbol{\lambda}})\Big]^{-1}   [\nabla L^{(0)}(\boldsymbol{w}_{\boldsymbol{\lambda}}) - \nabla L^{(j)}(\boldsymbol{w}_{\boldsymbol{\lambda}})].        
\end{equation}

Then on the direction given by~\eqref{cp:mu}, we combine~\eqref{pF_pl} and~\eqref{pw_pl} to have
\begin{align}
 &  \boldsymbol{\mu}(\boldsymbol{\lambda}) \cdot \nabla F_{\mathrm{cp}}(\boldsymbol{\lambda})   \\
    & = 2 \Big( \sum_{i=1}^{I-1}[ L^{(i)}(\boldsymbol{w}_{\boldsymbol{\lambda}})-L^{(0)}(\boldsymbol{w}_{\boldsymbol{\lambda}})][\nabla L^{(i)}(\boldsymbol{w}_{\boldsymbol{\lambda}})-\nabla L^{(0)}(\boldsymbol{w}_{\boldsymbol{\lambda}})] \Big) \\
    &\quad \Big[\sum_{i=1}^{I-1} \lambda^{(i)}\nabla^2 L^{(i)}(\boldsymbol{w}_{\boldsymbol{\lambda}})+(1-\sum_{i=1}^{I-1} \lambda^{(i)})\nabla^2 L^{(0)}(\boldsymbol{w}_{\boldsymbol{\lambda}})\Big]^{-1} \\
    &\quad \sum_{i=1}^{I-1} \Big([L^{(i)}(\boldsymbol{w}_{\boldsymbol{\lambda}})-L^{(0)}(\boldsymbol{w}_{\boldsymbol{\lambda}})][\nabla L^{(0)}(\boldsymbol{w}_{\boldsymbol{\lambda}}) - \nabla L^{(i)}(\boldsymbol{w}_{\boldsymbol{\lambda}})] \Big) \leq 0,
\end{align}
where we have used Assumption~\ref{ass:cp}, and the equality holds only when $\boldsymbol{\mu}(\boldsymbol{\lambda}) = \boldsymbol{0}$.

\end{proof}

The Algorithm~\ref{alg:fedfb_cp} gives the full description.
\begin{algorithm2e}
\SetAlgoNoLine
 \SetKwInOut{Input}{input}
 \SetKwBlock{Server}{Server executes:}{ }
 \SetKwBlock{Compute}{ClientUpdate$(i,\boldsymbol{w},\boldsymbol{\lambda})$:}{}
  \SetKwInOut{Output}{output}

 \Server{
 \Input{Learning rate $\{\alpha_t\}_{t\in\mathbb{N}}$\;}
Initialize $\lambda^{(i)}$ as $\frac{n^{(i)}}{n}$ for all $i \in [I]\backslash \set{0}$\;
\For{each iteration $t=1,2,\dots$}{
Clients perform updates\;
$\boldsymbol{w}_{\boldsymbol{\lambda}} \leftarrow$ SecAgg $\set{\boldsymbol{w}^{(i)}}$ for all $i$\;
$\mu^{(i)} \leftarrow L^{(i)} - L^{(0)}, i\in [I]\backslash \{0\}   $\;
\uIf{$t\% k = 0$}{
$\lambda^{(i)}    \leftarrow \lambda^{(i)}    + \frac{\alpha_t}{\Vert \boldsymbol{\mu}\Vert_2} \mu^{(i)},  \text{ for all }i\in [I] \backslash \set{0}$\;
Broadcast $\boldsymbol{\lambda}$ to clients\;
}

Broadcast $\boldsymbol{w}_{\boldsymbol{\lambda}}$ to clients\;
}
\Output{$\boldsymbol{w}_{\boldsymbol{\lambda}}$}
 }{}
 
 \Compute{
 $\boldsymbol{w}^{(i)} \leftarrow $ Gradient descent \wrt objective function $  \indicator{i \neq 0}\lambda^{(i)}L^{(i)}(\boldsymbol{w}) + \indicator{i = 0}(1-\sum_{j=1}^{I-1} \lambda^{(i)} )L^{(0)}(\boldsymbol{w})$\;
Send $\boldsymbol{w}^{(i)}$  to server via a SecAgg protocol;
Send $L^{(i)}(\boldsymbol{w})$ to server\;
 }{}

 \caption{\fedfb{} \wrt Client Parity}
 \label{alg:fedfb_cp}
\end{algorithm2e}

\section{Appendix - Experiments}\label{appendix:experiment}
We continue from Sec.~\ref{sec:experiments} and provide more details on experimental settings, and other supplementary experiment results. 
\subsection{Experiment setting}
% For all tasks, we randomly split the data into a 63\% training set, a 7\% validation set and a 30 \% testing set. 
The reported statistics are computed on the test set, and we set 10 communication rounds and 30 local epochs for all federated learning algorithms except \agnosticfair{}. We run 300 epochs for other methods. 
For all tasks, we randomly split data into a training set and a testing set at a ratio of 7:3. 
The batch size is set to be 128. 
For all methods, we choose learning rate from $\set{0.001, 0.005, 0.01}$. 
We solve \fedfb{} with sample weight learning rates $\alpha \in \set{0.001, 0.05, 0.08, 0.1, 0.2, 0.5, 1, 2}$ in parallel and select the $\alpha$ which achieves the highest fairness. 
All benchmark models are tuned according to the hyperparameter configuration suggested in their original works. 
We perform cross-validation on the training sets to find the best hyperparameter for all algorithms. 
% Specifically, we tune \qffl{} with $q\in \set{0, 0.001, 0.01, 0.1, 1, 2, 5, 10}$~\citep{li2020fair}, tune \ditto{} with $$

\subsection{Supplementary experiment results}
Fig.~\ref{fig:figure2_3clients} shows the accuracy versus fairness violation for the 3 clients' cases. We see a clear advantage of \fflfedavg{} over the \ufl{} in both 2 clients' cases and 3 clients' cases. The achievable fairness range of \fflfedavg{} is much wider than that of \ufl{}, though the accuracy of \fedavg{} is not guaranteed when the data heterogeneity increases. 

\begin{figure}
    \centering
    \includegraphics[width=\textwidth]{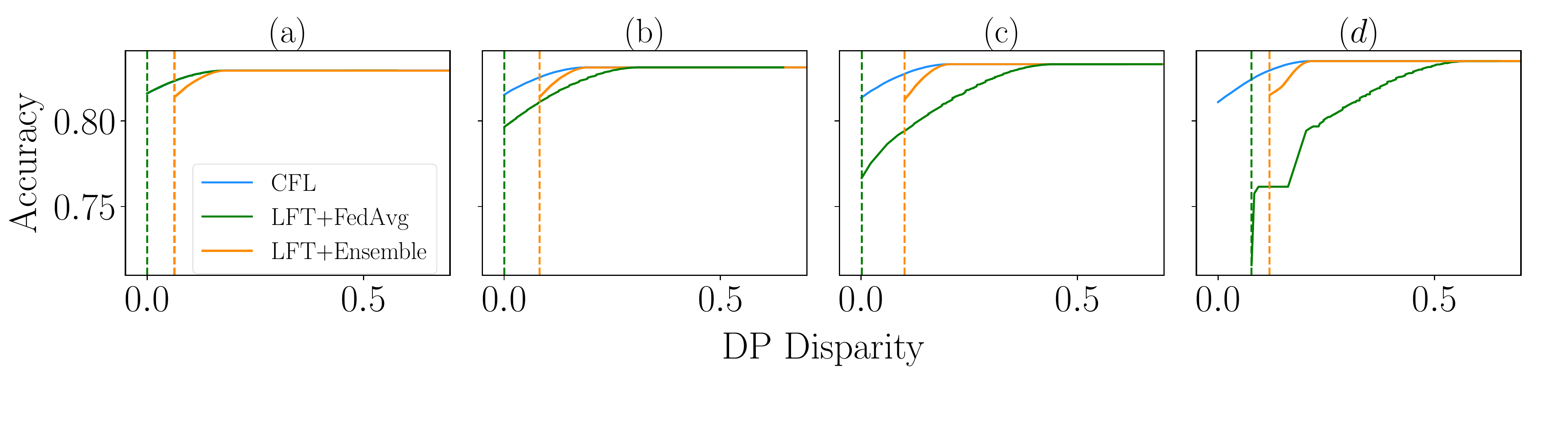}
    \caption{\textbf{Accuracy-Fairness tradeoff curves of CFL, \fflfedavg{}, and \ufl{} for three clients cases.} 
    The data heterogeneity is increasing from left to right. 
    The green dotted vertical line describes the lower bound of unfairness \fflfedavg{} can achieve, and and the orange dotted vertical line describes the lower bound of unfairness \ufl{} can achieve. Here the distribution setting is $\rvx \mid \rva = 0, \rvi = 0 \sim \cN(3, 1), \rvx \mid \rva = 1, \rvi = 0 \sim \cN(5, 1), \rvx \mid \rva = 0, \rvi = 1 \sim \cN(1, 1), \rvx \mid \rva = 1, \rvi = 1 \sim \cN(-1, 1), \rvx\mid \rva = 0, \rvi = 2 \sim \cN(1,1), \rvx \mid \rva=1, \rvi = 2 \sim \cN(2,1) , \rva \mid \rvi = i\sim \Ber(q_i)$ for $i=0,1,2$. 
    The data heterogeneity here is captured by $|q_2-q_0|$. (a) $q_0 = q_1 = q_2 = 0.5$. (b) $q_0 = 0.4, q_1 = 0.5, q_2 = 0.6$. (c): $q_0 = 0.3, q_1 = 0.5, q_2 = 0.7$. (d): $q_0 = 0.2, q_1 = 0.5, q_2 = 0.8$.}
    \label{fig:figure2_3clients}
\end{figure}

We also compare \fedavg{}, \ufl{}, \fflfedavg{}, \fedfb{}, \agnosticfair{} and CFL on logistic regression. Table~\ref{tab:dp_lg} shows that our method outperforms \ufl{} and \fflfedavg{}, while achieving similar performance as \agnosticfair{} and CFL but ensuring higher privacy.

% We present the comparison of accuracy and fairness \wrt client parity in Table~\ref{tab:dp_lg}. The competitive performance of \fedfb{} is observed. 

\begin{table}[t]
\caption{\textbf{Comparison of accuracy and fairness in the synthetic, Adult, COMPAS, and Bank datasets \wrt demographic parity (DP) on logistic regression.} 
The implementation of \ufl{}, \fflfedavg{}, and CFL are all based on FB.}
\label{tab:dp_lg}
\begin{center}
\begin{small}
\begin{sc}\tiny{
\begin{tabularx}{\textwidth}{ccc@{\extracolsep{\fill}}c@{\extracolsep{\fill}}c@{\extracolsep{\fill}}c@{\extracolsep{\fill}}c@{\extracolsep{\fill}}c@{\extracolsep{\fill}}c@{\extracolsep{\fill}}c@{\extracolsep{\fill}}r}
\toprule
\multicolumn{2}{c}{Property}&& \multicolumn{2}{c}{Synthetic} & \multicolumn{2}{c}{Adult} & \multicolumn{2}{c}{COMPAS} & \multicolumn{2}{c}{Bank}\\
Private&Fair&Method & Acc.($\uparrow$) & DP Disp.($\downarrow$)& Acc.($\uparrow$) & DP Disp.($\downarrow$)& Acc.($\uparrow$) & DP Disp.($\downarrow$)& Acc.($\uparrow$) & DP Disp.($\downarrow$) \\
\midrule

\cmark&\xmark&\fedavg{}    &.884$\pm$.001 & .419$\pm$.006 & .837$\pm$.007& .144$\pm$.015& .658$\pm$.006& .149$\pm$.022&   .900$\pm$.000& .026$\pm$.001 \\ \midrule
\cmark&\cmark&\ufl{}  &.712$\pm$.198 & .266$\pm$.175 & .819$\pm$006.& .032$\pm$.031& .606$\pm$.018& .089$\pm$.058&   .888$\pm$.005 & .008$\pm$.007 \\
\cmark&\cmark&\fflfedavg{} &.789$\pm$.138 & .301$\pm$.192 & .828$\pm$.002& .098$\pm$.008& .577$\pm$.003& \textbf{.053$\pm$.003}&   .892$\pm$.000& .013$\pm$.000 \\ 
\cmark&\cmark&FedFB (Ours)  &.756$\pm$.001 & \textbf{.085$\pm$.001} & .820$\pm$.000& \textbf{.002$\pm$.001} & .600$\pm$.003& .059$\pm$.003&   .890$\pm$.001& \textbf{.011$\pm$.002} \\
\midrule
% \agnosticfair{} & .622$\pm$.051 & \textbf{.028$\pm$.008} & .768$\pm$.000 & .003$\pm$.000 & .568$\pm$.018 & .034$\pm$.023 & .883$\pm$.000 & \textbf{.000$\pm$.000} \\
\xmark&\cmark&CFL &.709$\pm$.001 & .011$\pm$.001 & .800$\pm$.003& .016$\pm$.003& .587$\pm$.003&.032$\pm$.002&   .883$\pm$.000& .000$\pm$.000 \\
\bottomrule
\end{tabularx} }
\end{sc}
\end{small}
\end{center}
\end{table}

\subsection{Synthetic dataset generation} \label{appendix:synthetic}
We generate a synthetic dataset of 5,000 examples with two non-sensitive attributes $(\rvx_1, \rvx_2)$, a binary sensitive $\rva$, and a binary label $\rvy$. A tuple $(\rvx_1, \rvx_2, \rvy)$ is randomly generated based on the two Gaussian distributions: $(\rvx_1, \rvx_2) \mid \rvy \sim \cN([-2;-2], [10,1;1,3])$ and $(\rvx_1, \rvx_2) \mid \rvy = 1 \sim \cN([2;2], [5,1])$, where $\rvy \sim \Ber(0.6)$. For the sensitive attribute $\rva$, we generate biased data using an unfair scenario $p_{\rvx_1, \rvx_2}((\rvx_1^\prime, \rvx_2^\prime) \mid \rvy = 1)/ [p_{\rvx_1, \rvx_2}((\rvx_1^\prime, \rvx_2^\prime) \mid \rvy = 0) + p_{\rvx_1, \rvx_2}((\rvx_1^\prime, \rvx_2^\prime) \mid \rvy = 1)]$, where $p_{\rvx_1, \rvx_2}$ is the pdf of $(\rvx_1, \rvx_2)$. 

We split the data into three clients in a non-iid way. 
We randomly assign $50\%, 30\%, 20\%$ of the samples from group 0 and $20\%, 40\%, 40\%$ of the samples from group 1 to 1st, 2nd, and 3rd client, respectively. 
To study the empirical relationship between the performance of \fedfb{} and data heterogeneity, we split the dataset into other ratios to obtain the desired level of data heterogeneity. 
To get the dataset with low data heterogeneity, we draw samples from each group into three clients in a ratio of $33\%, 33\%, 34\%$ and $33\%, 33\%, 34\%$. 
For dataset with high data heterogeneity, the ratio we choose is $70\%, 10\%, 20\%$ and $10\%, 80\%, 10\%$.

% Specifically, we solve \qffl{} with $q \in \{0, 0.001, 0.01, 0.1, 1, 2, 5, 10\}$ in parallel and select the best q. For \ditto{}, we tune the regularization parameter $\lambda_{\text{Ditto}} \in \{0.01, 0.05, 0.1, 0.5, 1, 2, 5\}$. In \gifair{}, we tune the parameter $\lambda_{\text{GIFAIR}} \in \set{0.1\lambda_{\max}, 0.2\lambda_{\max}, 0.3\lambda_{\max}, 0.4\lambda_{\max}, 0.6\lambda_{\max}, 0.8\lambda_{\max}}$, where $\lambda_{\max}$ is a data-dependent constant. 

\subsection{Empirical relationship between accuracy, fairness, and the number of clients}
We investigate the empirical relationship between accuracy, fairness, and the number of clients. 
We generate three synthetic datasets of 3,333, 5,000, and 6,667 samples, and split them into two, three, and four clients, respectively. Table~\ref{tab:dp_num_clients} shows that our method outperforms under the three cases. 

\begin{table}[h]
\caption{\textbf{Comparison of accuracy and fairness in the synthetic datasets with different numbers of clients \wrt demographic parity (DP) on multilayer perceptron. }
The implementation of \ufl{}, \fflfedavg{}, and CFL are all based on FB.}
\label{tab:dp_num_clients}
\begin{center}
\begin{small}
\begin{sc}
\begin{tabularx}{\textwidth}{l@{\extracolsep{\fill}}c@{\extracolsep{\fill}}c@{\extracolsep{\fill}}c@{\extracolsep{\fill}}c@{\extracolsep{\fill}}c@{\extracolsep{\fill}}r}
\toprule
Number of clients & \multicolumn{2}{c}{two clients} & \multicolumn{2}{c}{three clients}& \multicolumn{2}{c}{four clients}\\
Method & Acc.($\uparrow$) & DP Disp.($\downarrow$)& Acc.($\uparrow$) & DP Disp.($\downarrow$)& Acc.($\uparrow$) & DP Disp.($\downarrow$) \\
\midrule

\fedavg{}   & .879$\pm$.004 & .360$\pm$.013& .886$\pm$.003& .406$\pm$.009 &.879$\pm$.003 & .382$\pm$.005 \\ \midrule
\ufl{} &.780$\pm$.090 & .161$\pm$.157 & .727$\pm$.194& .248$\pm$.194 &.720$\pm$.192 & .246$\pm$.180 \\
\fflfedavg{}&.866$\pm$.002 & .429$\pm$.017 &.823$\pm$.102 & .305$\pm$.131 & .608$\pm$.367& .491$\pm$.102\\
FedFB &.679$\pm$.007 & \textbf{.047$\pm$.012} &.725$\pm$.012 & \textbf{.051$\pm$.018} & .705$\pm$.004& \textbf{.005$\pm$.003}\\
\midrule
CFL&.668$\pm$.028 & .063$\pm$.035 & .726$\pm$.009 & .028$\pm$.016 & .670$\pm$.035 & .064$\pm$.020\\
\bottomrule
\end{tabularx}
\end{sc}
\end{small}
\end{center}
\vskip -0.2in
\end{table}

\subsection{Comparison with \fairfed{}}
\begin{table}[!h]
\caption{\textbf{Comparison of accuracy and fairness under the same setting as \citet{ezzeldin2021fairfed}.}
The statistics of \fairfed{} are from \citet{ezzeldin2021fairfed}.}
\label{tab:fairfed2}
% \vskip 0.15in
\begin{center}
\begin{small}
% \tiny{
\begin{sc}
\begin{tabularx}{\textwidth}{l@{\extracolsep{\fill}}l@{\extracolsep{\fill}}c@{\extracolsep{\fill}}c@{\extracolsep{\fill}}c@{\extracolsep{\fill}}c@{\extracolsep{\fill}}c@{\extracolsep{\fill}}c@{\extracolsep{\fill}}c@{\extracolsep{\fill}}c@{\extracolsep{\fill}}c@{\extracolsep{\fill}}c}
\toprule
&\multirow{3}{*}{Method} & \multicolumn{5}{c}{Adult} & \multicolumn{5}{c}{COMPAS}\\ 
& & \multicolumn{5}{c}{Heterogeneity Level $\alpha$} & \multicolumn{5}{c}{Heterogeneity Level $\alpha$}  \\
& & 0.1 & 0.2 & 0.5 & 10 & 5000 & 0.1 & 0.2 & 0.5 & 10 & 5000  \\ \midrule
\multirow{2}{*}{Acc.($\uparrow$)}& FairFed & .775 & .794 & .819& .824 & .824 & .594 & .586 & .608 & .636 & .640  \\
& \textbf{FedFB (Ours)}& .799 & .769 & .772 & .822 & .769 & .668 & .655 & .665 & .672 &  .666 \\ \midrule
\multirow{2}{*}{$|$SPD$|$($\downarrow$)} & FairFed & .\textbf{021} & .037 & .061 & .065 & .065 & .048 & .040 & .072 & .108 & .115  \\
& \textbf{FedFB (Ours)} & .042 & \textbf{.034} & \textbf{.038} & \textbf{.063} & \textbf{.060} & \textbf{.006} & \textbf{.011} & \textbf{.026} & \textbf{.032} & \textbf{.004}\\
\bottomrule
\end{tabularx}
\end{sc} 
% } 
\end{small}
\end{center}
\vskip -0.2in
\end{table}
\begin{table}[!h]
\caption{Comparison of accuracy and fairness in the synthetic, Adult, COMPAS, and Bank datasets \wrt{} demographic parity (DP) on multilayer perceptron.}
\label{tab:agnostic}
\begin{center}
\begin{small}
\begin{sc}
\begin{tabularx}{\textwidth}{l@{\extracolsep{\fill}}c@{\extracolsep{\fill}}c@{\extracolsep{\fill}}c@{\extracolsep{\fill}}c@{\extracolsep{\fill}}c@{\extracolsep{\fill}}c@{\extracolsep{\fill}}c@{\extracolsep{\fill}}r}
\toprule
& \multicolumn{2}{c}{Synthetic} & \multicolumn{2}{c}{Adult} & \multicolumn{2}{c}{COMPAS} & \multicolumn{2}{c}{Bank}\\
Method & Acc.($\uparrow$) & DP Disp.($\downarrow$)& Acc.($\uparrow$) & DP Disp.($\downarrow$)& Acc.($\uparrow$) & DP Disp.($\downarrow$)& Acc.($\uparrow$) & DP Disp.($\downarrow$) \\
\midrule
\textbf{FedFB (Ours)} &.725$\pm$.012 & .051$\pm$.018 & .804$\pm$.001& \textbf{.028$\pm$.001}& .616$\pm$.033& \textbf{.036$\pm$.028}&  .883$\pm$.000& \textbf{.000$\pm$.000} \\
\agnosticfair{} & .655$\pm$.028 & \textbf{.028$\pm$.033} & .790$\pm$.012 & .032$\pm$.016 & .612$\pm$.037 & .096$\pm$.046 & .883$\pm$.000 & \textbf{.000$\pm$.000} \\
\bottomrule
\end{tabularx}
\end{sc}
\end{small}
\end{center}
\vskip -0.15in
\end{table}

To make fairer comparison, we follow the exact same setting as \citet{ezzeldin2021fairfed} to re-split Adult and COMPAS dataset, employ $|SPD| = |\bP(\haty = 1 \mid \ra = 0) - \bP(\haty = 1 \mid \ra = 1)|$ as unfairness metric and present the results in Table~\ref{tab:fairfed2}. 
We observe that \fedfb{} achieves higher fairness than \fairfed{} while being robust to data heterogeneity.

\subsection{Comparison with \agnosticfair{}}
Lastly, we compare our method with \agnosticfair{}~\citep{du2020fairfl}, which exchanges the model parameters and the other information after every gradient update to mimic the performance of CFL with \textsc{FairnessConstraint} implementation~\citep{Zafar2017FC}. Table~\ref{tab:agnostic} shows that our \fedfb{} achieves similar performance as \agnosticfair{}, at a much lower cost of privacy.

\end{document}